%% file: neurips_2024_camera_ready.tex
\documentclass{article}

\usepackage[final]{neurips_2024}

\usepackage[utf8]{inputenc} 
\usepackage[T1]{fontenc}    
\usepackage[hidelinks]{hyperref}       
\usepackage{url}            
\usepackage{booktabs}       
\usepackage{amsfonts}       
\usepackage{nicefrac}       
\usepackage{microtype}      
\usepackage{xcolor}         

\usepackage{subcaption}

\usepackage{amsmath}
\usepackage{amssymb}
\usepackage{mathtools}
\usepackage{amsthm}
\usepackage{bm}
\usepackage{wrapfig}
\usepackage{comment}
\usepackage{array}

\usepackage[capitalize,noabbrev]{cleveref}

\theoremstyle{plain}
\newtheorem{theorem}{Theorem}[section]

\newtheorem{lemma}[theorem]{Lemma}

\theoremstyle{definition}

\theoremstyle{remark}

\title{Unity by Diversity: Improved Representation Learning for Multimodal VAEs}

\author{ 
  \href{mailto:thomas.sutter@inf.ethz.ch}{Thomas M.~Sutter}$^1$\thanks{Corresponding author: \texttt{thomas.sutter@inf.ethz.ch}}, Yang Meng$^{3}$, Andrea Agostini$^1$, Daphné Chopard$^{1,2}$, \\
  \textbf{Norbert Fortin}$^4$, \textbf{Julia E. Vogt}$^1$, \textbf{Babak Shahbaba}$^3$, \textbf{Stephan Mandt}$^{3,5}$ \\
  \\
    $^1$Department of Computer Science, ETH Zurich \\
    $^2$Department of Intensive Care and Neonatology, University Children’s Hospital Zurich \\
    $^3$Department of Statistics, UC Irvine \\
    $^4$Department of Neurobiology and Behavior, UC Irvine \\
    $^5$Department of Computer Science, UC Irvine
}

\begin{document}

\maketitle

\begin{abstract}
Variational Autoencoders for multimodal data hold promise for many tasks in data analysis, such as representation learning, conditional generation, and imputation.
Current architectures either share the encoder output, decoder input, or both across modalities to learn a shared representation. 
Such architectures impose hard constraints on the model. 
In this work, we show that a better latent representation can be obtained by replacing these hard constraints with a soft constraint. We propose a new mixture-of-experts prior, softly guiding each modality's latent representation towards a shared aggregate posterior.
This approach results in a superior latent representation and allows each encoding to preserve information better from its uncompressed original features. In extensive experiments on multiple benchmark datasets and two challenging real-world datasets, we show improved learned latent representations and imputation of missing data modalities compared to existing methods. 
\end{abstract}

\input{intro.tex}
\input{related_work.tex}

\input{methods.tex}
\input{experiments.tex}

\input{conclusion.tex}

\section*{Acknowledgements}
TS, AA, and DC are supported by the grant \#2021-911 of the Strategic Focal Area “Personalized Health and Related Technologies (PHRT)” of the ETH Domain (Swiss Federal Institutes of Technology).
BS and NF received funding from the NIH award R01-MH115697 and NSF award NCS-FR-2319618.
SM acknowledges support from the National Science Foundation (NSF) under an NSF CAREER Award IIS-2047418 and IIS-2007719, the NSF LEAP Center, by the Department of Energy under grant DE-SC0022331, the IARPA WRIVA program, the Hasso Plattner Research Center at UCI, and by gifts from Qualcomm and Disney.

\bibliographystyle{abbrvnat}
\bibliography{references}


\newpage
\appendix
\input{appendix.tex}

\section*{}
\newpage

\section*{NeurIPS Paper Checklist}
\begin{enumerate}

\item {\bf Claims}
    \item[] Question: Do the main claims made in the abstract and introduction accurately reflect the paper's contributions and scope?
    \item[] Answer: \answerYes{} 
    \item[] Justification: In the paper, we introduce a novel multimodal method that we evaluate against previous multimodal approaches. In the experiments, we show improved performance on multiple datasets.

\item {\bf Limitations}
    \item[] Question: Does the paper discuss the limitations of the work performed by the authors?
    \item[] Answer: \answerYes{} 
    \item[] Justification: We discuss the limitations of the paper and the proposed method in the limitations section (see \cref{sec:limitations})

\item {\bf Theory Assumptions and Proofs}
    \item[] Question: For each theoretical result, does the paper provide the full set of assumptions and a complete (and correct) proof?
    \item[] Answer: \answerYes{} 
    \item[] Justification: We describe the assumptions to \cref{thm:optimal_prior_distribution} in \cref{sec:background,sec:method} and the proof in \cref{sec:method}.

    \item {\bf Experimental Result Reproducibility}
    \item[] Question: Does the paper fully disclose all the information needed to reproduce the main experimental results of the paper to the extent that it affects the main claims and/or conclusions of the paper (regardless of whether the code and data are provided or not)?
    \item[] Answer: \answerYes{}{} 
    \item[] Justification: We describe all the experiments in full detail in the appendix (see \cref{app:experiments}) such that all the results on all datasets can be reproduced.

\item {\bf Open access to data and code}
    \item[] Question: Does the paper provide open access to the data and code, with sufficient instructions to faithfully reproduce the main experimental results, as described in supplemental material?
    \item[] Answer: \answerYes{} 
    \item[] Justification: We provide the full code to reproduce all results, including README files and environment requirements. All the data is publicly available, and we provide information on how to access the data.

\item {\bf Experimental Setting/Details}
    \item[] Question: Does the paper specify all the training and test details (e.g., data splits, hyperparameters, how they were chosen, type of optimizer, etc.) necessary to understand the results?
    \item[] Answer: \answerYes{} 
    \item[] Justification: We provide all the necessary information either in the appendix (see \cref{app:experiments}) or link to references, where the necessary information is provided.

\item {\bf Experiment Statistical Significance}
    \item[] Question: Does the paper report error bars suitably and correctly defined or other appropriate information about the statistical significance of the experiments?
    \item[] Answer: \answerYes{} 
    \item[] Justification: We report average results for all experiments performed over at least 3 seeds. For smaller experiments, we average over 5 seeds. We report the number of random seeds used per experiment in the appendix.

\item {\bf Experiments Compute Resources}
    \item[] Question: For each experiment, does the paper provide sufficient information on the computer resources (type of compute workers, memory, time of execution) needed to reproduce the experiments?
    \item[] Answer: \answerYes{} 
    \item[] Justification: We report the compute time for the experiments conducted (including a statement that more compute time was needed for the development) and the used GPU models.
    
\item {\bf Code Of Ethics}
    \item[] Question: Does the research conducted in the paper conform, in every respect, with the NeurIPS Code of Ethics \url{https://neurips.cc/public/EthicsGuidelines}?
    \item[] Answer: \answerYes{} 
    \item[] Justification: We completely follow the Neurips Code of Ethics.

\item {\bf Broader Impacts}
    \item[] Question: Does the paper discuss both potential positive societal impacts and negative societal impacts of the work performed?
    \item[] Answer: \answerYes{} 
    \item[] Justification: We discuss the potential broader impact of the proposed work alongside its limitations in \cref{sec:limitations}.
    
\item {\bf Safeguards}
    \item[] Question: Does the paper describe safeguards that have been put in place for responsible release of data or models that have a high risk for misuse (e.g., pretrained language models, image generators, or scraped datasets)?
    \item[] Answer: \answerNA{} 
    \item[] Justification: We do not use any pretrained language models, image generators, or scraped datasets.

\item {\bf Licenses for existing assets}
    \item[] Question: Are the creators or original owners of assets (e.g., code, data, models), used in the paper, properly credited and are the license and terms of use explicitly mentioned and properly respected?
    \item[] Answer: \answerYes{} 
    \item[] Justification: We cite all the papers we took inspiration from or we used code from. We include licenses for all datasets.

\item {\bf New Assets}
    \item[] Question: Are new assets introduced in the paper well documented and is the documentation provided alongside the assets?
    \item[] Answer: \answerYes{} 
    \item[] Justification: We will publicly release our code upon acceptance and document it accordingly. We attach the code to the submission as supplementary files.

\item {\bf Crowdsourcing and Research with Human Subjects}
    \item[] Question: For crowdsourcing experiments and research with human subjects, does the paper include the full text of instructions given to participants and screenshots, if applicable, as well as details about compensation (if any)? 
    \item[] Answer: \answerNA{} 
    \item[] Justification:

\item {\bf Institutional Review Board (IRB) Approvals or Equivalent for Research with Human Subjects}
    \item[] Question: Does the paper describe potential risks incurred by study participants, whether such risks were disclosed to the subjects, and whether Institutional Review Board (IRB) approvals (or an equivalent approval/review based on the requirements of your country or institution) were obtained?
    \item[] Answer: \answerNA{} 
    \item[] Justification:

\end{enumerate}


\end{document}

%% file: intro.tex
\section{Introduction}
\label{sec:intro}
The fusion of diverse modalities and data types is transforming our understanding of complex phenomena, enabling more nuanced and comprehensive insights through the integration of varied information sources. Consider, for instance, the role of a medical practitioner who synthesizes multiple tests and measurements during diagnosis and treatment. This process involves merging shared information across different tests and identifying test-specific details, both of which are critical for optimal patient care and medical decision-making. 

Among the existing methods, multimodal Variational Autoencoders (VAEs) have emerged as a promising approach for jointly modeling and learning from weakly-supervised heterogeneous data sources. While scalable multimodal VAEs utilizing a single shared latent space efficiently handle multiple modalities \citep{wu_multimodal_2018, shi_variational_2019, sutter2021}, finding an optimal method to aggregate these modalities remains challenging. The aggregation methods and resulting joint representations are often suboptimal and overly restrictive \citep{daunhawer2022, sutter2023imposing}, leading to inferior latent representations and generative quality. This trade-off between shared and modality-specific information in the latent representations of multimodal VAEs results in limited quality or coherence in generated samples, even in relatively simple scenarios.

In this work, we propose a novel multimodal VAE, termed the multimodal variational mixture-of-experts prior (MMVM) VAE, to overcome the aforementioned limitations. Instead of modeling the dependencies between different modalities through a joint posterior approximation, we introduce a multimodal and data-dependent prior distribution (see \cref{fig:exp_architectures}). Our proposed multimodal objective is inspired by the VAMP-prior formulation introduced by \citet{tomczak2017vaebb}, which is traditionally used to learn an optimal prior distribution between unimodal data samples, whereas we aim for an optimal prior between different modalities of the same data sample. The resulting regularization term in the VAE objective can be interpreted as minimizing the distance between positive pairs, similar to contrastive learning methods \citep{vandenOord2019, tian2020}; see \cref{sec:method} for details.

We demonstrate the superior performance of the MMVM VAE on three multimodal benchmark datasets, comparing it to unimodal VAEs and multimodal VAEs with joint posterior approximations. Our evaluation focuses on the generative coherence and the quality of the learned latent representations. While independent unimodal VAEs fail to leverage additional modalities during training, they avoid multimodal aggregation disturbances in data reconstruction. On the other hand, multimodal VAEs with a joint posterior approximation must combine both shared and modality-specific information. Previous work by \citet{daunhawer2022} has shown that this approach results in a trade-off between reconstruction quality and learned latent representation. In contrast, the MMVM VAE accurately reconstructs all modalities and learns meaningful latent representations.

In more practical settings, we address two challenging tasks from the neuroscience and medical domain. First, we analyze hippocampal neural activities from multiple subjects in a memory experiment \citep{allen16}. By treating each subject as a modality, our MMVM VAE enables the description of underlying neural patterns shared across subjects while quantifying individual differences in brain activity and behavior, thereby providing potential insights into the neural mechanisms underlying memory impairment.
Second, we tackle identifying cardiopulmonary diseases from chest radiographs using the MIMIC-CXR dataset \citep{johnson2019mimic}, which reflects real-world conditions with images of varying quality. By leveraging both frontal and lateral X-ray views as distinct modalities, our MMVM method learns representations that consistently improve disease classification compared to existing VAEs.

This paper advances multimodal machine learning by providing a robust framework for integrating diverse data types and improving the quality of learned representations and generative models.

%% file: related_work.tex

\begin{figure*}[t!]
    \begin{subfigure}{0.2635\textwidth}
        \includegraphics[width=1.0\textwidth]{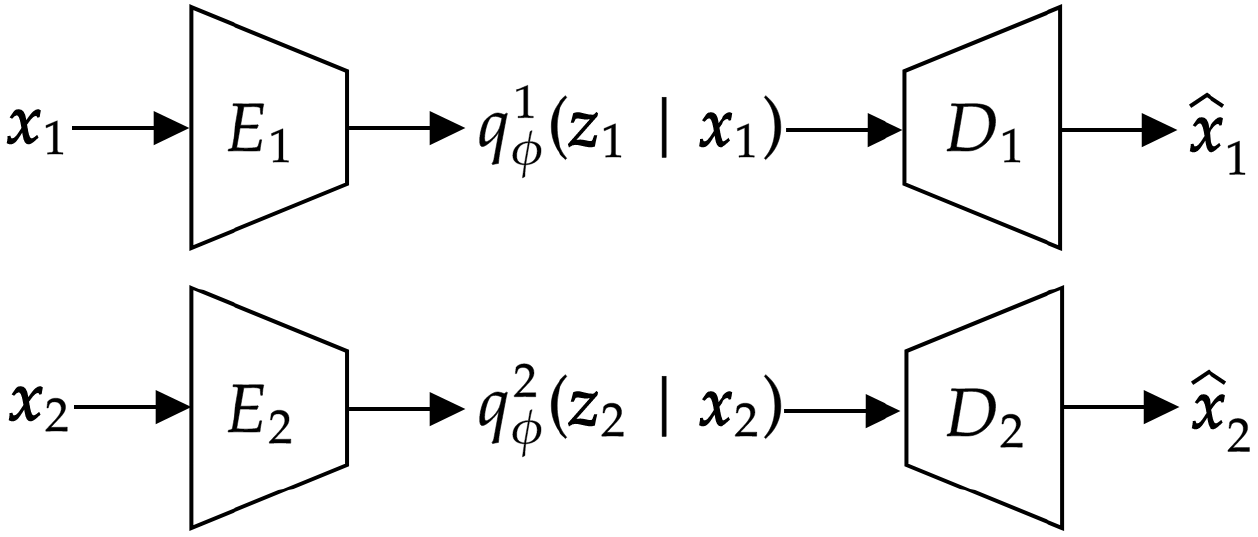}
        \caption{Independent VAEs}
        \label{fig:exp_arch_ind_vaes}
    \end{subfigure}
    \hfill
    \begin{subfigure}{0.4165\textwidth}
        \includegraphics[width=1.0\textwidth]{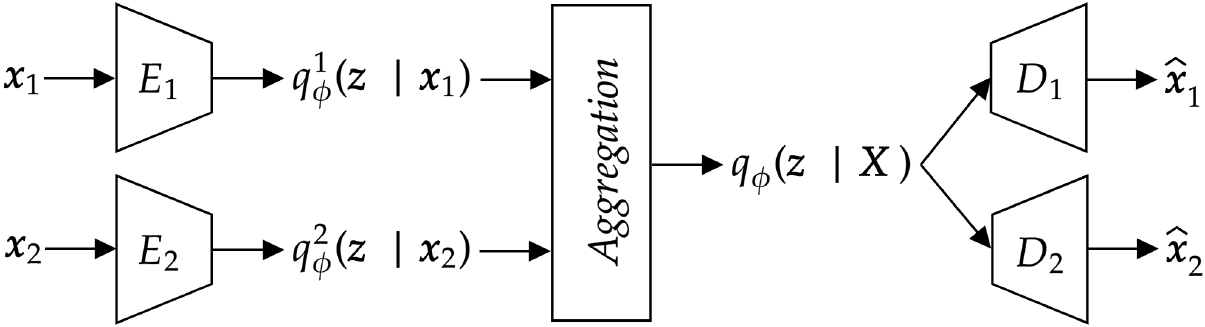}
        \caption{Aggregated VAE}
        \label{fig:exp_arch_agg_vaes}
    \end{subfigure}
    \hfill
    \begin{subfigure}{0.2635\textwidth}
        \includegraphics[width=1.0\textwidth]{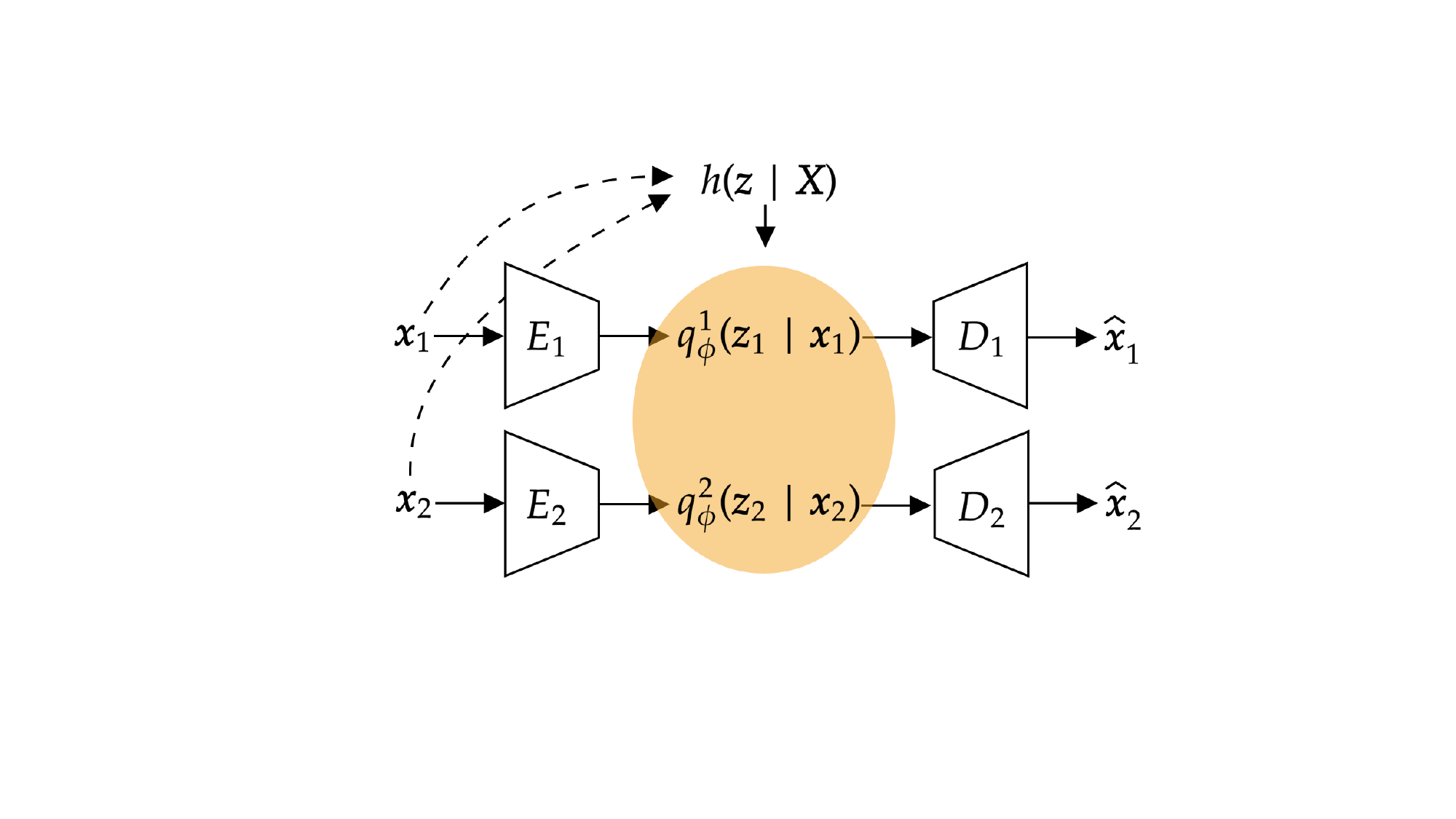}
        \caption{MMVM VAE}
        \label{fig:exp_arch_mmvamp_vaes}
    \end{subfigure}
    \caption{
    Independent VAEs (\cref{fig:exp_arch_ind_vaes}) provide reconstructions for individual modalities but lack information sharing across modalities. Multimodal VAEs with joint posterior approximation (\cref{fig:exp_arch_agg_vaes}) aggregate unimodal posteriors into a joint posterior but may incur poor reconstruction quality. Our proposed MMVM VAE (\cref{fig:exp_arch_mmvamp_vaes}) enhances independent VAEs with a data-dependent prior, $h (\bm{z} \mid \bm{X})$, allowing soft-sharing of information between modalities while preserving modality-specific reconstructions.
    }
    \label{fig:exp_architectures}
\end{figure*}

\section{Related Work}
\label{sec:related_work}
\textbf{Multimodal Learning.}
While there is a long line of research on multimodal machine learning (ML) \citep{baltruvsaitis2018multimodal,liang2022foundations}, multimodal generative ML has gained additional attraction in recent years~\citep{manduchi2024challenges}, driven by impressive results in text-to-image generation \citep{ramesh_zero-shot_2021,ramesh_hierarchical_2022,saharia_photorealistic_2022}.
Unlike these methods, we focus on scalable methods that are designed for a large number of modalities to generate any modality from any other modality without having to train a prohibitive number of different models\footnote{There are $2^M-1$ different subsets for a dataset of $M$ modalities and, hence, paths for any-to-any mappings.}.

\noindent
\textbf{Multimodal VAEs.}
Scalable multimodal VAEs using a joint posterior approximation are based on aggregation in the latent space\footnote{\citet{sutter2021} describe how different implementations of joint posterior multimodal VAEs relate to different abstract mean definitions.}.
Multimodal VAEs that learn a joint posterior approximation of all modalities \citep[e.g.,][]{wu_multimodal_2018,shi_variational_2019,sutter2021} require restrictive assumptions, which lead to inferior performance.
\citet{daunhawer2022} show that aggregation-based multimodal VAEs cannot achieve the same generative quality as unimodal VAEs and struggle with learning meaningful representations depending on the relation between modalities.
If we can predict one modality from another, mixture-of-experts-based posterior approximations perform best if only a single modality is given as input, while product-of-experts-based approximations excel if the full set of modalities is available.
Extensions \citep{sutter_multimodal_2020,daunhawer_self-supervised_2020,palumbo_mmvae_2022} have introduced modality-specific latent subspaces that lead to improved generative quality but cannot completely overcome these limitations.
In contrast, the proposed MMVM method uses neither an aggregated latent space nor modality-specific latent subspaces as in previous works. It only leverages a data-dependent prior distribution to regularize the learned posterior approximations.
A related line of work with different constraints is multiview VAEs \citep{bouchacourt2018,hosoya2018}.
In contrast to multimodal VAEs, multiview VAEs often use a single encoder and decoder for all views (thereby sharing the parameter weights between views).
While initial attempts also assume knowledge about the number of shared and independent generative factors, extensions \citep{locatello2020weakly,sutter2023mvhg,sutter_differentiable_2023} infer these properties during training.

\noindent
\textbf{Role of Prior in VAE Formulations.}
\citet{tomczak2017vaebb} first incorporated data-dependent priors into VAEs by introducing the VAMP-prior.
In contrast to \citet{tomczak2017vaebb}, who are primarily interested in better ELBO approximations, our focus is on learning better multimodal representations and overcoming the limitations faced in previous multimodal VAE works.
\citet{sutter_multimodal_2020} used a data-dependent prior combined with a joint posterior approximation defining a Jensen-Shannon divergence regularization based on the geometric mean.
However, their work lacks a rigorous derivation and relies on the suboptimal conditional generation during training \citep[][]{daunhawer2022}.
\citet{joy2022} also presented a multimodal VAE inspired by the VAMP-prior VAE. They leverage the VAMP prior to model missing modalities rather than using it as a regularization objective between multimodal samples, as we do in this work.
An additional line of work \citep[e.g., ][]{bhattacharyya2019,mahajan2020} leverages normalizing flows to increase the expressivity of the multimodal prior distribution, but this sacrifices the method's scalability.

%% file: methods.tex

\section{Background on Multimodal VAEs}
\label{sec:background}
\textbf{Problem Specification.}
We consider a dataset $\mathbb{X} = \{ \bm{X}^{(i)} \}_{i=1}^n$ where each $\bm{X}^{(i)} = \{ \bm{x}_1^{(i)}, \ldots, \bm{x}_M^{(i)} \}$ is a set of $M$ modalities $\bm{x}_m$ with latent variables $\bm{z} = \{ \bm{z}_1^{(i)}, \ldots, \bm{z}_M^{(i)} \}$. The modalities $\bm{x}_m^{(i)}$ could represent images of the same object taken from different camera angles, text-image pairs, or---as in this paper---neuroscience data from different animal subjects with shared experimental conditions and multiple medical measurements of a patient. When contextually clear, we remove the superscript $(i)$ to improve readability.

Inspired by variational autoencoders~\citep[VAEs, ][]{kingma2013}, we aim to learn an objective for representation learning while sharing information from different data modalities. 
For example, we would like to embed neuroscience data into a shared latent space to make brain activations comparable across subjects. At the same time, we want to avoid imposing assumptions on information sharing that are too strong to be able to take individual traits of the data modalities into account. As is typical in VAEs, this procedure involves a decoder (or likelihood) $p_\theta (\bm{X} \mid \bm{z})$, an encoder (or variational distribution) ${q_\phi(\bm{z} \mid \bm{X})}$, and a prior ${h(\bm{z} | \bm{X} )}$ that we allow to depend on the input. 

\noindent
\textbf{Data-Dependent Prior and Objective.} 
The VAE framework allows us to derive an ELBO-like learning objective $\mathcal{E}$ as follows
\begin{align}
    \mathcal{E} (\bm{X}) = &~ \mathbb{E}_{q_\phi(\bm{z} \mid \bm{X})} \left[ \log p_\theta (\bm{X} \mid \bm{z}) - \log \frac{q_\phi(\bm{z} \mid \bm{X})}{h(\bm{z} | \bm{X} )} \right]. \nonumber
\end{align}
Above, $\theta$ and $\phi$ denote the learnable model variational parameters. 
Importantly, our approach allows for an input-dependent prior $h(\bm{z} \mid \bm{X})$. Data-dependent priors can be justified from an empirical Bayes standpoint~\citep{efron2012large} and enable information sharing across data points with an intrinsic multimodal structure, as in our framework. They effectively amortize computation over many interrelated inference tasks. We stress that by making the prior data dependent, our model no longer allows for unconditional generation; however, this property can be restored by incorporating pseudo inputs \citep{tomczak2017vaebb}, hyper-priors \citep{sonderby2016ladder}, or ex-post density estimation techniques \citep{ghosh2019variational}.
We discuss the objective in more detail in \cref{app:mm_vamp_vae}, where we prove that the resulting objective is upper bounded by the mean squared reconstruction error, ensuring the existence of (local) optima and thus tractable optimization.

\noindent
\textbf{Encoder and Decoder.} We now specify our encoder and decoder assumptions. A simple encoder choice relies on a single neural network encoder that expects multi-modal inputs, but this approach fails if one or more modalities are missing~\citep{suzuki2022survey}. This shortcoming has motivated multiple approaches \citep{wu_multimodal_2018,shi_variational_2019,sutter2021} with separate encoders $q_\phi^m(\bm{z}_m| \bm{x}_m)$---one for each modality $m$---that are then \emph{aggregated} in the latent space, e.g., by using a product or mixture distribution. Samples drawn from the joint distribution, e.g., $q_\phi(\bm{z} \mid \bm{X}) = \frac{1}{M} \sum_{m=1}^M q_\phi^m (\bm{z} \mid \bm{x}_m)$, reconstruct all modalities:
\begin{align}
    \mathcal{E} (\bm{X}) = &~ \mathbb{E}_{q_\phi(\bm{z} \mid \bm{X})} \left[ \log h(\bm{z} \mid \bm{X}) \right] + \mathbb{E}_{q_\phi(\bm{z}  \mid \bm{X})} \left[\log \frac{p_\theta^m (\bm{x}_m \mid \bm{z})}{q_\phi(\bm{z} \mid \bm{X})}\right].
\end{align}

As argued and discussed in this paper, such aggregation can be overly restrictive. 
        Instead, this paper explores a different aggregation mechanism that preserves the individual encoders \emph{and} decoders for each modality. Hence, we assume independent decoders $p_\theta^m(\bm{x}_m| \bm{z}_m)$ for every modality $m$, assuming conditional independence of each modality given their latent representation \citep[see also][]{wu_multimodal_2018,shi_variational_2019,sutter2021}.

Following this assumption, we rewrite the objective $\mathcal{E}$ as
\begin{align}
\label{eq:ELBO-1}
    \mathcal{E} (\bm{X}) = &~ \mathbb{E}_{q_\phi(\bm{z} \mid \bm{X})} \left[ \log h(\bm{z} \mid \bm{X}) \right] + \sum_{m=1}^M \mathbb{E}_{q_\phi(\bm{z}_m  \mid \bm{x}_m)} \left[\log \frac{p_\theta^m (\bm{x}_m \mid \bm{z}_m)}{q^m_\phi(\bm{z}_m \mid \bm{x}_m)}\right].
\end{align}
Our assumptions imply that the likelihood and posterior entropy terms (the second term in \cref{eq:ELBO-1}) decouple across modalities, i.e. 
    $q_{\phi} (\bm{z} \mid \bm{X}) = \prod_{m=1}^M q_\phi^{m} (\bm{z}_m \mid \bm{x}_m)$ and
    $p_\theta (\bm{X} \mid \bm{z}) = \prod_{m=1}^M p_\theta (\bm{x}_m \mid \bm{z}_m)$. In contrast, the cross-entropy between the encoder and prior (the first term in \cref{eq:ELBO-1}) does not decouple and may result in information sharing across modalities. We specify further design choices in the next section.

\section{Multimodal Variational Mixture VAE}
\label{sec:method}
We propose the multimodal variational mixture-of-experts prior (MMVM) VAE, a novel multimodal VAE. 
The main idea is to design a mixture-of-experts prior across modalities that induces a soft-sharing of information between modality-specific latent representations rather than hard-coding this through an aggregation approach.

VAEs are an appealing model class that allows us to infer meaningful representations and preserve modality-specific information due to the reconstruction loss.
Contrastive learning approaches, on the other hand, have shown impressive results on representation learning tasks related to extracting shared information between modalities by maximizing the similarity of their representations \citep{radford_learning_2021}.
Contrastive approaches focus on the shared information between modalities, neglecting potentially useful modality-specific information.
We are interested in preserving modality-specific information, which is necessary to generate missing modalities conditionally.

Therefore, we leverage the idea of maximizing the similarity of representations for \emph{generative models}.
We propose a prior distribution that models the dependency between the different views and a new multimodal objective that encourages similarity between the unimodal posterior approximations $q_\phi^m (\bm{z}_m \mid \bm{x}_m)$ using the regularization term in the objective as a "soft-alignment" without the need for an aggregation-based joint posterior approximation.
We discuss objectives based on data-dependent priors in more detail in \cref{app:mm_vamp_vae}.

To this end, we define a data-dependent MMVM prior distribution in the form of a mixture-of-experts distribution of all unimodal posterior approximations
\begin{align}
\label{eq:mm_vamp_prior}
    h(\bm{z} \mid \bm{X}) & = \prod_{m=1}^M h (\bm{z}_m  \mid \bm{X}) \hspace{0.5cm} \text{where} \hspace{0.5cm} h (\bm{z}_m \mid \bm{X}) = \frac{1}{M} \sum_{\tilde{m}=1}^M q_\phi^{\tilde{m}} (\bm{z}_m \mid \bm{x}_{\tilde{m}}), \hspace{0.25cm} \forall ~ m \leq M.
\end{align}
This notation implies that we use the variational distributions of all modalities $\tilde{m}$ to construct a mixture distribution and then use the same mixture distribution as a prior for any modality $m$. Finally, we build the product distribution over the $M$ components.

Our construction of a variational mixture of posteriors is similar to the VAMP-prior of \citet{tomczak2017vaebb} that 
proposes the aggregate posterior $q(\bm{z}) \equiv \frac{1}{N}\sum_{i=1}^n q_\phi(\bm{z}\mid\bm{x}^{(i)})$ of a unimodal VAE as a prior. Note, however, that our approach considers mixtures in \emph{modality} space and not data space. In contrast to \citet{tomczak2017vaebb}, our variational mixture is conditioned on a specific instance $\bm{X}$ and, therefore, does not share information across different instances $\bm{X}^{(i)} \in \mathbb{X}$. Rather, we share information across the different modalities $\bm{x}_m^{(i)} \in  \bm{X}^{(i)}$ \emph{within} a given instance.
Intuitively, we build the \emph{aggregate posterior} in modality space and replicate this aggregate posterior over all modalities. We stress that this aggregate posterior differs from the standard definition as an average of variational posteriors over the empirical data distribution. Even though the prior appears factorized over the modality space, each factor still shares information across all data modalities by conditioning on the multimodal feature vector $\bm{X}$ (\cref{eq:mm_vamp_prior}).

\Cref{fig:exp_architectures} graphically illustrates the behavior of the proposed MMVM VAE compared to a set of independent VAEs and an aggregation-based multimodal VAE.
A set of independent VAEs (\cref{fig:exp_arch_ind_vaes}) cannot share information among modalities.
Aggregation-based VAEs (\cref{fig:exp_arch_agg_vaes}), in contrast, enforce a shared representation between the modalities.
The MMVM VAE (\Cref{fig:exp_arch_mmvamp_vaes}) enables the soft-sharing of information between modalities through its input data-dependent prior $h(\bm{z} \mid \bm{X})$.

\noindent
\textbf{Minimizing Jenson-Shannon Divergence.}
The "rate" term $R$ in the objective, i.e., the combination of variational entropy and cross-entropy, reveals a better understanding of the effect of the mixture prior.
Defining $R = KL(q_\phi(\bm{z} \mid \bm{X})|| h(\bm{z}|\bm{X}))$ where $KL$ denotes the Kullback-Leibler divergence, the factorization in \cref{eq:mm_vamp_prior} implies that 
\begin{align}
    R = \sum_{m=1}^M KL \left(q_\phi^m (\bm{z}_m | \bm{x}_m) || \frac{1}{M} \sum_{\tilde{m}}^M q_\phi^{\tilde{m}} (\bm{z}_m | \bm{x}_{\tilde{m}}) \right) = M \cdot JS (q_\phi^1 (\bm{z}_1 | \bm{x}_1), \ldots, q_\phi^M (\bm{z}_M | \bm{x}_M)), \nonumber
\end{align}
where $JS(\cdot)$ is the Jensen-Shannon divergence between $M$ distributions \citep{lin_divergence_1991}.
Hence, maximizing the objective $\mathcal{E}(\bm{X})$ of the proposed MMVM VAE is equal to minimizing $M$ times the JS divergence between all the unimodal posterior approximations $q_\phi^m (\bm{z}_m \mid \bm{x}_m )$.
Minimizing the Jensen-Shannon divergence between the posterior approximations is directly related to pairwise similarities between posterior approximation distributions of positive pairs, similar to contrastive learning approaches but in a generative approach.

\subsection{Optimality of the MMVM Prior}
\cref{thm:optimal_prior_distribution} shows that \cref{eq:mm_vamp_prior} is \emph{optimal} in the sense that it is the unique minimizer of the cross entropy between our chosen variational distribution and an arbitrary prior.

\begin{lemma}[]
\label{thm:optimal_prior_distribution}
   The expectation on the right-hand side of \Cref{eq:ELBO-1} is maximized when for each $m\in \{1, \cdots, M\}$,  the prior  $h(\bm{z}_m | \bm{X})$ is equal to the aggregated posterior of a multimodal sample given on the first line of \cref{eq:mm_vamp_prior}.
\end{lemma}
\begin{proof}
Since the cross-entropy term in \cref{eq:ELBO-1} involves an expectation over the data $\bm{X}$ and both $q_\phi(\bm{z} \mid \bm{X})$ and $h(\bm{z} \mid \bm{X})$ depend on $\bm{X}$, we can prove the identity for a given value of $\bm{X}$.

We exploit the factorization of both the variational posterior and the prior over the modalities. Interpreting the cross-entropy between the variational distribution and prior as a functional $F$ of the prior $h$, we have
\begin{align*}
F[h(\bm{z}|\bm{X}))] & \equiv \mathbb{E}_{q_\phi(\bm{z}|\bm{X})}\left[\log h(\bm{z}|\bm{X}))\right] = \mathbb{E}_{\prod_{m=1}^M q_\phi^m(\bm{z}_m|\bm{x}_m)}\left[\log \prod_{m=1}^M h(\bm{z}_m|\bm{X}))\right] \\
& = \sum_{m=1}^M \mathbb{E}_{q_\phi^m(\bm{z}_m|\bm{x}_m)}\left[\log h(\bm{z}_m|\bm{X}))\right] = M \cdot \mathbb{E}_{\frac{1}{M} \sum_{\tilde{m}=1}^M q_\phi^{\tilde{m}}(\bm{z}_m|\bm{x}_{\tilde{m}})}\left[\log h(\bm{z}_m|\bm{X}))\right].
\end{align*}
As a result, we see that $F[h(\cdot)]$ is an expectation over a mixture distribution. We can solve for the optimal distribution $h (\cdot)$ by adding a Lagrange multiplier that enforces $h(\cdot)$ normalizes to one:
\begin{align*}
\max_{h(\bm{z}_m \mid \bm{X})} F[h(\bm{z}_m \mid \bm{X})] + \gamma \left( \int h (\bm{z}_m  \mid  \bm{X}) d\bm{z}_m -1 \right) = \max_{h(\bm{z}_m \mid \bm{X})} \mathcal{L}[h, \gamma]\nonumber
\end{align*}

To maximize the Lagrange functional $\mathcal{L}[h, \gamma]$, we compute its (functional) derivatives with respect to $h (\bm{z}_m | \bm{X})$ and $\gamma$.
\begin{align}
    \frac{\partial \mathcal{L} [h (\bm{z}_m|\bm{X}), \gamma]}{\partial h (\bm{z}_m|\bm{X})} =
    &~ \frac{\frac{1}{M}\sum_{\tilde{m}=1}^M q_\phi^{\tilde{m}} (\bm{z}_m | \bm{X})}{h (\bm{z}_m| \bm{X})} + \gamma \overset{!}{=} 0 \nonumber \\
    \frac{\partial \mathcal{L} [h (\bm{z}_m | \bm{X}), \gamma]}{\partial \gamma} =&~ \int_{\bm{z}_m} h (\bm{z}_m | \bm{X}) d\bm{z}_m - 1 \overset{!}{=} 0 \nonumber
\end{align}
The first condition implies that for \emph{any} value of $\bm{z}_m$, the ratio between the mixture distribution and the prior is constant, while the second condition demands that the prior be normalized. These conditions can only be met if the prior \emph{equals} the mixture distribution, which proves the claim.   
\end{proof}

%% file: experiments.tex
\begin{figure*}
    \centering
    \begin{subfigure}[t]{1.0\textwidth}
        \hspace{0.55cm}
        \includegraphics[width=0.945\textwidth]{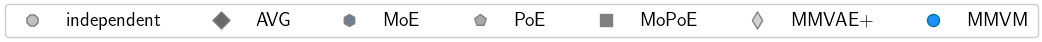}
    \end{subfigure}
    \centering
    \begin{subfigure}[t]{0.3\textwidth}
        \includegraphics[width=1.0\textwidth]{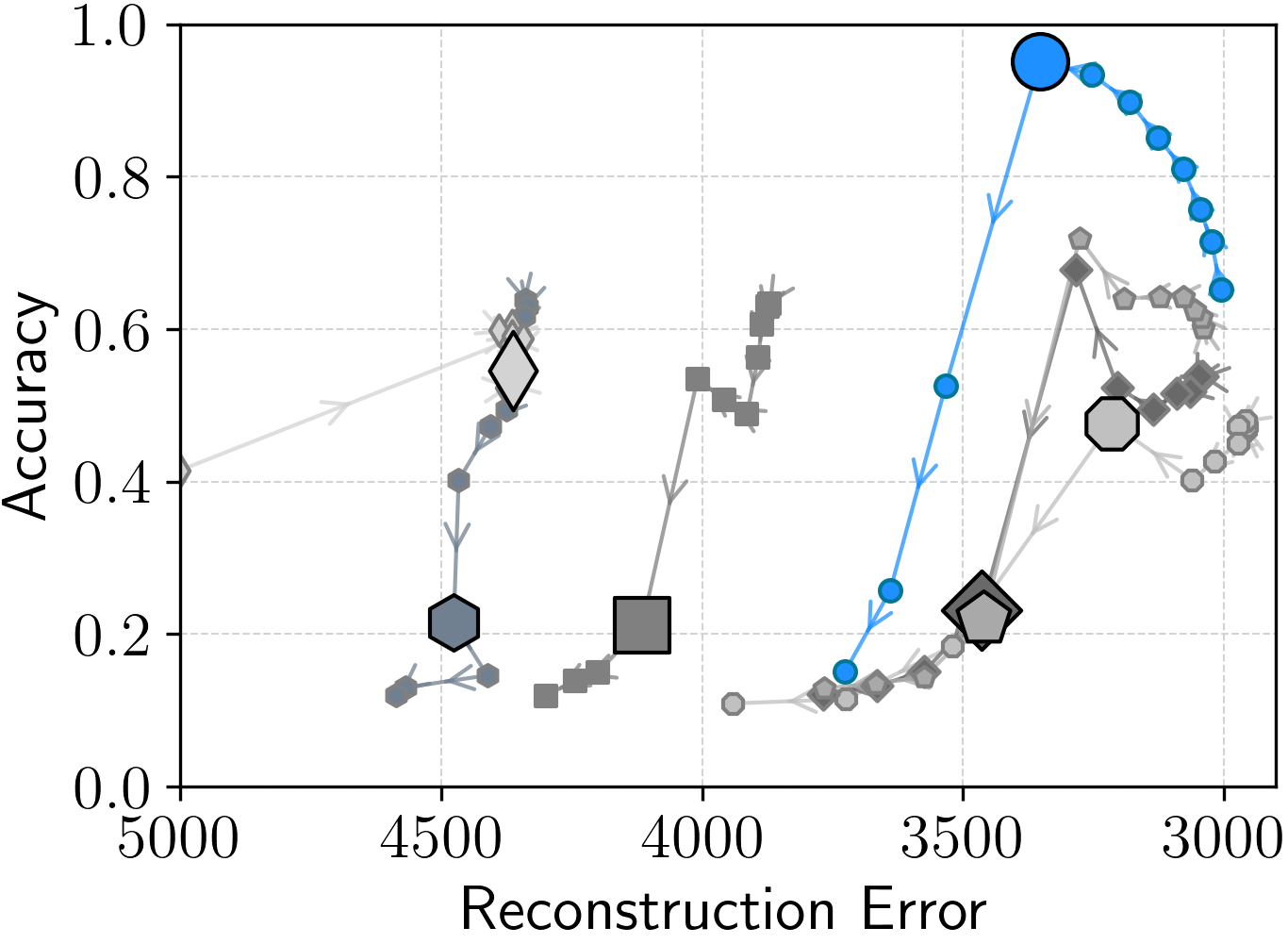}
        \caption{LR: translated PoyMNIST}
        \label{fig:exp_benchmarks_downstream_polymnist}
    \end{subfigure}
    \hspace{0.5cm}
    \centering
    \begin{subfigure}[t]{0.3\textwidth}
        \includegraphics[width=1.0\textwidth]{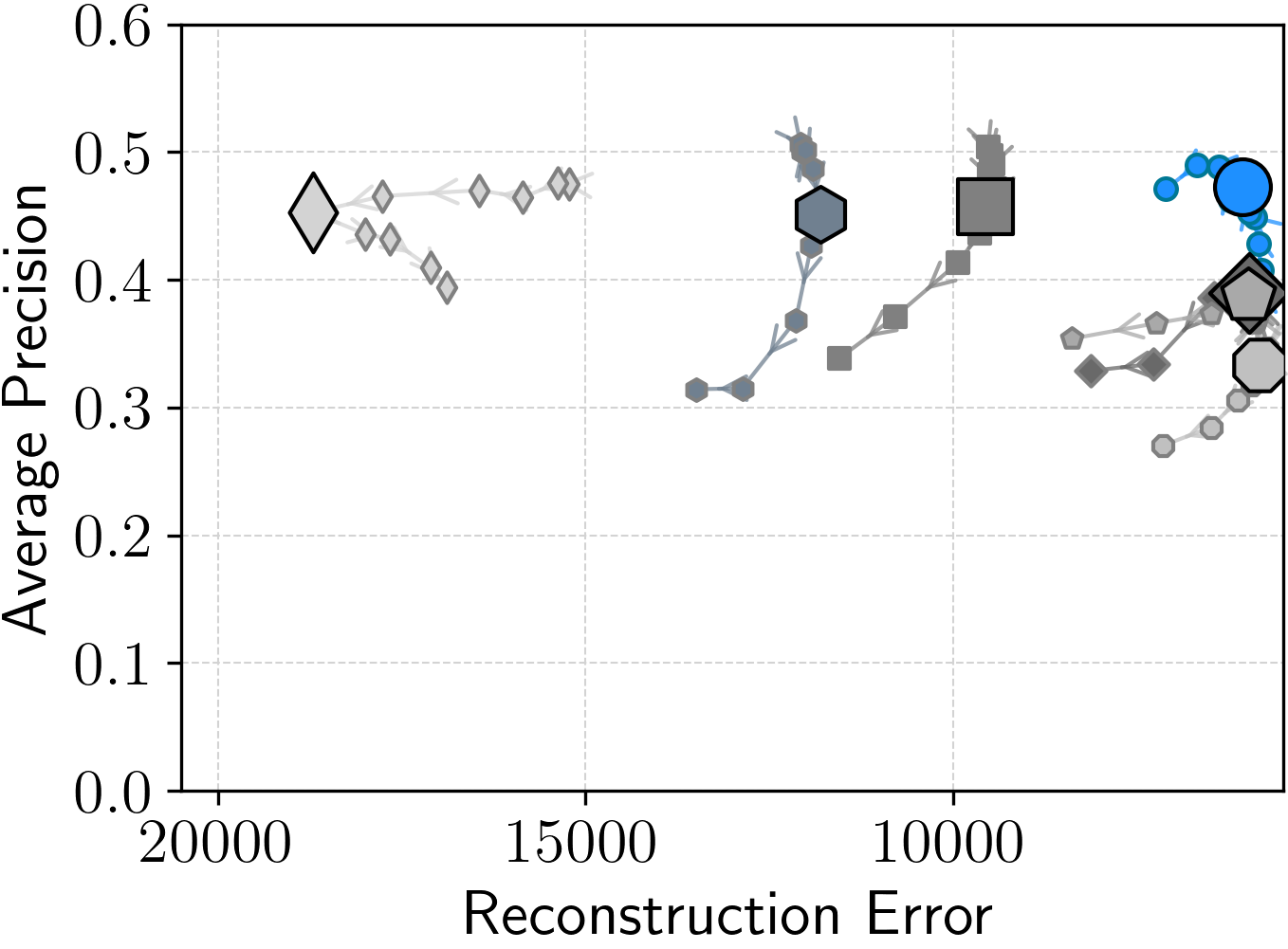}
        \caption{LR: Bimodal CelebA}
        \label{fig:exp_benchmarks_downstream_celeba}
    \end{subfigure}
    \hspace{0.25cm}
    \begin{subfigure}[t]{0.3\textwidth}
        \includegraphics[width=1.0\textwidth]{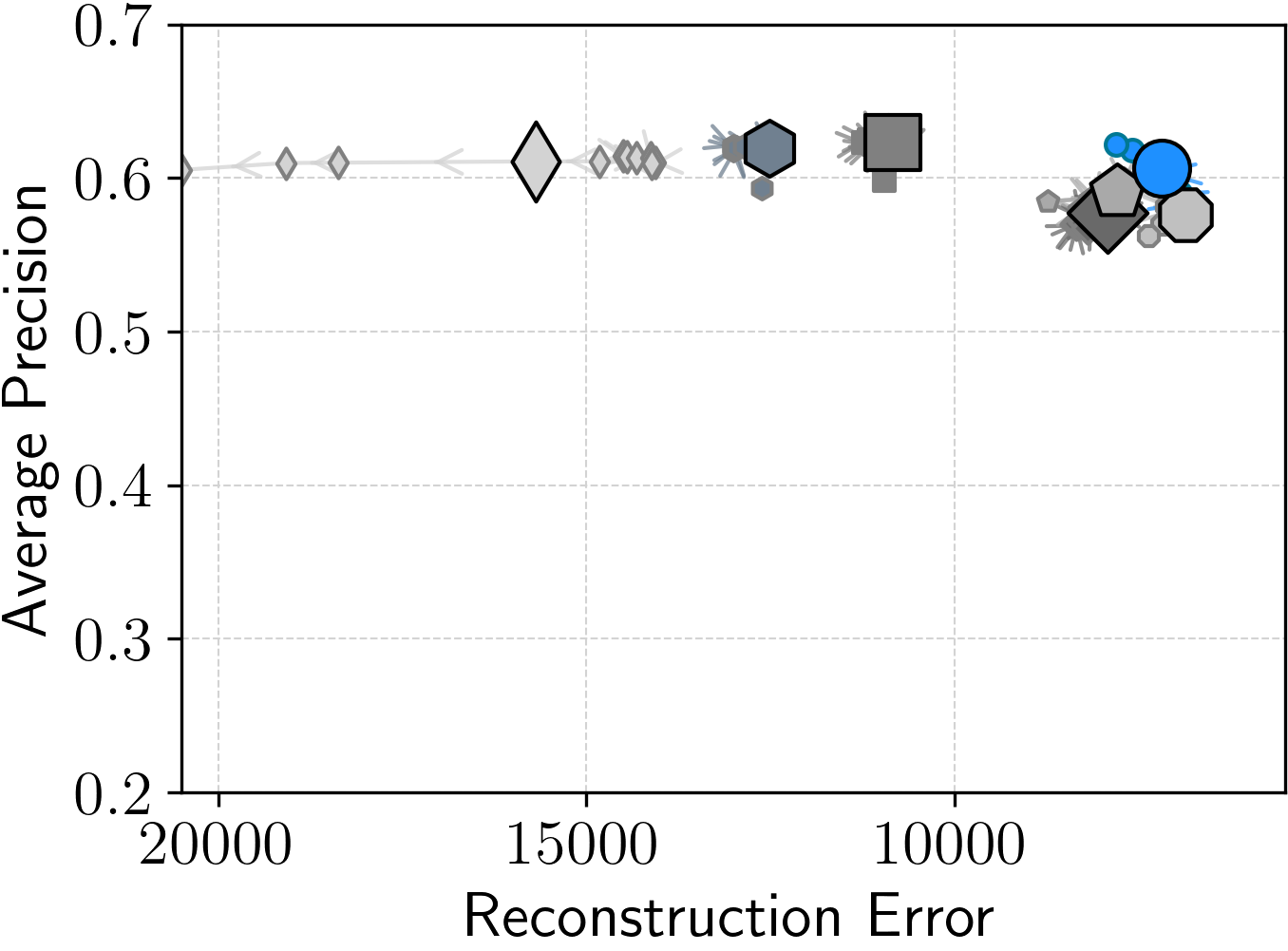}
        \caption{LR: CUB    }
        \label{fig:exp_benchmarks_downstream_cub}
    \end{subfigure}
    \hspace{0.25cm}
    \centering
    \begin{subfigure}[t]{0.3\textwidth}
        \includegraphics[width=1.0\textwidth]{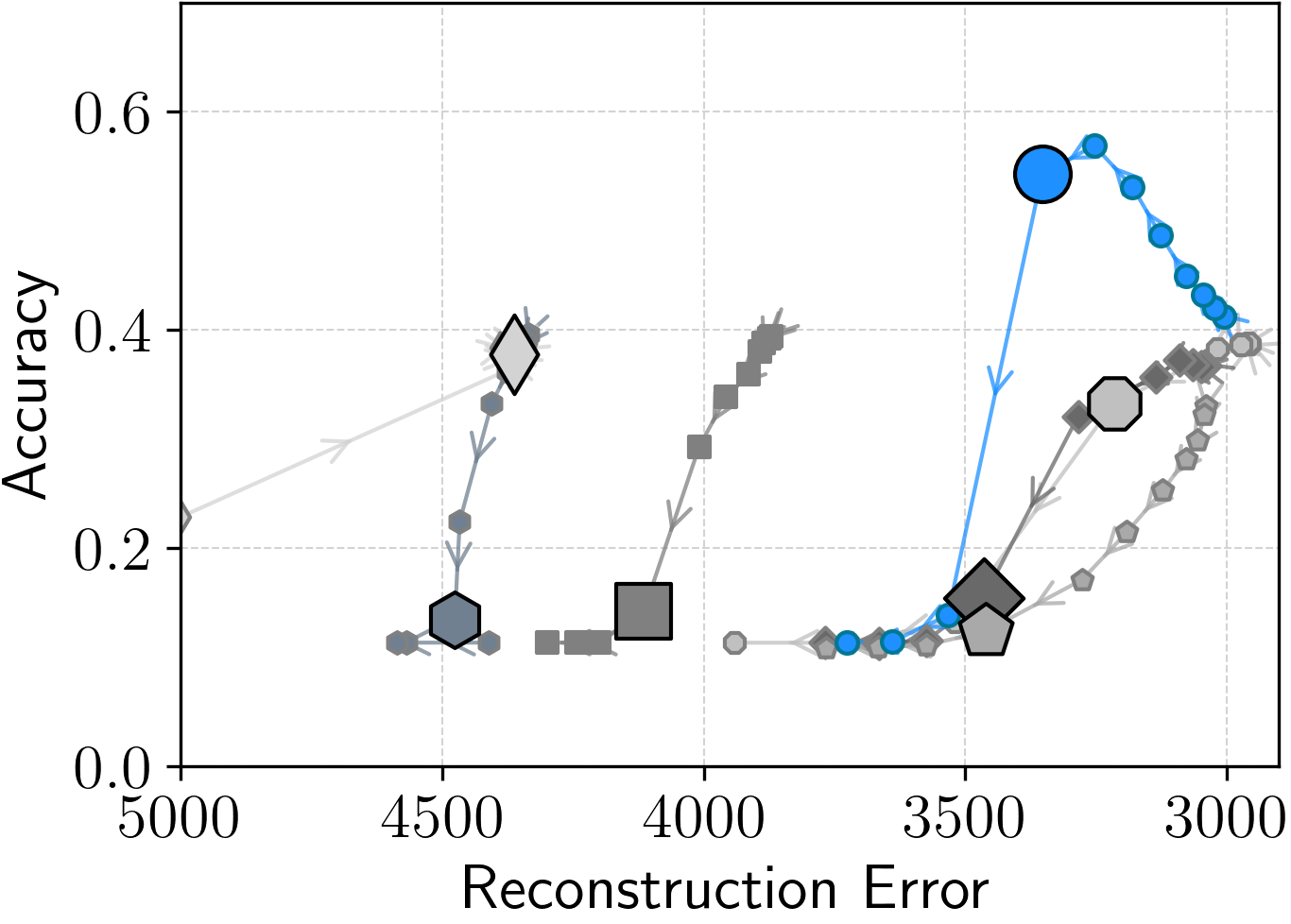}
        \caption{Coh: translated PoyMNIST}
        \label{fig:exp_benchmarks_coherence_polymnist}
    \end{subfigure}
    \hspace{0.5cm}
    \centering
    \begin{subfigure}[t]{0.3\textwidth}
        \includegraphics[width=1.0\textwidth]{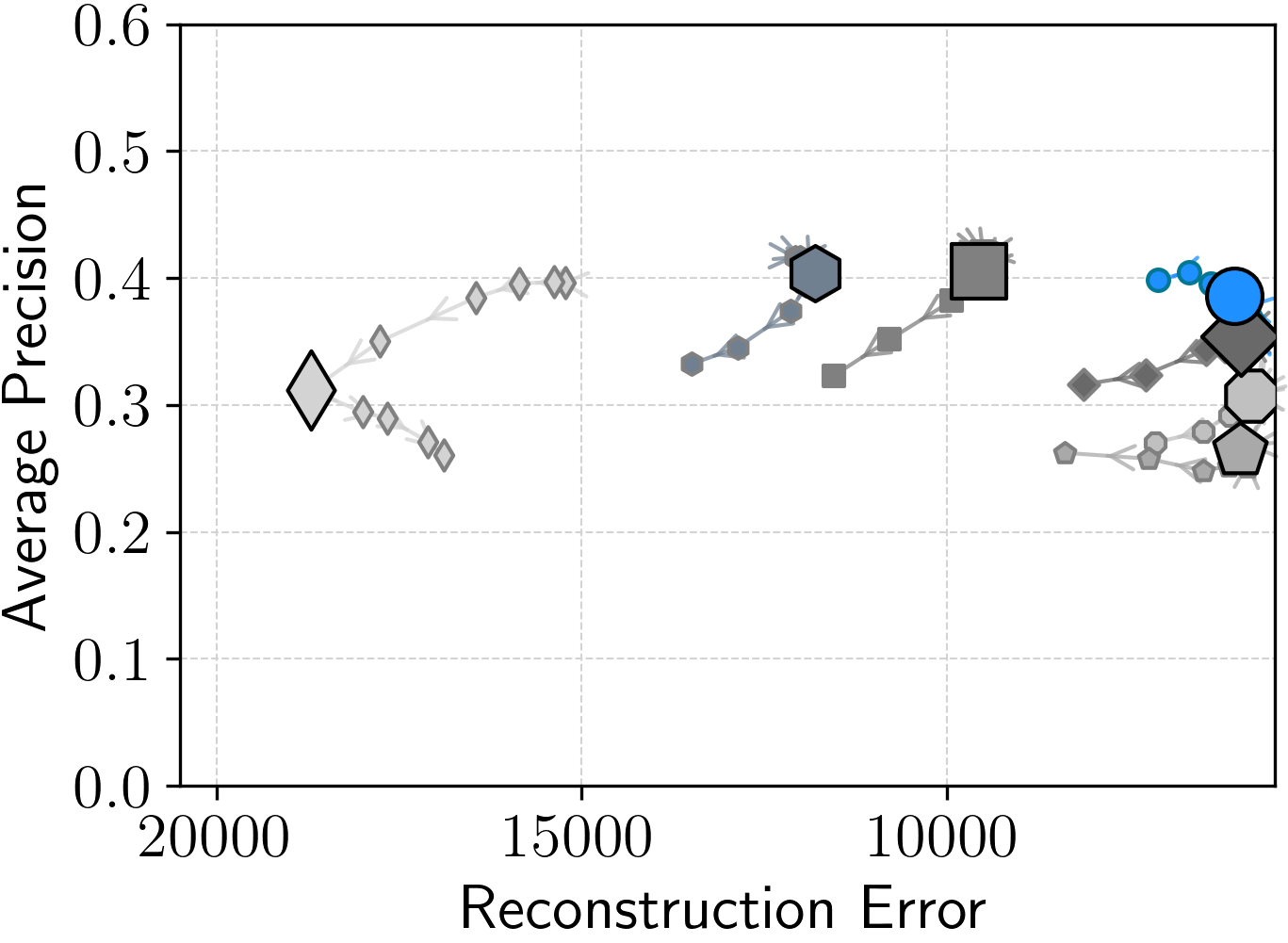}
        \caption{Coh: Bimodal CelebA}
        \label{fig:exp_benchmarks_coherence_celeba}
    \end{subfigure}
    \hspace{0.25cm}
    \begin{subfigure}[t]{0.3\textwidth}
        \includegraphics[width=1.0\textwidth]{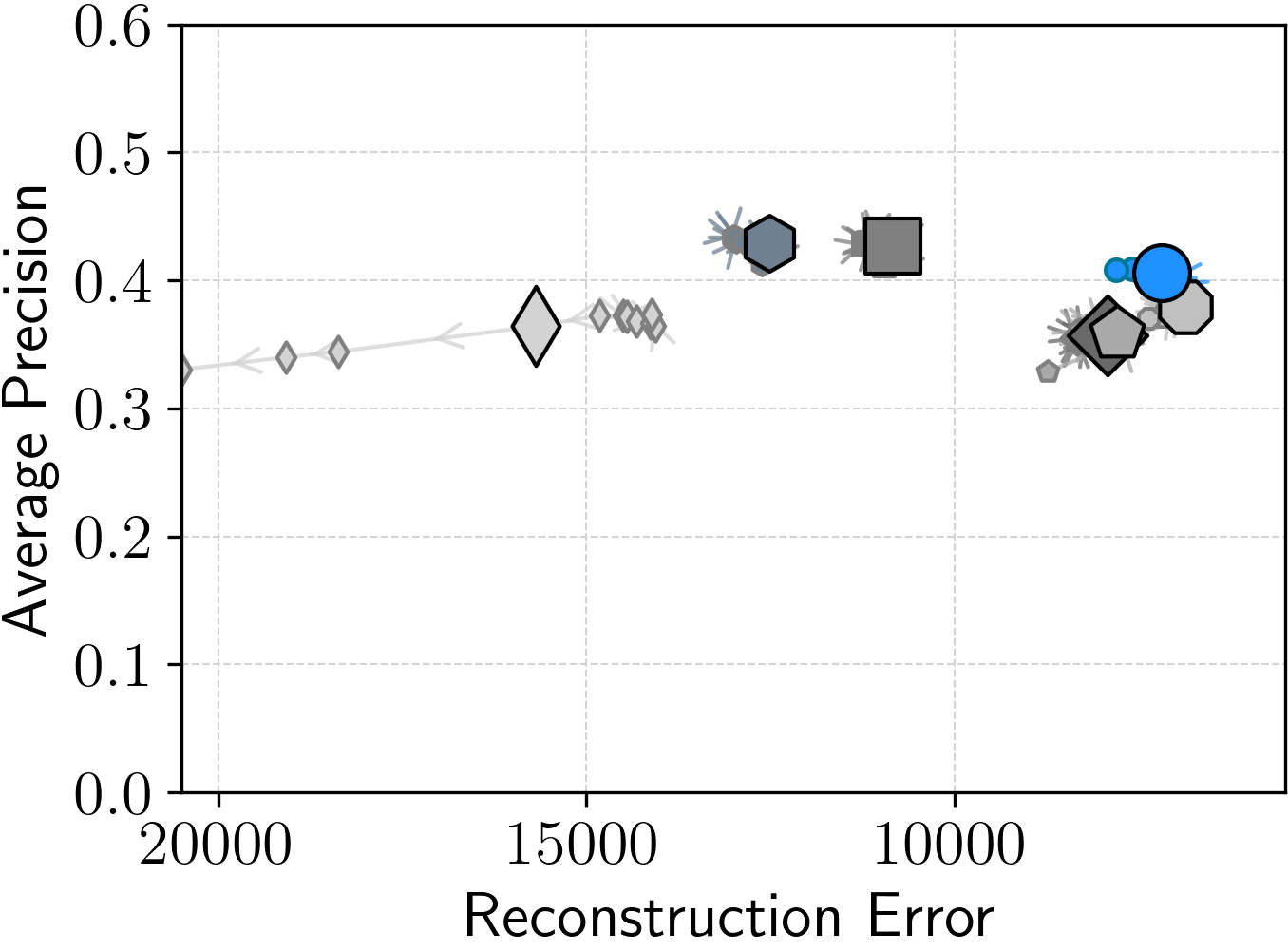}
        \caption{Coh: CUB}
        \label{fig:exp_benchmarks_coherence_cub}
    \end{subfigure}
    \caption{
    Results on the benchmark datasets translated PolyMNIST, bimodal CelebA, and CUB. An optimal model would be in the top right corner with low reconstruction error and high classification performance.
    The proposed MMVM method either achieves a higher classification performance, latent representation (LR, \cref{fig:exp_benchmarks_downstream_polymnist,fig:exp_benchmarks_downstream_celeba,fig:exp_benchmarks_downstream_cub}) or coherence of generated samples (Coh, \cref{fig:exp_benchmarks_coherence_polymnist,fig:exp_benchmarks_coherence_celeba,fig:exp_benchmarks_coherence_cub}), with the same reconstruction loss or the same classification performance with lower reconstruction loss.
    Every point averages runs over multiple seeds and a specific $\beta$ value (see \cref{sec:exp_benchmarks}).
    }
    \label{fig:exp_benchmarks}
    \vspace{-0.25cm}
\end{figure*}

\section{Experiments}
\label{sec:experiments}
We evaluate the proposed MMVM VAE on three benchmark datasets and two challenging real-world applications.
\vspace{-0.25cm}

\subsection{Benchmark Datasets}
\label{sec:exp_benchmarks}
We first compare the proposed method against five strong VAE-based learning approaches on three frequently used multimodal benchmark datasets\footnote{The code for the experiments on the benchmark datasets can be found here: \url{https://github.com/thomassutter/mmvmvae}.}.

\noindent
\textbf{Datasets.} 
We perform benchmark experiments on the translated PolyMNIST \citep{daunhawer2022,sutter2021}, the bimodal CelebA \citep{sutter_multimodal_2020}, and the CUB image-captions \citep{shi_variational_2019} dataset.
The translated PolyMNIST dataset uses multiple instances of the MNIST dataset \citep{lecun1998} with different backgrounds but shared digits. The digits of the different modalities are randomly translated, so we cannot predict their location across modalities.
Bimodal CelebA extends the CelebA dataset \citep{liu_deep_2015} with an additional text modality based on the attributes describing the faces.
Similarly, the CUB image-captions dataset extends the Caltech bird dataset \citep{WahCUB_200_2011} with human-generated captions describing the images.
Please see \cref{app:experiments} for more details regarding the datasets.

\noindent
\textbf{Baselines.} We evaluate our proposed method against a set of jointly-trained independent VAEs \citep[\textit{independent}, ][]{kingma2013}, different aggregation-based multimodal VAEs, and an aggregation-based multimodal VAE with additional modality-specific latent spaces.
For the set of independent VAEs, there is no interaction or regularization between the different modalities during training.
For the aggregation-based multimodal VAEs, we use a multimodal VAE with a joint posterior approximation function.
We evaluate four different aggregation functions: a simple averaging \citep[\textit{AVG}, ][]{hosoya2018}, a product-of-experts \citep[\textit{PoE}, ][]{wu_multimodal_2018}, a mixture-of-experts \citep[\textit{MoE}, ][]{shi_variational_2019}, and a mixture-of-products-of-experts \citep[\textit{MoPoE}, ][]{sutter2021}.
For the multimodal VAE with modality-specific subspaces, we use MMVAE+ method \citep[\textit{MMVAE+}, ][]{palumbo_mmvae_2022}.
We train all VAE methods as $\beta$-VAEs \citep{higgins_beta-vae_2016}, where $\beta$ is an additional hyperparameter weighting the rate term $R$ of the VAE (see \cref{sec:method}).
\Cref{app:experiments_polymnist} provides the implementation details of the proposed method and the baseline alternatives.

\begin{figure*}
    \centering
    \begin{subfigure}[t]{\textwidth}
        \hspace{1.125cm}
        \includegraphics[width=0.835\textwidth]{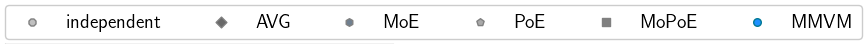}
    \end{subfigure}
    \centering
    \begin{subfigure}[t]{0.425\textwidth}
        \includegraphics[width=1.0\textwidth]{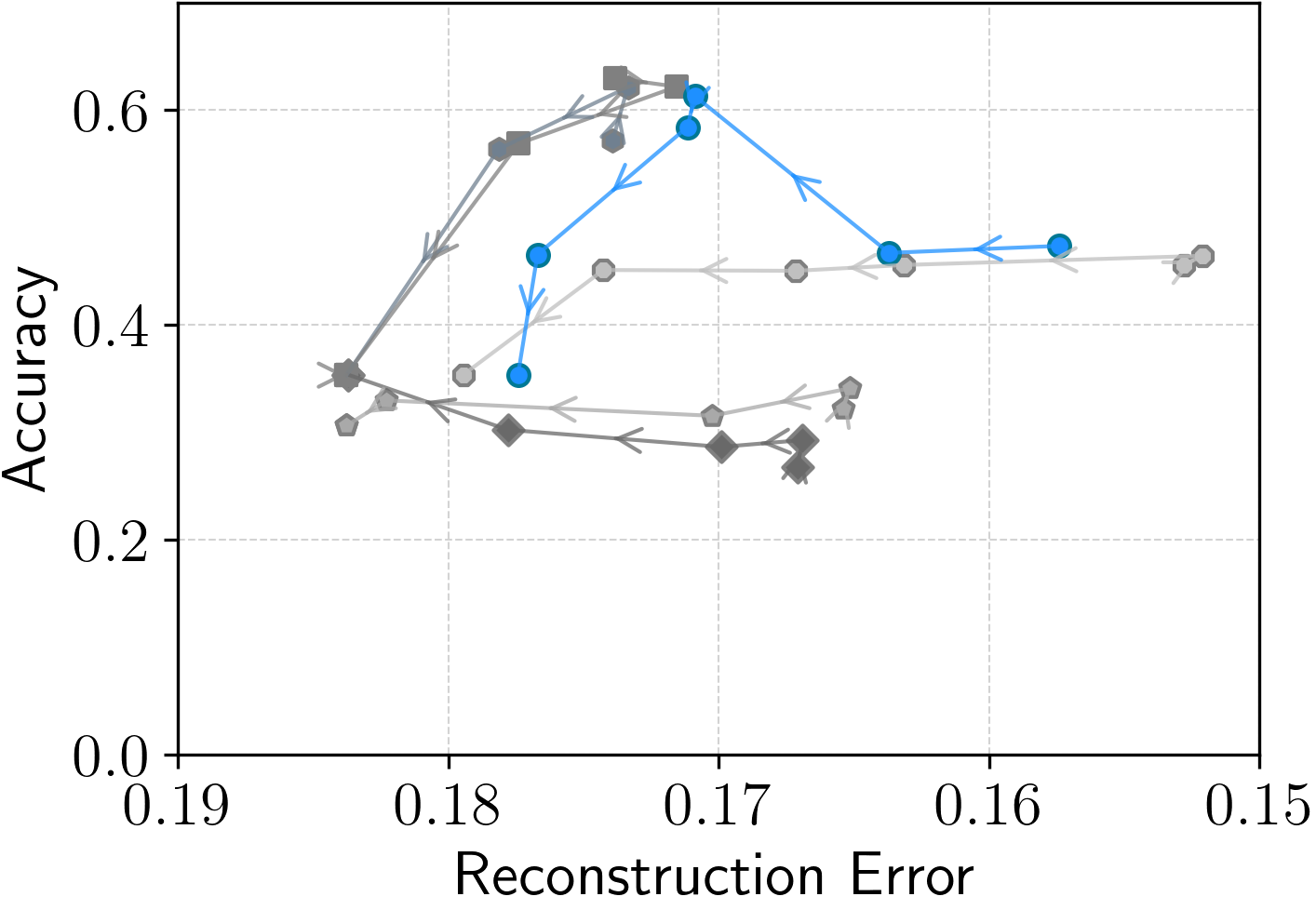}
        \caption{Latent Representation Classification}
        \label{fig:exp_rats_downstream_recloss}
    \end{subfigure}
    \hspace{0.5cm}
    \centering
    \begin{subfigure}[t]{0.425\textwidth}
        \includegraphics[width=1.0\textwidth]{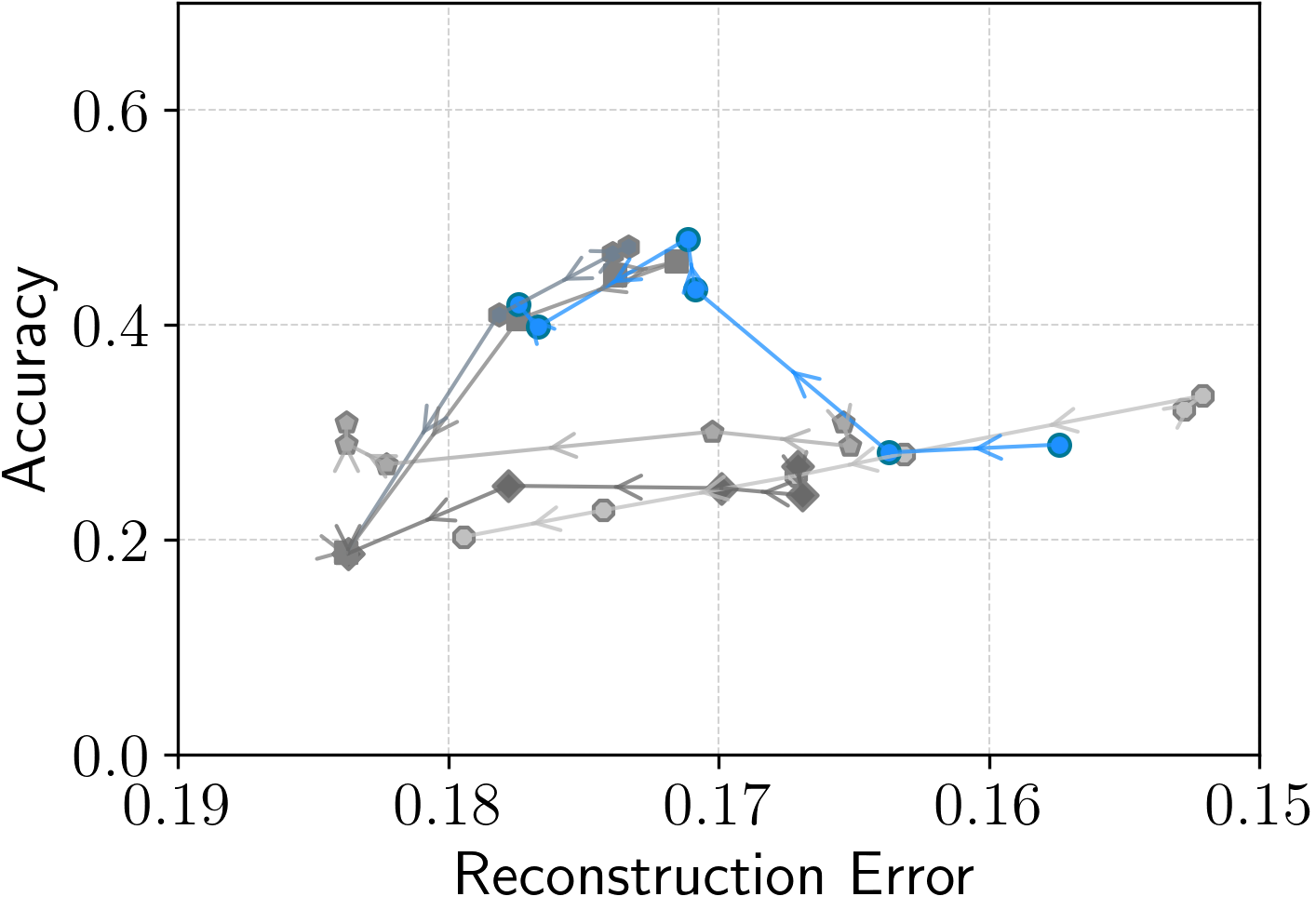}
        \caption{Conditional Generation Coherence}
        \label{fig:exp_rats_coherence_recloss}
    \end{subfigure}
    \hspace{0.25cm}
    \caption{
    Results based on a memory experiment conducted on five rats, each regarded as a separate modality.
    We report the performance of the latent representation classification and the conditional generation coherence against the reconstruction loss for different $\beta$ values for the different VAE methods.
    Every point in the figures represents a specific $\beta$ value, where $\beta = (10^{-5}, 10^{-4}, 10^{-3}, 2.5\times 10^{-3}, 5\times 10^{-3}, 10^{-2})$.
    An optimal model would be in the top right corner.
    }
    \label{fig:exp_rats}
    \vspace{-0.5cm}
\end{figure*}

\noindent
\textbf{Evaluation.} 
We test the methods' ability to infer meaningful representation when only a subset of modalities is available.
In addition, we evaluate all methods in terms of their data imputation performance, where we withhold a subset of modalities at test time and conditionally generate them from the shared latent representations. In this imputation task, we assess whether the generated modalities are both of high quality and display the expected shared information, which we refer to as \emph{coherence}. 
We assess the quality of the learned latent representations using linear classifiers trained on representations of the training set and the coherence using nonlinear classifiers trained on original samples of the training set\footnote{We use the same evaluation procedure as in previous work \citep[e.g.,][]{shi_variational_2019,sutter2021}}.
We use the reconstruction error as a proxy for how well each method learns the underlying data distribution.
We assess each method by relating their achieved reconstruction error to either the learned latent representation 
classification or the coherence (\cref{fig:exp_benchmarks}).
We evaluate all methods for multiple values of $\beta$, where the average performance over multiple seeds with a single $\beta$ value leads to a single point in \cref{fig:exp_benchmarks}.
Evaluating the methods for different values of $\beta$ considers that the optimal $\beta$ value is model- and data-dependent.
In addition, increasing $\beta$ emphasizes a more structured latent space \citep{higgins_beta-vae_2016}. Hence, highlighting the dynamics between reconstruction error and classification performance for different multimodal objectives provides additional insights.
We chose $\beta \in \{ 2^{-8}, \ldots, 2^3 \}$ on the PolyMNIST dataset, $\beta \in \{ 2^{-5}, \ldots, 2^4 \}$ on the CelebA dataset, and $\beta \in \{ 2^{-2}, \ldots, 2^2 \}$ on the CUB dataset.
In all figures, the arrows go from small to large values of $\beta$.
See \Cref{app:experiments_evaluation} for more details on the evaluation metrics.

\noindent
\textbf{Results.}
\Cref{fig:exp_benchmarks} shows that the proposed MMVM VAE consistently outperforms the other VAE-based methods (\textit{independent, AVG, MoE, PoE, MoPoE}) on all datasets and both tasks.
We can show that our method overcomes the limitations of aggregation-based multimodal VAEs on translated PolyMNIST described in \citet{daunhawer_self-supervised_2020}.
Also, MMVM VAE can learn meaningful representations and generate coherent samples across different modalities while achieving high reconstruction quality for both text-image datasets bimodal CelebA and CUB image-captions.
The coherent conditional generation is especially surprising as the proposed MMVM VAE decoder is never confronted with a sample from another modality.
For all benchmark datasets, the proposed MMVM either achieves better latent representation classification and coherence performance with a similar reconstruction loss or lower reconstruction loss with a similar classification performance than other multimodal VAE approaches.
In summary, we can show that the newly proposed MMVM VAE overcomes the limitations of previous aggregation-based approaches to multimodal learning (see \cref{sec:related_work}) and outperforms previous works on all three benchmark datasets.
We provide more results on the benchmark datasets in \cref{app:experiments}.

\subsection{Hippocampal Neural Activities}
\label{sec:exp_rateegs}

\textbf{Dataset.}
Temporal organization is crucial to memory, affecting various perceptual, cognitive, and motor processes. While we have made progress in understanding the brain's processing of the spatial context of memories, our knowledge of their temporal structure is still very limited. To this end, neuroscientists have recorded neural activity in the hippocampus of rats performing a complex sequence memory task \citep{allen16,shahbaba_hippocampal_2022}.
More specifically, this study investigates the temporal organization of memory and behavior by recording neural activity from the dorsal CA1 region of the hippocampus. 
Briefly, the task involves presenting rats with a repeated sequence of non-spatial events (four stimuli: odors A, B, C, D) at a single port \citep{shahbaba_hippocampal_2022}.
Since the same experimental setup was conducted across all rats, we consider the rats as different ''modalities'' and apply our proposed MMVM method to extract meaningful latent representations. While the existence (and importance) of subject-specific effects is well-known in neuroscience, such effects tend to be treated as unexplained variance because of the lack of the required analytical tools to extract and utilize this information properly.

\begin{figure*}[t!]
    \centering
    \begin{subfigure}[t]{0.95\textwidth}
        \centering
        \includegraphics[width=0.85\textwidth]{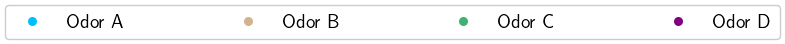}
    \end{subfigure}
    \centering
    \begin{subfigure}[t]{0.22\textwidth}
        \includegraphics[width=1.0\textwidth]{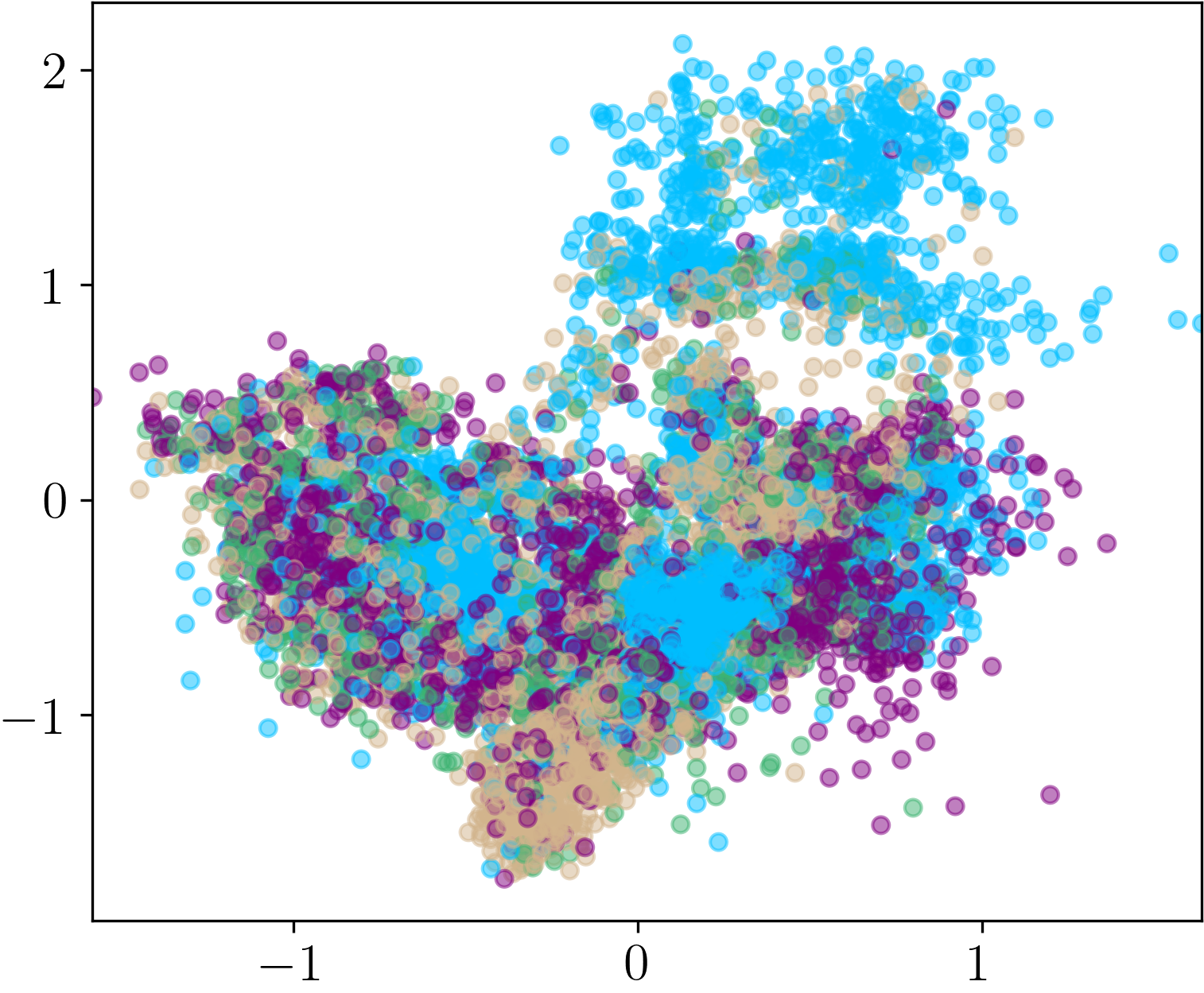}
        \caption{independent}
        \label{fig:rats_latent_odor_independent}
    \end{subfigure}
    \hspace{0.25cm}
    \centering
    \begin{subfigure}[t]{0.22\textwidth}
        \includegraphics[width=1.0\textwidth]{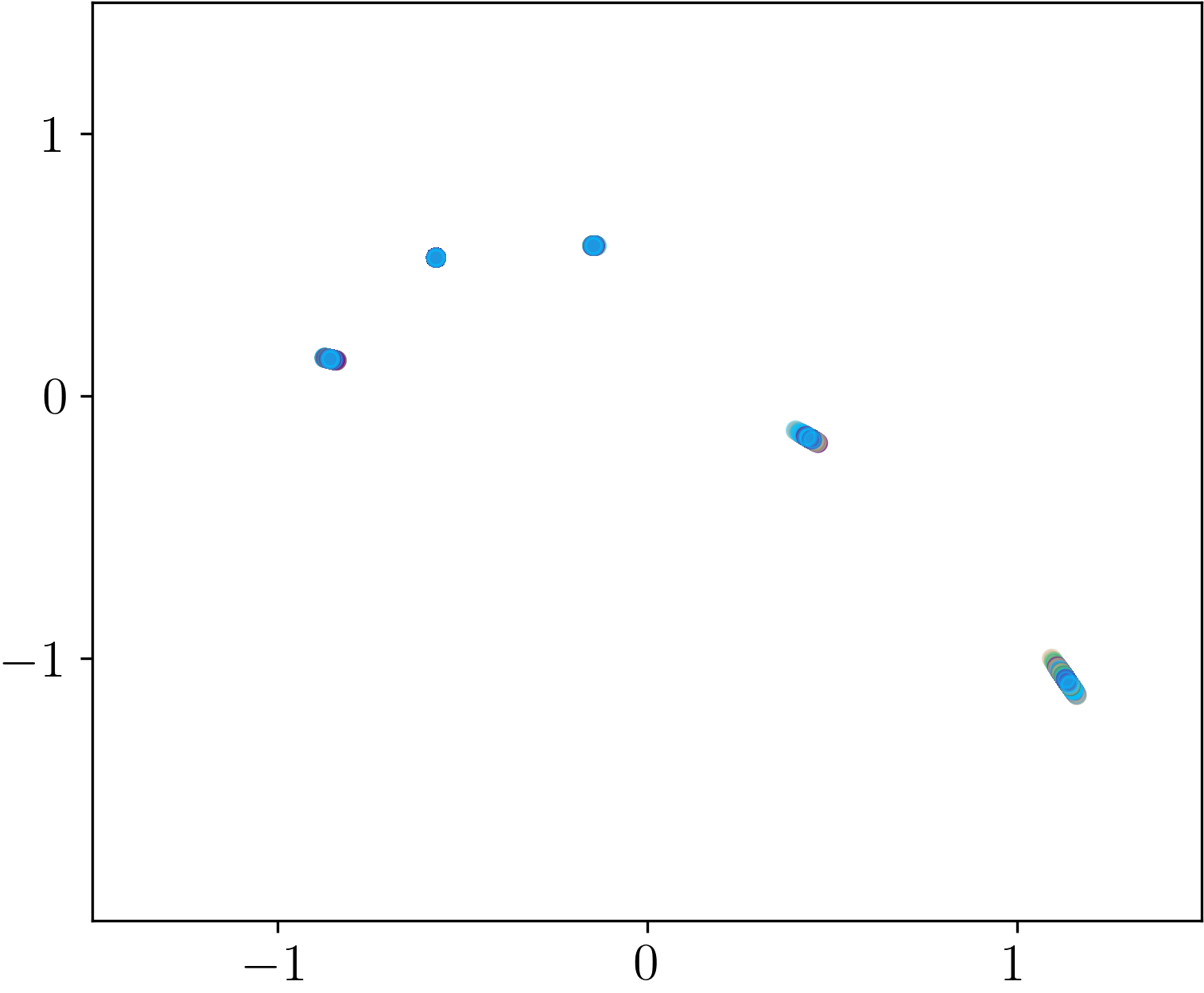}
        \caption{AVG}
        \label{fig:rats_latent_odor_joint}
    \end{subfigure}
    \hspace{0.25cm}
    \centering
    \begin{subfigure}[t]{0.22\textwidth}
        \includegraphics[width=1.0\textwidth]{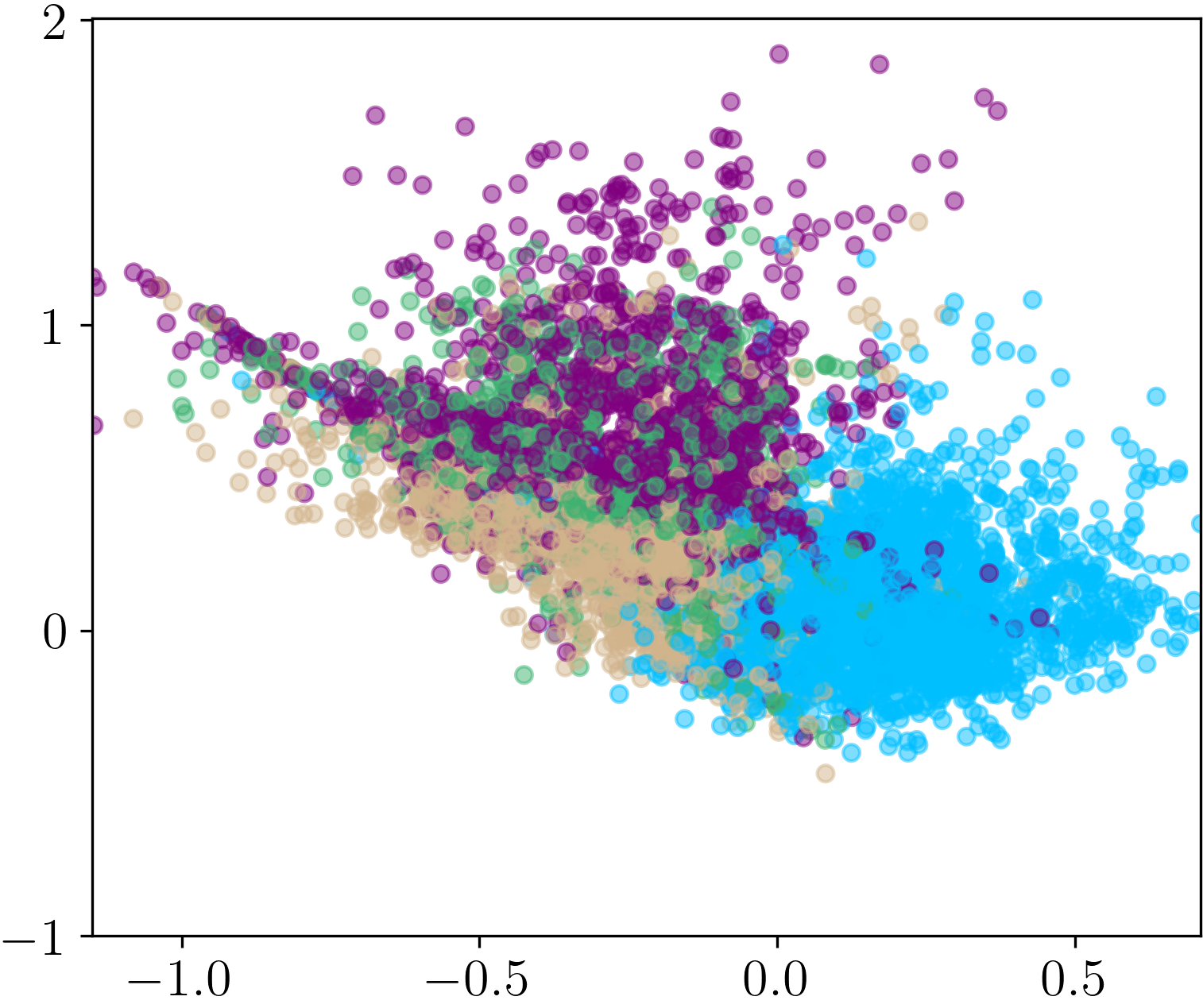}
        \caption{MoPoE}
        \label{fig:rats_latent_odor_mopoe}
    \end{subfigure}
    \hspace{0.25cm}
    \centering
    \begin{subfigure}[t]{0.22\textwidth}
        \includegraphics[width=1.0\textwidth]{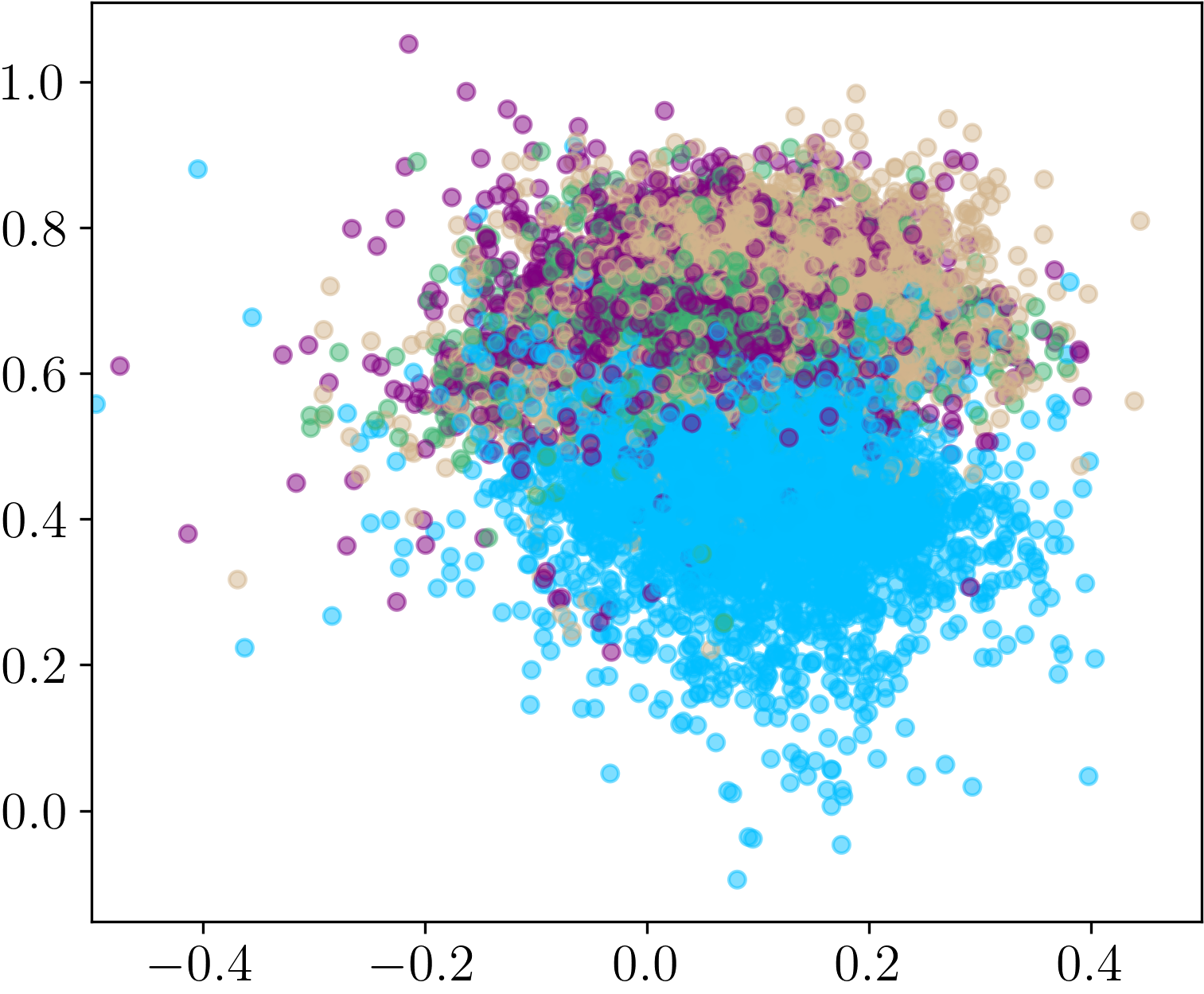}
        \caption{MMVM}
        \label{fig:rats_latent_odor_mixedprior}
    \end{subfigure}
    \hspace{0.25cm}
    \caption{Latent neural representation during a memory experiment. Each model's performance is evaluated based on its own optimal $\beta$ value (0.00001, 0.01, 0.00001, 0.001 for independent, AVG, MoPoE, and MMVM respectively) in terms of the unimodal latent representation classification accuracy according to \Cref{fig:exp_rats_downstream_recloss}. Our method can distinguish the odor stimuli in the latent space with a clear separation of odors similar to MoPoE VAE (4 different colors). Conversely, unimodal and AVG models failed to combine multi-views as the odor separation only occurred within single views.}
    \label{fig:rats_latent_odor}
    \vspace{-0.5cm}
\end{figure*}
\noindent
\textbf{Results.}
Our proposed MMVM method outperforms\footnote{The code for the hippocampal neural activity experiments can be found here: \url{https://github.com/yangmeng96/mmvmvae-hippocampal}} most previous works regarding learned latent representations and conditional generation coherence.
Only the \textit{MoPoE} VAE achieves a classification performance comparable to the MMVM method but with a higher reconstruction loss.
\Cref{fig:exp_rats_downstream_recloss} shows the separation of the latent representation (measured by the accuracy of a multinomial logistic regression classifier) against the reconstruction loss.
Similar to the results on the benchmark datasets, the proposed MMVM VAE outperforms previous works by providing a clear separation of odors in the latent space while maintaining a low reconstruction loss. 
\Cref{fig:exp_rats_coherence_recloss} compares the coherence of conditional generation accuracy against the reconstruction loss. As before, our proposed approach outperforms the alternatives.
The proposed MMVM method allows learning an aligned latent representation across different modalities.
We show the 2-dimensional latent representations for every rat and four VAE encoders in \Cref{fig:rats_latent_odor}.
Each dot is the two-dimensional latent representation of a 100 ms sub-window of one odor trial for one rat and is colored according to its ground truth odor value.
\Cref{fig:rats_latent_odor} shows the odor stimuli separation on the latent space and how good MMVM VAE is in separating the odors. 
At the same time, two baseline models fail to extract the shared information between rats. Although it shows separation in some views, the independent model does not provide a connection between views. The five tiny clusters in \Cref{fig:rats_latent_odor_joint} show that, instead of showing a clear odor separation on the latent space, the \textit{AVG} VAE separated the data by rats.
In other words, the five rats' latent representations were far from each other, so the \textit{AVG} VAE failed to connect the five views.
See also \cref{app:experiments_rats} for more results.

\subsection{MIMIC-CXR}
\label{sec:exp_mimic}

\textbf{Dataset.} To assess the performance of our approach in a real-world setting, we evaluate the proposed MMVM method on the automated analysis of chest X-rays, a common and critical medical task. For this purpose, we use the MIMIC-CXR dataset \citep{johnson2019mimic}, a well-established and extensive collection of chest X-rays.
The dataset reflects real clinical challenges with varying image quality due to technical issues, patient positioning, and obstructions. The dataset includes different views, which provide complementary information valuable for improving diagnostic~\citep{raoof2012interpretation}. In this work, we consider frontal and lateral images as two modalities (see \cref{app:mimiccxr} for further details). 
Each set of X-rays is labeled with different cardiopulmonary conditions, which have been automatically extracted from the associated radiology reports~\citep{irvin2019chexpert}. This results in instances with incomplete label sets~\citep{Haque2021.07.30.21261225}, which presents a challenge for fully supervised approaches and motivates the need for self-supervised methods instead.

\begin{table}
\caption{VAE latent representation quality evaluation. Average AUROC [in \%] over three seeds of the two unimodal latent representations ($z_F$ and $z_L$) on a subset of MIMIC-CXR labels. The latent representations learned by the MMVM VAE lead to better classification performance compared to the other VAEs and are competitive with the fully-supervised method. Full results in Appendix~\ref{app:mimiccxr_addresults}.}
\vspace{2pt}
\begin{tabular}{l>{\centering}m{1.32cm}>{\centering}m{1.56cm}>{\centering}m{1.82cm}ccc>{\centering}m{0.5cm}}
\toprule
 & All~labels & No~Finding & Cardiomegaly & Edema & Lung Lesion & Pneumonia \\
\midrule
\textit{supervised} & \textit{\small70.5 \scriptsize $\pm$12.1} & \textit{\small73.0 \scriptsize $\pm$1.4} & \textit{\textbf{\small80.3} \scriptsize $\pm$1.4} & \textit{\textbf{\small87.1} \scriptsize $\pm$0.9} & \textit{\small54.8 \scriptsize $\pm$2.5} & \textit{\textbf{\small61.3} \scriptsize $\pm$0.4} \\
\midrule
independent & \small68.0 \scriptsize $\pm$8.3 & \small75.3 \scriptsize $\pm$1.4 & \small73.5 \scriptsize $\pm$2.8 & \small79.2 \scriptsize $\pm$3.9 & \small60.1 \scriptsize $\pm$1.2 & \small55.8 \scriptsize $\pm$0.8 \\
AVG & \small69.8 \scriptsize $\pm$8.5 & \small76.3 \scriptsize $\pm$1.5 & \small76.1 \scriptsize $\pm$2.4 & \small81.3 \scriptsize $\pm$3.3 & \small60.4 \scriptsize $\pm$1.4 & \small57.3 \scriptsize $\pm$0.6 \\
MoE & \small68.9 \scriptsize $\pm$8.7 & \small76.5 \scriptsize $\pm$0.7 & \small74.9 \scriptsize $\pm$1.6 & \small80.2 \scriptsize $\pm$2.3 & \small59.6 \scriptsize $\pm$1.3 & \small56.9 \scriptsize $\pm$1.0 \\
MoPoE & \small70.3 \scriptsize $\pm$8.9 & \small77.2 \scriptsize $\pm$0.2 & \small76.3 \scriptsize $\pm$0.8 & \small82.1 \scriptsize $\pm$1.2 & \small60.8 \scriptsize $\pm$0.6 & \small57.8 \scriptsize $\pm$0.7 \\
PoE & \small70.4 \scriptsize $\pm$8.3 & \small75.9 \scriptsize $\pm$1.3 & \small76.7 \scriptsize $\pm$1.9 & \small81.8 \scriptsize $\pm$2.7 & \small61.3 \scriptsize $\pm$2.1 & \small57.8 \scriptsize $\pm$0.4 \\
MMVM & \textbf{\small73.1} \scriptsize $\pm$8.8 & \textbf{\small78.7} \scriptsize $\pm$0.4 & \small79.6 \scriptsize $\pm$0.9 & \small85.3 \scriptsize $\pm$1.0 & \textbf{\small63.6} \scriptsize $\pm$0.7 & \small59.5 \scriptsize $\pm$0.7 \\
\bottomrule
\end{tabular}
\label{tab:exp_mimic_cxr}
\vspace{-0.25cm}
\end{table}

\noindent
\textbf{Results.} We evaluate\footnote{The code for the MIMIC-CXR experiments can be found here: \url{https://github.com/agostini335/mmvmvae-mimic}} the quality of the unimodal latent representations of the MMVM VAE by comparing them with those learned by a set of jointly-trained independent VAEs (\textit{independent}) as well as with representations from other multimodal VAEs that use aggregation-based approaches (\textit{AVG}, \textit{MoE}, \textit{MoPoE}, and \textit{PoE}) (see \cref{sec:exp_benchmarks} for more details). We do this by training binary random forest classifiers independently for each method and all labels on the inferred representations of a subset of the training set. \Cref{tab:exp_mimic_cxr} shows the AUROC for these classifiers, averaged over three seeds and both unimodal representations for a subset of labels. In addition, we also report the performance of a deep nonlinear network trained in a fully supervised manner (\textit{supervised}) on the same train/test split for reference purposes. Detailed experiment information can be found in \Cref{app:mimiccxr_addresults}, with extensive results for each modality and label available in \Cref{tab:mimic_cxr_frontal} and~\Cref{tab:mimic_cxr_lateral}.
Overall, our approach shows performance improvements across all labels compared to the other VAEs and is highly competitive with the fully-supervised method, surpassing it in average performance over all labels.
Examining each unimodal representation separately provides further insights into the VAEs' ability to leverage information from other modalities during training.   
For example, in the \textit{Cardiomegaly} prediction task,
the MMVM VAE's lateral representations~$z_L$ slightly outperform the PoE VAE's frontal representations~$z_F$ (MMVM $z_L$: 78.7\%, PoE $z_F$: 78.5\%), even though the lateral modality seems generally less informative (supervised $x_L$: 79.0\%, $x_F$: 81.7\%) for this task. The same observation can be made for other labels (see detailed results and discussion in \cref{app:mimiccxr_addresults}). This illustrates the MMVM VAE's ability to soft-share information between modality-specific latent representations during training, thereby enhancing the representation quality of the weaker modality.

\section{Broader Impact \& Limitations}
\label{sec:limitations}
This paper aims to advance the field of Machine Learning by providing a natural and fundamental solution for integrating data across modalities. The proposed approach can be applied to various scientific and engineering problems with a potentially significant societal impact.
In the field of neuroscience specifically, our method could allow neuroscientists to leverage individual differences in brain activity and behavior to understand basic information processing in the brain, as well as to capture distinct longitudinal changes to understand how it is affected in disease. In translational research, it could help identify subjects more susceptible to disease or potentially more responsive to treatment.
While we can show that the proposed method learns better representations and generates more coherent samples, we cannot directly generate random samples anymore (see also \cref{sec:method}).
Although we show results on two real-world datasets, \cref{sec:exp_mimic,sec:exp_rateegs}, additional experiments on even larger scale multimodal datasets would help further evaluate the proposed method, e.g. \citep{Damen2018EPICKITCHENS,wu2023nextgpt}. However, training and evaluating our methods on such datasets requires immense computing resources.

%% file: conclusion.tex
\section{Discussion \& Conclusion}
\label{sec:conclusion}
In this work, we presented a new multimodal VAE, called MMVM VAE, based on a data-dependent multimodal variational mixture-of-experts prior.
By focusing on a multimodal prior, the proposed MMVM VAE overcomes the limitations of previous methods with over-restrictive definitions of joint posterior approximations.
The proposed objective leveraging the MMVM prior takes inspiration from contrastive learning objectives, where the goal is to minimize the similarity between positive pairs while maximizing the similarity between negative pairs \citep[see, e.g., ][]{chen_simple_2020,tian2020}.
In the MMVM objective, we minimize the similarity between different modalities of the same sample (positive pairs) via the regularizing term in the objective, whereas the second part of the objective, the reconstruction loss, prevents degenerate solutions.

In extensive experiments on three different benchmark datasets, we show that MMVM VAE outperforms previous works in terms of learned latent representations as well as generative quality and coherence of missing modalities.
We also demonstrate its efficacy on two challenging real-world applications and show improved performance compared to previous VAEs and even a fully-supervised approach.
Future research could involve studying the representation-distortion tradeoff from an information-theoretical perspective~\citep{yang2022towards, yang2023estimating} and applying similar ideas to more powerful multimodal generative models and representation learning methods.
We see a lot of potential in applying the MMVM regularization to other multimodal and multiview objectives, e.g., as an additional guidance signal for diffusion models.
While masked modeling has shown impressive results as an objective for representation learning, current multimodal masked modeling objectives concatenate the embedding tokens coming from different modalities \citep[see, e.g., ][]{bachmann2022}.
Adding the MMVM regularization objective would offer an interesting alternative to sharing information from different modalities.

%% file: appendix.tex
\onecolumn
\section{MMVM VAE}
\label{app:mm_vamp_vae}

\subsection{Bound on the proposed objective}
The objective $\mathcal{E}(\bm{x}_m)$ for a single modality $m$ is given by
\begin{align}
    \mathcal{E}(\bm{x}_m) =&~ \mathbb{E}_{q_\phi^m (\bm{z}_m \mid \bm{x}_m)} \left[ \log p_\theta (\bm{x}_m \mid \bm{z}_m) \right] - KL \left[ q_\phi(\bm{z}_m \mid \bm{x}_m) \mid \mid p_\theta (\bm{z}_m) \right] \\
    \label{eq:elbo_no_reg}
    \leq &~ \mathbb{E}_{q_\phi^m (\bm{z}_m \mid \bm{x}_m)} \left[ \log p_\theta (\bm{x}_m \mid \bm{z}_m) \right] \\
    \label{eq:elbo_no_reg_zero_variance}
    \leq &~ \log p_\theta (\bm{x}_m \mid \bm{\mu}_m (\bm{x}_m)),
\end{align}
where $\bm{\mu}_m (\bm{x}_m) = f_m(\bm{x}_m)$ is the output of the optimized (mean) encoder $f_m(\cdot)$ of modality $m$.
\Cref{eq:elbo_no_reg} follows from the non-negativity of the KL divergence.
Without regularization term, the maximizing distribution is a delta distribution with zero variance (\cref{eq:elbo_no_reg_zero_variance}).
\Cref{eq:elbo_no_reg_zero_variance} equals the maximum-likelihood version of the proposed MMVM VAE (for a single modality).
Put differently, the MSE of a "vanilla" autoencoder is an upper bound on the objective $\mathcal{E}(\bm{x}_m)$ for any prior distribution $p_\theta (\bm{z}_m)$, hence, also for the newly introduced MMVM prior distribution $h(\bm{z}_m \mid \bm{X})$.

Let us have a look at when the KL term actually vanishes given a MMVM prior distribution $h(\bm{z}_m \mid \bm{X})$.
The KL term can only vanish if all posterior approximation $q_\phi^{\tilde{m}} (\bm{z}_m \mid \bm{x}_{\tilde{m}})$ map to a single mode $q_\phi^m(\bm{z}_m \mid \bm{x}_m)$.
In that case
\begin{align}
    h(\bm{z}_m \mid \bm{X}) = \frac{1}{M} \sum_{\tilde{m}=1}^M q_\phi^{\tilde{m}} (\bm{z}_m \mid \bm{x}_{\tilde{m}}) = q_\phi^m(\bm{z}_m \mid \bm{x}_m)
\end{align}
and $KL \left[ q_\phi^m(\bm{z}_m \mid \bm{x}_m) \mid \mid h(\bm{z}_m \mid \bm{X}) \right] = 0$.

Another scenario is when $q_\phi^m(\bm{z}_m \mid \bm{x}_m)$ and $h(\bm{z}_m \mid \bm{x}_m)$ have disjoint modes.
Hence, $q_\phi^m(\bm{z}_m \mid \bm{x}_m)$ is only a match to itself.
In this case, we have
\begin{align}
    \mathbb{E}_{q_\phi^m(\bm{z}_m \mid \bm{x}_m)} \left[\log \left( \frac{1}{M} \sum_{\tilde{m}=1}^M q_\phi^{\tilde{m}}(\bm{z}_m \mid \bm{x}_{\tilde{m}}) \right) \right] \approx &~ \mathbb{E}_{q_\phi^m(\bm{z}_m \mid \bm{x}_m)} \left[ \log \left( \frac{1}{M} q_\phi^m(\bm{z}_m \mid \bm{x}_m) \right) \right] \\
    =&~ - \log M + \mathbb{E}_{q_\phi^m(\bm{z}_m \mid \bm{x}_m)} \left[ \log q_\phi^m(\bm{z}_m \mid \bm{x}_m) \right]
\end{align}
So, our objective will still be (up to a constant) the maximum likelihood objective, leading the variances to shrink to zero.
In this case, the objective will reduce to the limiting independent VAE, where the modalities do not "talk" to each other and there is no multimodal alignment.

In the most interesting case, there will be a non-trivial overlap between $q_\phi^m(\bm{z}_m \mid \bm{x}_m)$ and $h(\bm{z}_m \mid \bm{x}_m)$, i. e. between the different unimodal posterior approximations, leading to a multimodal alignment through the soft sharing that we wish to see.

In addition, \cref{fig:objective_vs_msebound} empirically shows that the negative mean squared error (MSE) of the vanilla autoencoder (AE) upper bounds the proposed objective $\mathcal{E}$.
We show that lowering the $\beta$ value of the regularizer $R$ (see \cref{sec:method}) approximates the negative MSE of the vanilla AE.

\begin{figure}
    \centering
    \includegraphics[width=0.5\textwidth]{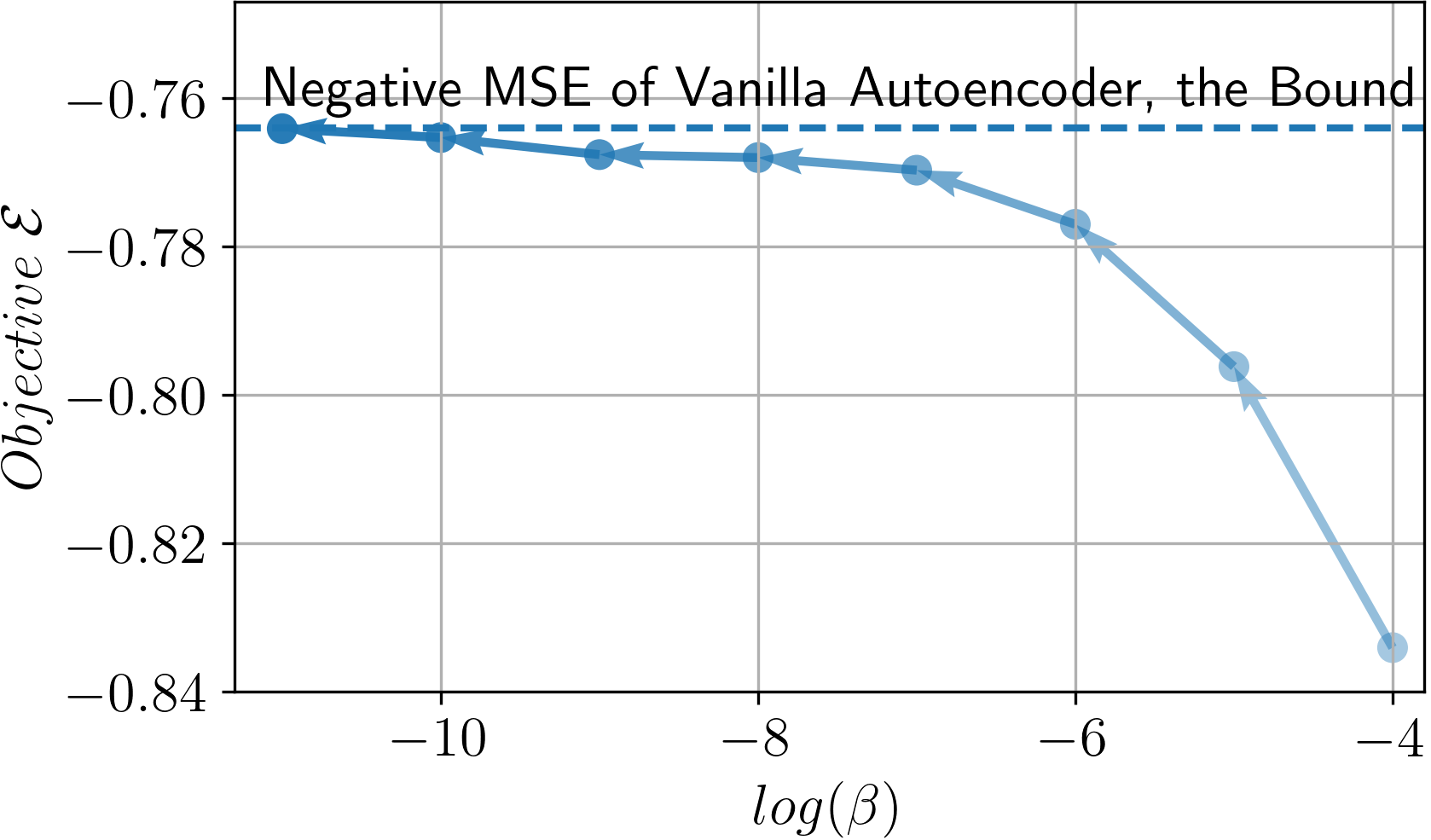}
    \caption{
    We compare the achieved values of the proposed objective $\mathcal{E}$ to the vanilla Autoencoder's negative mean squared error (MSE).
    Lowering the $\beta$ value of the regularizer $R$ in the objective (see \cref{sec:method}) approximates the negative MSE bound provided by the vanilla AE.
    This proves empirically that the negative MSE of the vanilla AE indeed upper bounds the proposed objective $\mathcal{E}$.
    }
    \label{fig:objective_vs_msebound}
\end{figure}

\section{Experiments}
\label{app:experiments}

\subsection{Dataset Licences}
\label{app:experiments_dataset_licenses}
\begin{itemize}
    \item PolyMNIST: originally published in \citep{sutter2021}, downloaded the data from \url{https://drive.google.com/drive/folders/1lr-laYwjDq3AzalaIe9jN4shpt1wBsYM?usp=sharing} and the code from \url{https://github.com/thomassutter/MoPoE}, published under the MIT license
    \item Bimodal CelebA: originally published in \citep{sutter_multimodal_2020}, downloaded from \url{https://drive.google.com/drive/folders/1lr-laYwjDq3AzalaIe9jN4shpt1wBsYM?usp=sharing}, licensed under the MIT license. The original CelebA dataset was published in \citep{liu_deep_2015}, license not found.
    \item CUB image-captions: originally published in \citep{shi_variational_2019}, downloaded from \url{http://www.robots.ox.ac.uk/~yshi/mmdgm/datasets/cub.zip}, licensed under GPL 3.0.
    \item MIMIC-CXR: originally published in \citep{johnson2019mimic}, downloaded from \url{https://physionet.org/content/mimic-cxr/2.0.0/}, licensed under PhysioNet Credentialed Health Data License 1.5.0 (see \url{https://physionet.org/content/mimic-cxr/view-license/2.0.0/}).
    \item MIMIC-CXR-JPG: originally published in \citep{johnson2024mimic}, downloaded from \url{https://physionet.org/content/mimic-cxr-jpg/2.1.0/}, licensed under PhysioNet Credentialed Health Data License 1.5.0 (see \url{https://physionet.org/content/mimic-cxr-jpg/2.1.0/}).
    \item Hippocampal Neural Activity data: originally published in \citep{shahbaba_hippocampal_2022}, downloaded from \url{https://datadryad.org/stash/dataset/doi:10.7280/D14X30}, licensed under CC0 1.0 Universal (CC0 1.0) Public Domain Dedication (see \url{https://creativecommons.org/publicdomain/zero/1.0/})
\end{itemize}

\subsection{Evaluation Details}
\label{app:experiments_evaluation}
We evaluate the different methods based on the coherence of their imputed samples, the quality of their latent representations, and their reconstruction error.
We assume access to the full set of modalities during training, but we do not make this assumption at test time.
Hence, there is a need for methods that can conditionally generate samples of these missing modalities, given the available modalities.
In other words, we want to be able to impute missing modalities.
Imputed modalities should not only be of high generative quality but also display the same shared information as the available modalities (i.e., be coherent).
For example, if we generate a sample of modality $\bm{x}_2$ based on modality $\bm{x}_1$ from the PolyMNIST dataset (see \Cref{fig:exp_data_polymnist_translated75}), the generated sample of modality $\bm{x}_2$ should contain the same digit information as modality $\bm{x}_1$ but show the background of modality $\bm{x}_2$.

\subsubsection{Latent Representation Evaluation}
We assess the learned representations based on subsets of modalities and not the full set.
The quality of the representations serves as a proxy of how useful the learned representations are for additional downstream tasks that are not part of the training objective.
Hence, high-quality representations of subsets of modalities are also the basis for conditionally generating coherent samples.
We assess their quality by using the classification performance of linear classifiers.
We train independent classifiers on the unimodal representations of the training set and evaluate them on unimodal test set representations.
To assess the learned latent representation, we train a logistic regression classifier on $10000$ latent representations of the training set.
The accuracy is computed on all representations of the test set.

\subsubsection{Coherence Evaluation}
The coherence of conditionally generated samples shows how well the content of the imputed modalities aligns with the content of the available modalities in terms of the shared information.
We evaluate the coherence using ResNet-based classifiers \citep{he_deep_2016} that are trained on samples from the original training set of every modality.
Using the described procedure, generated samples have to be visually similar to the original samples to have high coherence.
Otherwise, the nonlinear classifier will not be able to predict digits correctly.

To compute the coherence of conditionally generated samples, we train additional deep classifiers on original samples of the training set.
We use a ResNet-based non-linear classifier that is trained on the full original training set.
The prediction of this classifier is then used to determine the class (on PolyMNIST: digit) of the conditionally generated sample of a missing modality.
The nonlinear classifier reaches an accuracy of above 98\% on the original test set.
For other datasets, the nonlinear classifier is trained to predict the shared information of the multimodal dataset.
Hence, it serves as a good oracle for determining the digit of generated samples.

\paragraph{Conditional Generation with the MMVM VAE and the independent VAEs}
To generate modality $\bm{x}_{\tilde{m}}$ conditioned on modality $\bm{x}_m$, we proceed as follows:
\begin{enumerate}
    \item We encode modality $\bm{x}_m$ using the encoder $q_\phi^m (\bm{z}_m \mid \bm{x}_m)$
    \item We sample a latent vector $\bm{z} \sim q_\phi^m (\bm{z}_m \mid \bm{x}_m)$
    \item We input the latent representation $\bm{z}$ into the decoder $p_\theta^{\tilde{m}} (\bm{x}_{\tilde{m}} \mid \bm{z})$ of modality $\bm{x}_{\tilde{m}}$
\end{enumerate}
To compute the coherence numbers reported, we perform the above steps for every modality $\bm{x}_m$ where $m \in \{ 1, \ldots, M \}$ and average the achieved coherence accuracies.

\paragraph{Conditional Generation with the aggregated multimodal VAEs}
All aggregation-based multimodal VAEs conditionally generate samples in the same way.
Hence, the following procedure applies to the multimodal VAEs used in \cref{sec:experiments} (\citep[AVG, ][]{hosoya2018}, \citep[PoE, ][]{wu_multimodal_2018}, \citep[MoE, ][]{shi_variational_2019}, and \citep[MoPoE, ][]{sutter2021}).
To generate modality $\bm{x}_{\tilde{m}}$ conditioned on modality $\bm{x}_m$, we proceed as follows:
\begin{enumerate}
    \item We encode modality $\bm{x}_m$ using the encoder $q_\phi^m (\bm{z}_m \mid \bm{x}_m)$
    \item We sample a latent vector $\bm{z} \sim q_\phi^m (\bm{z}_m \mid \bm{x}_m)$
    \item We input the latent representation $\bm{z}$ into the decoder $p_\theta^{\tilde{m}} (\bm{x}_{\tilde{m}} \mid \bm{z})$ of modality $\bm{x}_{\tilde{m}}$
\end{enumerate}
If we would have access to a multimodal subset $\bm{X}_A$ consisting of more than one modality, i.e. $|A| > 1$, we would have the joint posterior approximation distribution $q_\phi(\bm{z} \mid \bm{X}_A)$ of the subset $\bm{X}_A$, sample a latent vector from that distribution, i.e. $\bm{z} \sim q_\phi (\bm{z} \mid \bm{X}_A)$, and generate modality $\bm{x}_{\tilde{m}}$ using the decoder $p_\theta^{\tilde{m}} (\bm{x}_{\tilde{m}} \mid \bm{z})$ of modality $\bm{x}_{\tilde{m}}$

\subsubsection{Generative Quality}
The reconstruction error is a proxy for how well every method learns the underlying data distribution.
We evaluate the different VAE models by their achieved reconstruction error against either the learned latent representation classification or coherence (e.g., \cref{fig:exp_benchmarks}).
We do this for multiple values of $\beta$, where the average performance over multiple seeds with a single $\beta$ value leads to a scatter point in the figures.
This way, we can assess the trade-off between accurately reconstructing data samples and inferring shared information.
In addition, we compute the FID \citep{heusel2017gans} values for modalities that can be summarized as natural images in the appendix.
A low FID score correlates with a high-quality generated image.

\subsection{Implementation (General)}
We use the scikit-learn \citep{scikit-learn} package for the linear classifiers to evaluate the learned latent representations.
All code is written using Python 3.11, PyTorch \citep{pytorch} and Pytorch-Lightning \citep{PyTorch_Lightning_2019}.
We base the implementations of the aggregation-based VAE methods on the official implementations.
Hence, we base our implementations on the following repositories:
\begin{itemize}
    \item AVG: \url{https://github.com/HaruoHosoya/gvae}
    \item PoE: \url{https://github.com/mhw32/multimodal-vae-public}
    \item MoE: \url{https://github.com/iffsid/mmvae}
    \item MoPoE: \url{https://github.com/thomassutter/MoPoE}
\end{itemize}

\begin{figure}
    \centering
    \includegraphics[width=0.6\textwidth]{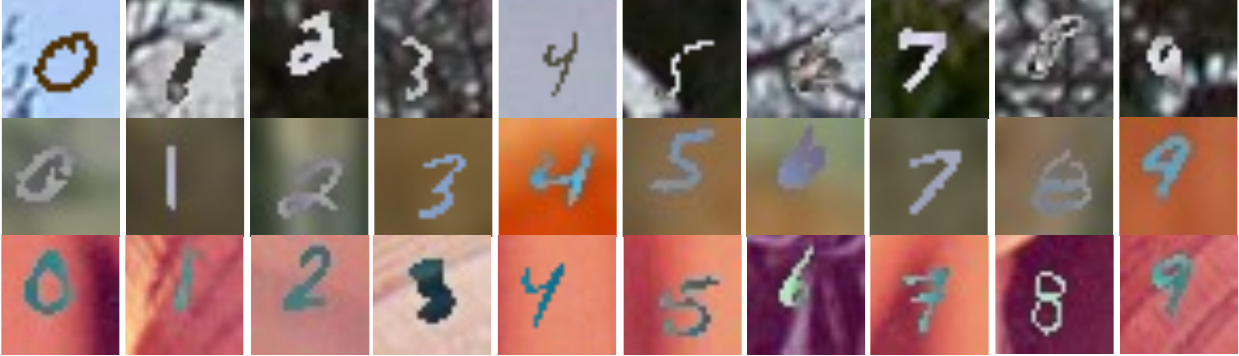}
    \caption{
    PolyMNIST (translated, scale=75\%): every column is a multimodal tuple $\bm{X}$, and every row shows samples of a single modality $\bm{x}_m$.
    We see the random translation between samples by looking at images from a single row or column.
    }
    \label{fig:exp_data_polymnist_translated75}
\end{figure}

\subsection{PolyMNIST}
\label{app:experiments_polymnist}
\subsubsection{Dataset}
The dataset is based on the original MNIST dataset \citep{lecun1998}.
Compared to the original dataset, the digits are scaled down by a factor of $0.75$ such that there is more space for the random translation.
In its original form, the PolyMNIST consists of 5 different modalities.
We only use the first three modalities in this work.
The background of every modality $\bm{x}_m$ consists of random patches of size $28 \times 28$ from a large image.
The digit is placed at a random position of the patch.
We refer to the original publication \citep{sutter2021} for details on the background images.
Using this setup, every modality has modality-specific information given by its background image and shared information given by the digit, which is shared between all modalities.
An additional difficulty compared to the original PolyMNIST is the random translation of the digits.
The dataset can be found at \url{https://github.com/thomassutter/MoPoE}.

\subsubsection{Implementation \& Training}
We use the same network architectures for all multimodal VAEs.
Every multimodal VAE consists of a ResNet-based encoder and a ResNet-based Decoder \citep{he_deep_2016}.
All modalities share the same architecture but are initialized differently.
We assume Gaussian distribution for all unimodal posterior approximations, i.e.
\begin{align}
    q_\phi^m(\bm{z}_m \mid \bm{x}_m) = \mathcal{N}(\bm{z}_m; \bm{\mu}_m, \bm{\sigma}_m \bm{I}),
\end{align}
where the parameters $\bm{\mu}_m$ and $\bm{\sigma}_m$ are inferred using neural networks such that we have
\begin{align}
    q_\phi^m(\bm{z}_m \mid \bm{x}_m) = q_\phi^m(\bm{z}_m; \bm{\mu}_m(\bm{x}_m), \bm{\sigma}_m(\bm{x}_m)) = \mathcal{N}(\bm{z}_m; \bm{\mu}_m(\bm{x}_m), \bm{\sigma}_m(\bm{x}_m))
\end{align}
The conditional data distributions $p_\theta (\bm{x}_m \mid \bm{z}_m)$ are modeled using the Laplace distribution, where the location parameter is modeled with a neural network (decoder) and the scale parameter is set to 0.75 \citep{shi_variational_2019}, i.e.
\begin{align}
    p_\theta (\bm{x}_m \mid \bm{z}_m) = \mathcal{L}(\bm{x}_m; \bm{\mu}_m, \bm{\sigma}_m),
\end{align}
where $\mathcal{L}(\cdot)$ defines a Laplace distribution. 
It follows that
\begin{align}
    p_\theta (\bm{x}_m \mid \bm{z}_m) = p_\theta(\bm{x}_m; \bm{\mu}_m(\bm{z}_m), \bm{\sigma}_m) = \mathcal{L}(\bm{x}_m; \bm{\mu}_m(\bm{z}_m), \bm{\sigma}_m)
\end{align}

We use the method of \citet{hosoya2018} for the implementation of the aggregated VAE.
In this approach, a simplistic version of the joint posterior distribution is chosen where for Gaussian distribution joint posterior approximation $\mathcal{N}(\bm{\mu}_s, \bm{\sigma}_s\bm{I})$ we have the following distribution parameters $\bm{\mu}_s$ and $\bm{\sigma}_s$
\begin{align}
    \bm{\mu}_s = &~ \frac{1}{M} \sum_{m=1}^M \bm{\mu}_m \\
    \bm{\sigma}_s = &~ \frac{1}{M} \sum_{m=1}^M \bm{\sigma}_s
\end{align}
$\bm{\mu}_m$ and $\bm{\sigma}_m$ are the distribution parameters of the unimodal posterior approximations $\mathcal{N}(\bm{\mu}_m, \bm{\sigma}_m \bm{I})$.

During training and evaluation, no weight-sharing takes place, i.e. every modality has its own encoder and decoder.
We use the same architectures as in \citep{daunhawer2022}.
For all experiments on this dataset, we use an Adam optimizer \citep{kingma_adam_2014} with an initial learning rate of $0.0005$, and a batch size of $256$.
We train all models for $500$ epochs.

We use NVIDIA GTX 2080 GPUs for all our runs. Each experiment can be run with $4$ CPU workers and $16$ GB of memory.
An average run takes around $7$ hours. To train all methods used in this paper, we had to train $11 \times 5 \times 6 = 330$ different models:
$11$ different $\beta$ values, $5$ different seeds, and $6$ different methods.
Hence, the total GPU compute time used to generate the results for the PolyMNIST dataset is equal to $330 \times 7 = 2310$ hours.
We---of course---also had to invest development GPU time, which we did not measure.

\subsubsection{Additional Results}
\label{app:experiments_polymnist_additional_results}
We generated \Cref{fig:exp_benchmarks_downstream_polymnist,fig:exp_benchmarks_coherence_polymnist} by plotting the classification accuracy of a linear classifier, which we trained on the learned latent representation, against the reconstruction error on the test set.
For the learned latent representation, we train a classifier on the unimodal latent representations.
For the aggregated VAE, this means that we train the classifier on samples of the unimodal posterior approximations and not the joint posterior approximations.
Using this procedure, we test the different methods according to their performance in case of missing data, e.g., we only have access to a single modality instead of the full set at test time.
For the reconstruction loss, however, we computed the error given the full set of modalities.
The idea for \cref{fig:exp_benchmarks_downstream_polymnist,fig:exp_benchmarks_coherence_polymnist} is to compare the reconstruction quality (i.e., how well can we learn the data distribution?) against metrics that are related to the "generative factors" of the data and relate the different modalities to each other, i.e. the shared information of a multimodal dataset.

In \cref{fig:exp_polymnist_condrec}, we evaluate the performance of individual modalities in case of missing modalities.
For that, we reconstruct every modality if it was the only modality available at test time.
Hence, the modalities in the aggregated VAE have to be reconstructed based on the unimodal posterior approximations and not the joint posterior approximation.
For the independent VAEs and the MMVM-VAE, the reconstruction of a modality is only based on its own unimodal posterior approximation.
Hence, for the latter two methods, nothing changes in this setting.
The performance of the learned latent representation and the coherence of generated samples are evaluated in the same way as in \cref{fig:exp_benchmarks_downstream_polymnist,fig:exp_benchmarks_coherence_polymnist}.

\Cref{fig:exp_polymnist_condrec} shows that the reconstruction error of the aggregated VAE increases a lot if every modality needs to reconstruct itself.
Interestingly, we can see that the "self-reconstruction error" (the x-axis in \cref{fig:exp_polymnist_downstream_condrecloss,fig:exp_polymnist_coherence_condrecloss}) decreases with an increasing $\beta$-value, which is different to the other two methods and also different to the aggregated VAE's behavior in \cref{fig:exp_benchmarks_downstream_polymnist,fig:exp_benchmarks_coherence_polymnist}.

In addition, we also show the conditional FID values the different multimodal VAEs achieve.
Conditional FIDs come from the conditional generation of modality $\bm{x}_m$ given $\bm{x}_m$.
From all the conditionally generated samples of modality $\bm{x}_m$, we then compute the conditional FID values.
\Cref{fig:exp_polymnist_fid} shows the aggregated values of all conditional FIDs.
We again show the FID values in relation to the downstream task performance and the conditional generation coherence.
We see that only the proposed MMVM VAE reaches FID values similar to the ones of a set of independent VAEs.
However, the set of independent VAEs cannot achieve the same latent representation classification nor coherence as the proposed MMVM VAE.

\begin{figure*}
    \centering
    \begin{subfigure}[t]{1.0\textwidth}
        \centering
        \includegraphics[width=1.0\textwidth]{figures/legend_coherence_recloss.png}
    \end{subfigure}
    \centering
    \begin{subfigure}[t]{0.45\textwidth}
        \includegraphics[width=1.0\textwidth]{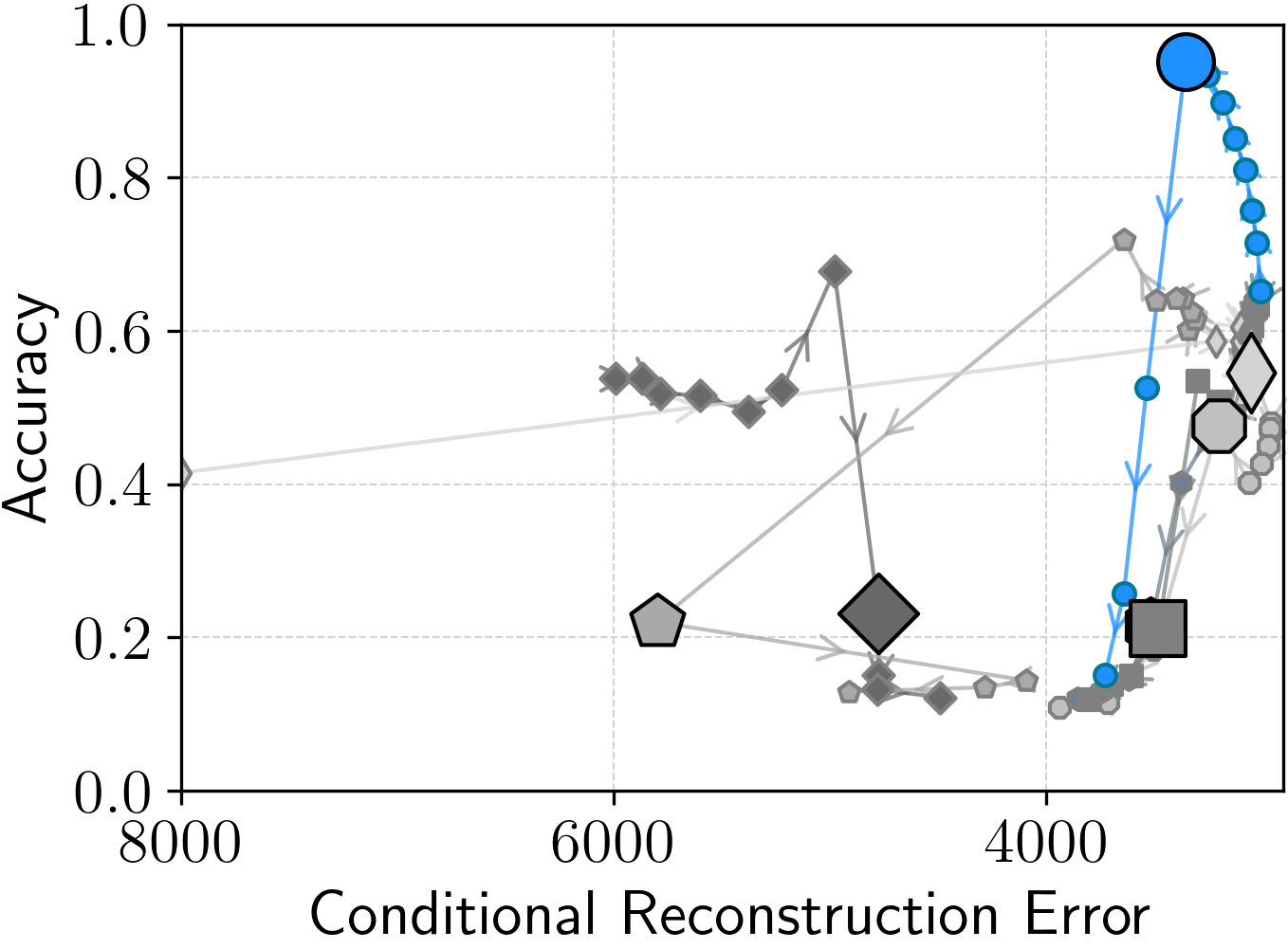}
        \caption{Latent Representation Classification}
        \label{fig:exp_polymnist_downstream_condrecloss}
    \end{subfigure}
    \hspace{0.5cm}
    \centering
    \begin{subfigure}[t]{0.45\textwidth}
        \includegraphics[width=1.0\textwidth]{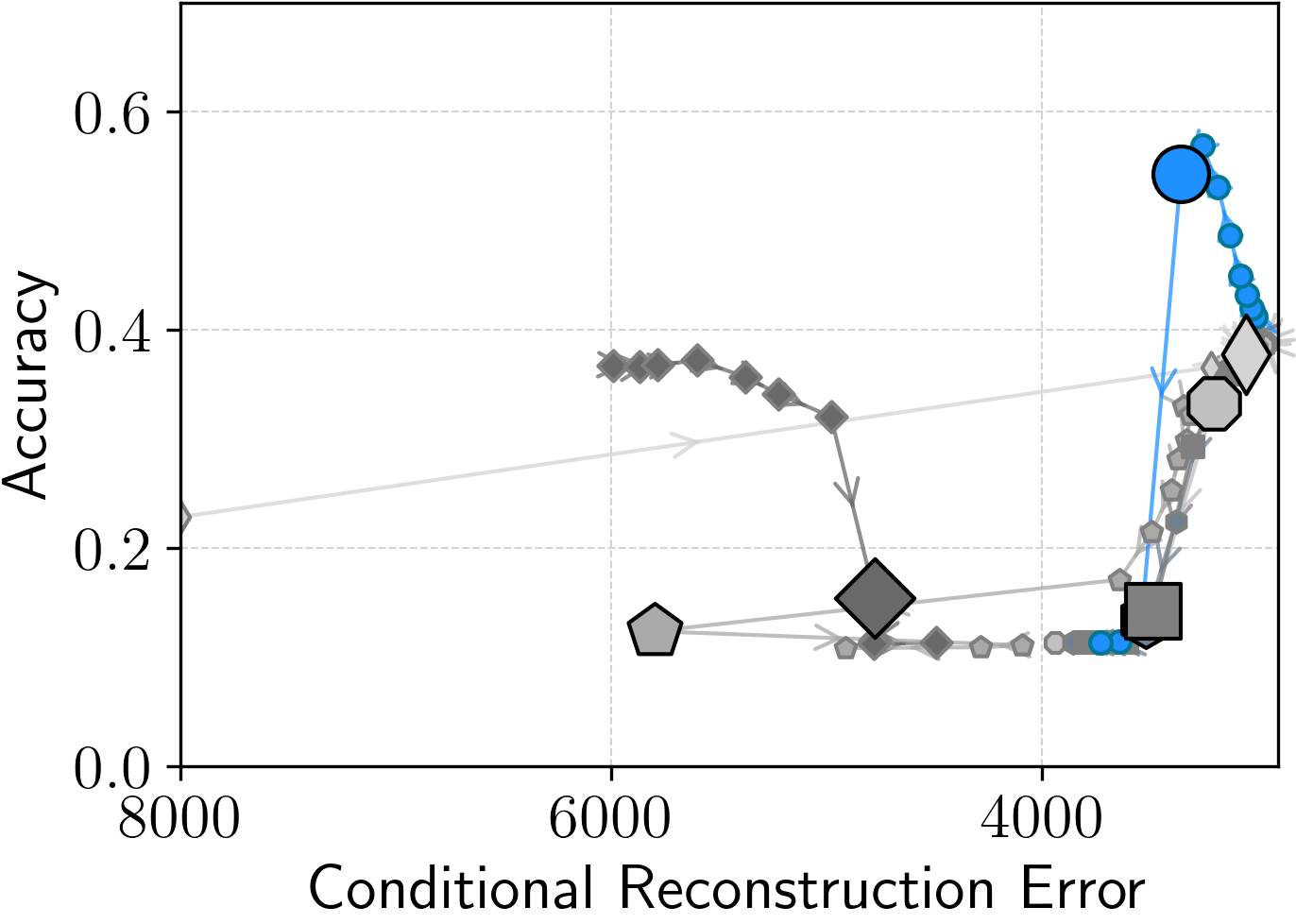}
        \caption{Conditional generation coherence}
        \label{fig:exp_polymnist_coherence_condrecloss}
    \end{subfigure}
    \hspace{0.25cm}
    \caption{
    Results on the PolyMNIST dataset for different VAE methods.
    We report the performance of the latent representation classification and the conditional generation coherence against the conditional reconstruction loss for different $\beta$ values.
    Every point in the figures above is the average of five runs over different seeds and a specific $\beta$ value where $\beta = 2^{k}$ for $k \in \{ -8, \ldots, 3 \}$.
    Different to \cref{fig:exp_benchmarks_downstream_polymnist,fig:exp_benchmarks_coherence_polymnist}, the x-axis is the sum of the self-reconstruction losses if only a single modality is given as input.
    Hence, for the aggregated VAE methods, every modality is decoded by its own unimodal posterior approximation instead of the joint posterior approximation.
    }
    \label{fig:exp_polymnist_condrec}
\end{figure*}

\begin{figure*}
    \centering
    \begin{subfigure}[t]{1.0\textwidth}
        \centering
        \includegraphics[width=1.0\textwidth]{figures/legend_coherence_recloss.png}
    \end{subfigure}
    \centering
    \begin{subfigure}[t]{0.45\textwidth}
        \includegraphics[width=1.0\textwidth]{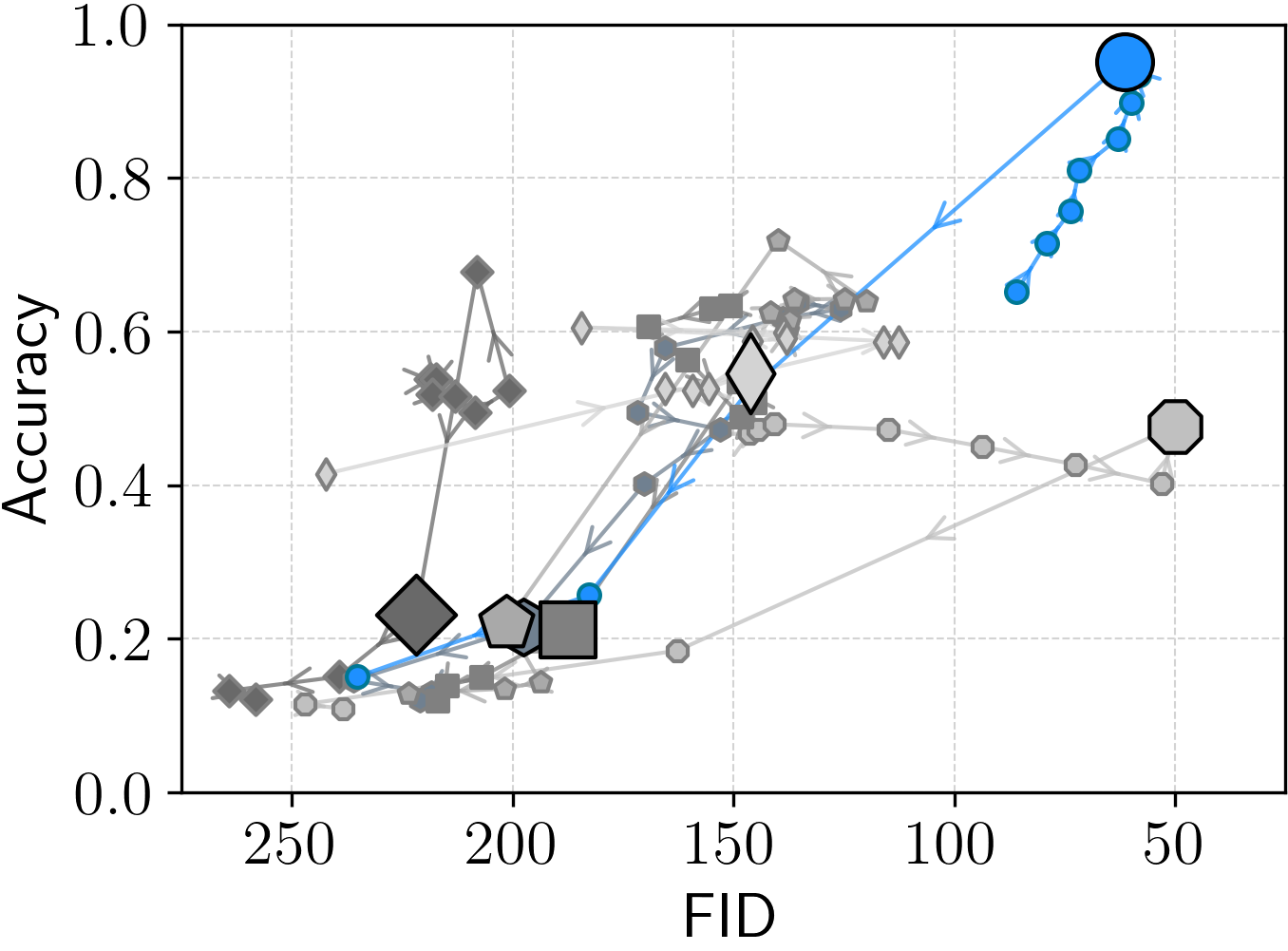}
        \caption{Latent Representation Classification}
        \label{fig:exp_polymnist_downstream_fid}
    \end{subfigure}
    \hspace{0.5cm}
    \centering
    \begin{subfigure}[t]{0.45\textwidth}
        \includegraphics[width=1.0\textwidth]{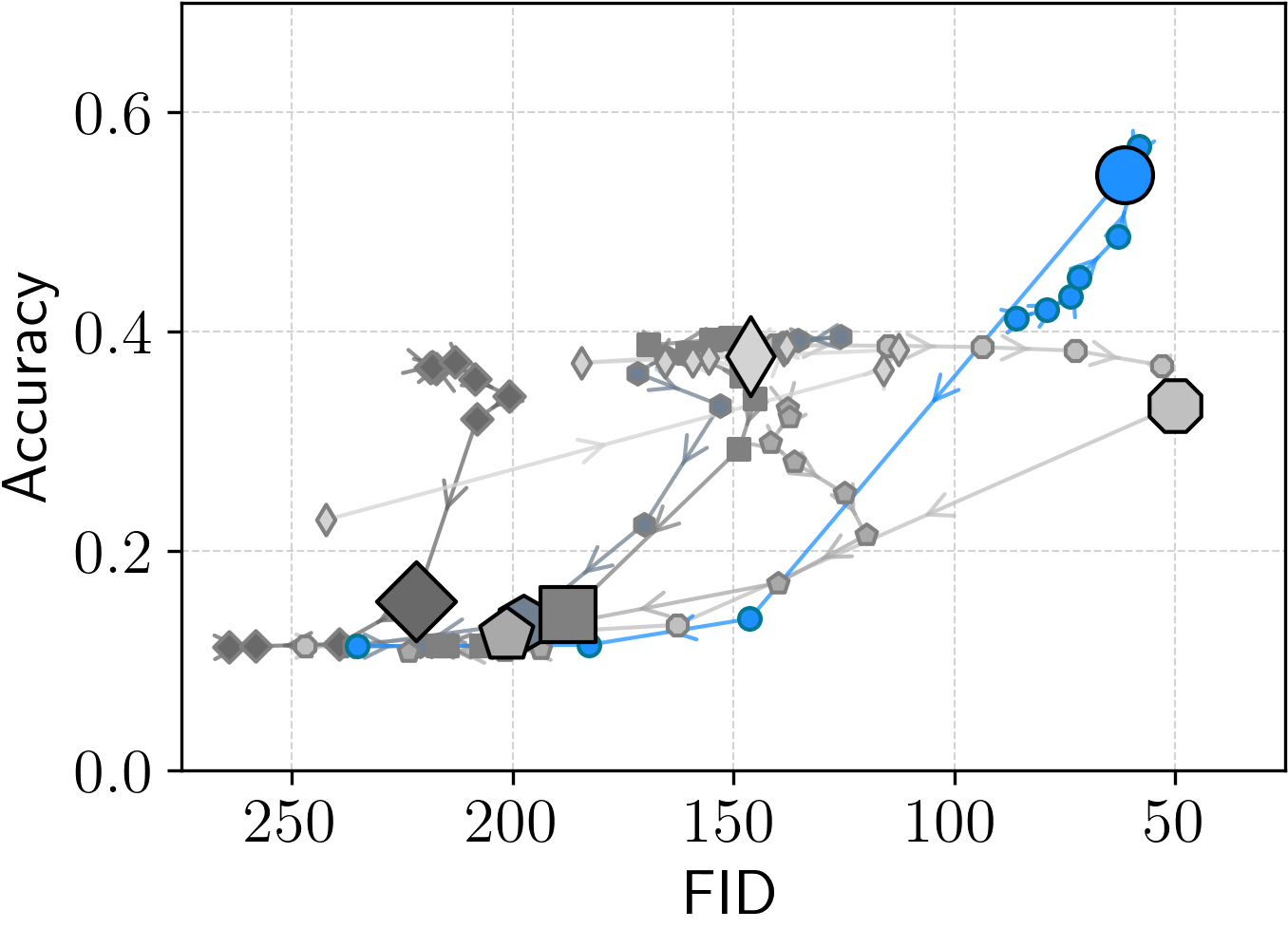}
        \caption{Conditional generation coherence}
        \label{fig:exp_polymnist_coherence_fid}
    \end{subfigure}
    \hspace{0.25cm}
    \caption{
    Results on the PolyMNIST dataset for different VAE methods.
    We report the performance of the latent representation classification and the conditional generation coherence against the conditional FID values.
    Every point in the figures above is the average of five runs over different seeds and a specific $\beta$ value where $\beta = 2^{k}$ for $k \in \{ -8, \ldots, 3 \}$.
    Different to \cref{fig:exp_benchmarks_downstream_polymnist,fig:exp_benchmarks_coherence_polymnist}, the x-axis is not the reconstruction error but the average FID value computed from conditionally generated samples.
    }
    \label{fig:exp_polymnist_fid}
\end{figure*}

\begin{figure}
    \centering
    \centering
    \begin{subfigure}[t]{0.3\textwidth}
        \includegraphics[width=1.0\textwidth]{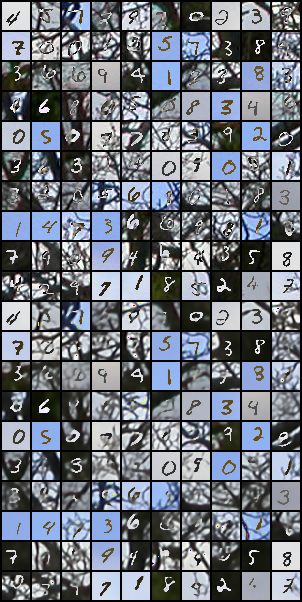}
        \caption{m0 $\rightarrow$ m0}
        \label{fig:exp_pm_cond_gen_unimodal_m0_to_m0}
    \end{subfigure}
    \hfill
    \centering
    \begin{subfigure}[t]{0.3\textwidth}
        \includegraphics[width=1.0\textwidth]{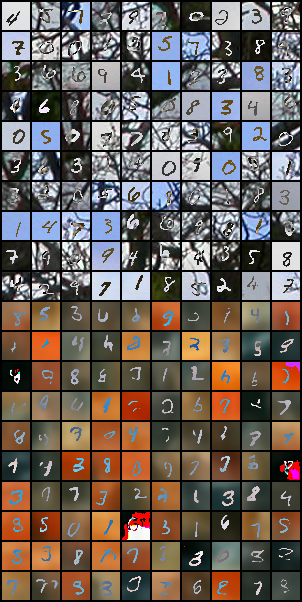}
        \caption{m0 $\rightarrow$ m1}
        \label{fig:exp_pm_cond_gen_unimodal_m0_to_m1}
    \end{subfigure}
    \hfill
    \centering
    \begin{subfigure}[t]{0.3\textwidth}
        \includegraphics[width=1.0\textwidth]{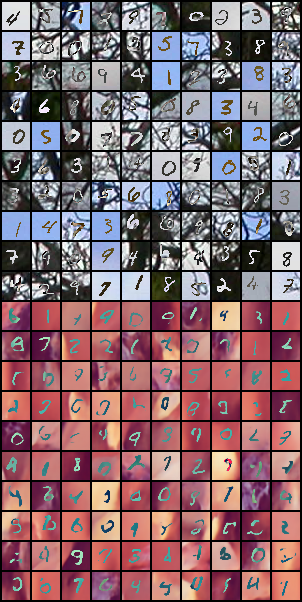}
        \caption{m0 $\rightarrow$ m2}
        \label{fig:exp_pm_cond_gen_unimodal_m0_to_m2}
    \end{subfigure}
    \caption{
    Qualitative results for the conditional generation task for the set of unimodal VAEs.
    }
    \label{fig:exp_pm_cond_gen_unimodal}
\end{figure}

\begin{figure}
    \centering
    \centering
    \begin{subfigure}[t]{0.3\textwidth}
        \includegraphics[width=1.0\textwidth]{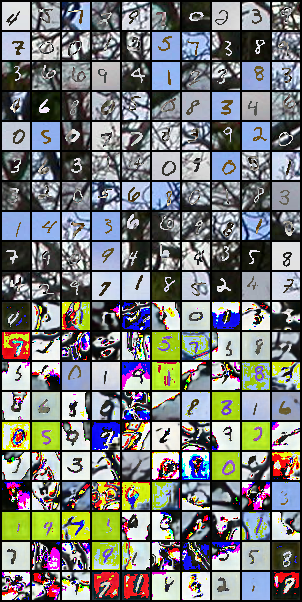}
        \caption{m0 $\rightarrow$ m0}
        \label{fig:exp_pm_cond_gen_joint_m0_to_m0}
    \end{subfigure}
    \hfill
    \centering
    \begin{subfigure}[t]{0.3\textwidth}
        \includegraphics[width=1.0\textwidth]{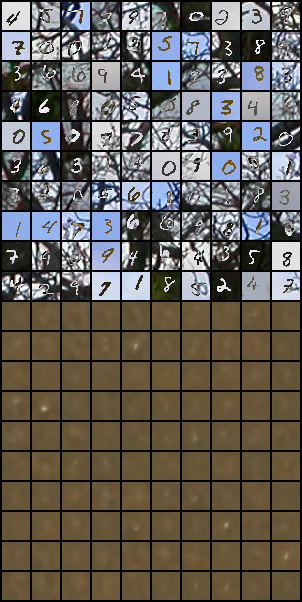}
        \caption{m0 $\rightarrow$ m1}
        \label{fig:exp_pm_cond_gen_joint_m0_to_m1}
    \end{subfigure}
    \hfill
    \centering
    \begin{subfigure}[t]{0.3\textwidth}
        \includegraphics[width=1.0\textwidth]{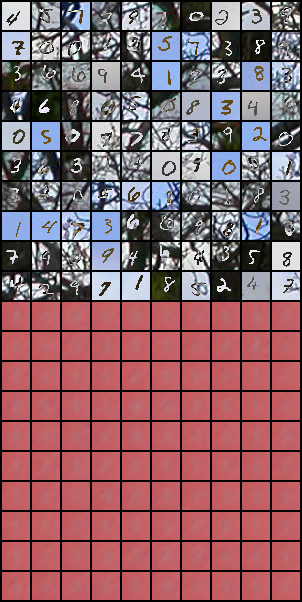}
        \caption{m0 $\rightarrow$ m2}
        \label{fig:exp_pm_cond_gen_joint_m0_to_m2}
    \end{subfigure}
    \caption{
    Qualitative results for the conditional generation task for the aggregation-based multimodal VAE (AVG).
    }
    \label{fig:exp_pm_cond_gen_joint}
\end{figure}

\begin{figure}
    \centering
    \centering
    \begin{subfigure}[t]{0.3\textwidth}
        \includegraphics[width=1.0\textwidth]{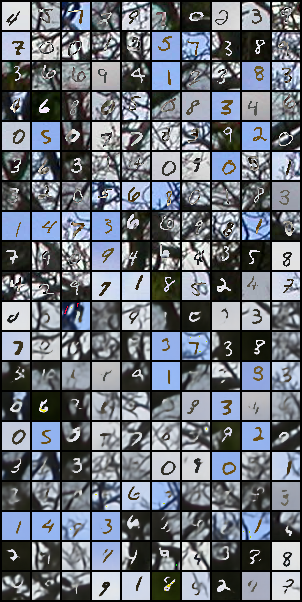}
        \caption{m0 $\rightarrow$ m0}
        \label{fig:exp_pm_cond_gen_mixed_m0_to_m0}
    \end{subfigure}
    \hfill
    \centering
    \begin{subfigure}[t]{0.3\textwidth}
        \includegraphics[width=1.0\textwidth]{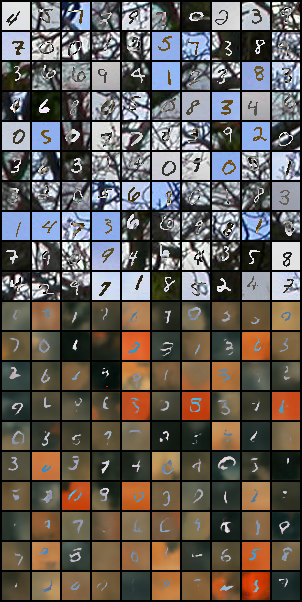}
        \caption{m0 $\rightarrow$ m1}
        \label{fig:exp_pm_cond_gen_mixed_m0_to_m1}
    \end{subfigure}
    \hfill
    \centering
    \begin{subfigure}[t]{0.3\textwidth}
        \includegraphics[width=1.0\textwidth]{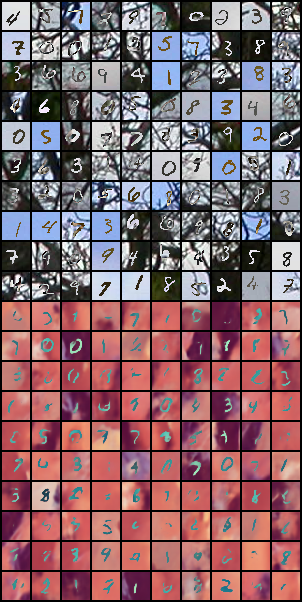}
        \caption{m0 $\rightarrow$ m2}
        \label{fig:exp_pm_cond_gen_mixed_m0_to_m2}
    \end{subfigure}
    \caption{
    Qualitative results for the conditional generation task for the proposed MMVM VAE.
    }
    \label{fig:exp_pm_cond_gen_mixed}
\end{figure}

\begin{figure*}
    \centering
    \begin{subfigure}[t]{1.0\textwidth}
        \centering
        \includegraphics[width=1.0\textwidth]{figures/legend_coherence_recloss.png}
    \end{subfigure}
    \centering
    \begin{subfigure}[t]{0.45\textwidth}
        \includegraphics[width=1.0\textwidth]{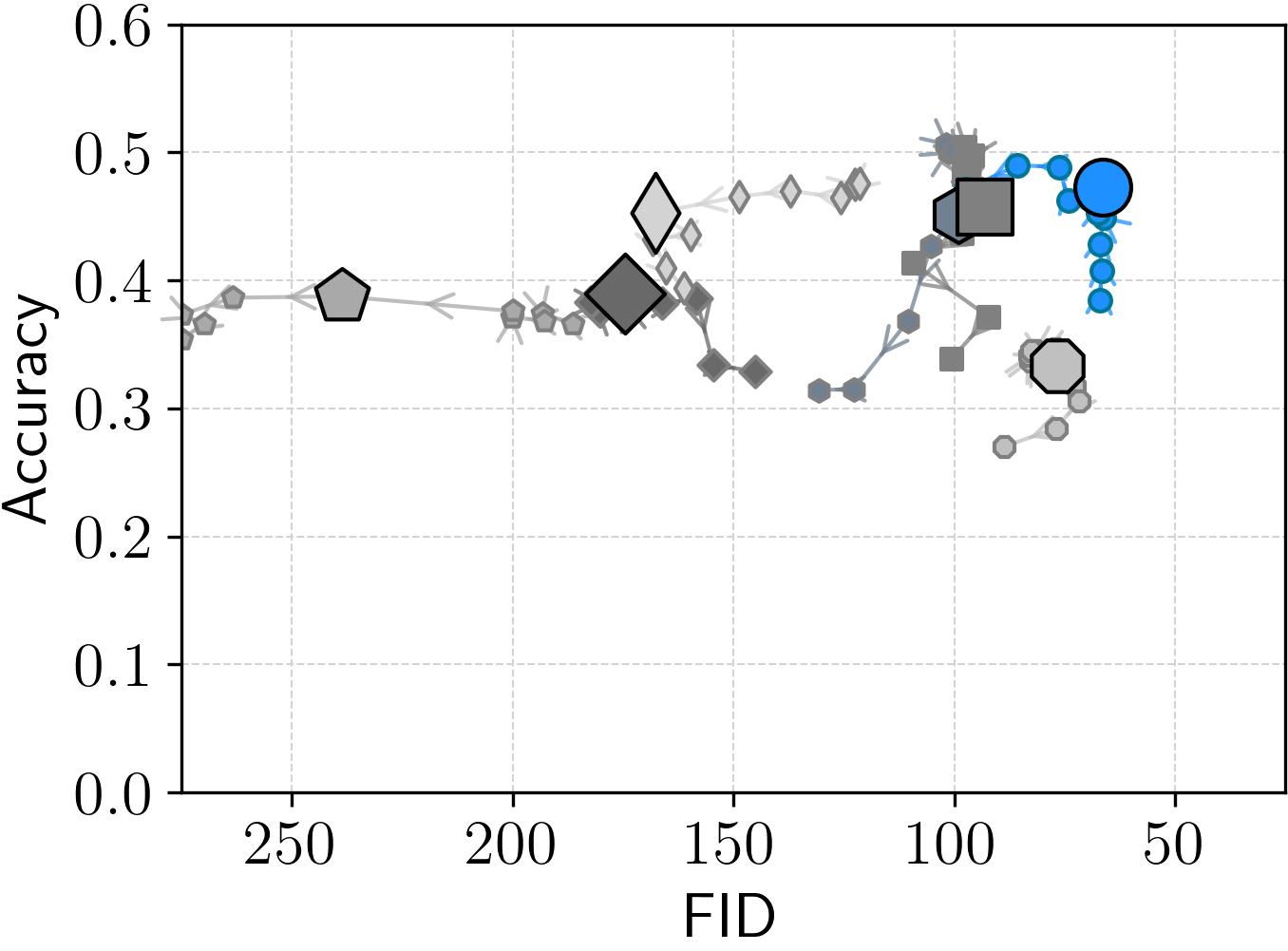}
        \caption{Latent Representation Classification}
        \label{fig:exp_celeba_downstream_fid}
    \end{subfigure}
    \hspace{0.5cm}
    \centering
    \begin{subfigure}[t]{0.45\textwidth}
        \includegraphics[width=1.0\textwidth]{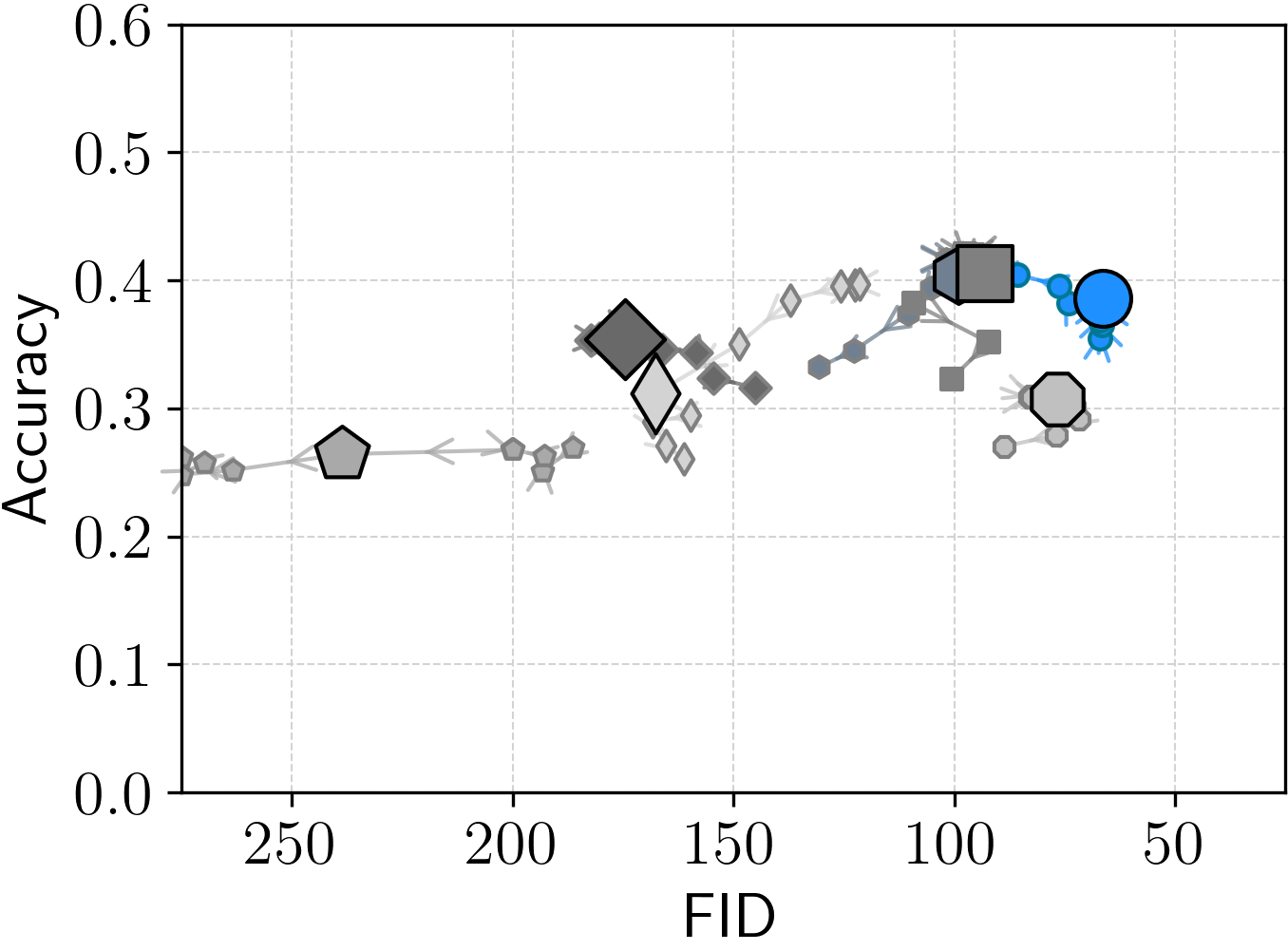}
        \caption{Conditional generation coherence}
        \label{fig:exp_celeba_coherence_fid}
    \end{subfigure}
    \hspace{0.25cm}
    \caption{
    Results on the bimodal CelebA dataset for the different VAE methods.
    We report the performance of the latent representation classification and the conditional generation coherence against the conditional FID values.
    Every point in the figures above is the average of five runs over different seeds and a specific $\beta$ value where $\beta = 2^{k}$ for $k \in \{ -5, \ldots, 4 \}$.
    Unlike in \cref{fig:exp_benchmarks_downstream_polymnist,fig:exp_benchmarks_coherence_polymnist}, the x-axis is not the reconstruction error but the average FID value computed from conditionally generated samples.
    }
    \label{fig:exp_celeba_fid}
\end{figure*}

\subsection{Bimodal CelebA}
\label{app:experiments_celeba}

\begin{figure}
    \centering
    \includegraphics[width=0.4\columnwidth]{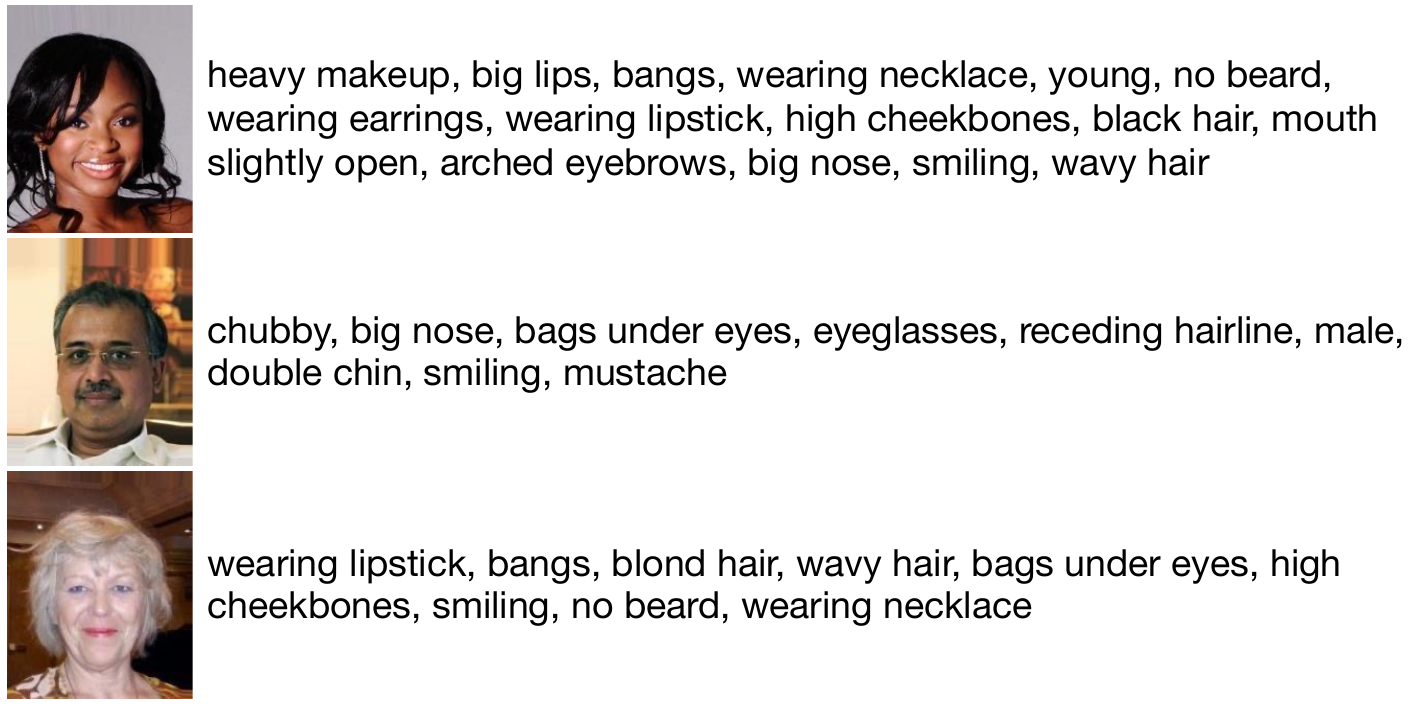}
    \caption{
    Bimodal CelebA. Three samples of image-text pairs.
    To introduce another level of difficulty to the text modality, we added a random translation to the starting point of the text attributes.
    }
    \label{fig:exp_data_celeba_bimodal}
\end{figure}
\subsubsection{Dataset}
We take the dataset from \citet{sutter_multimodal_2020}.
It extends the original CelebA dataset \citep{liu_deep_2015} to the multimodal setting by creating a second modality based on the attributes.
In the original dataset, every image comes with a set of $40$ labels, such as blond hair, smiling, etc.
The difficulty of the individual attributes is not only due to their visual appearance but also their frequency in the dataset.
In the multimodal extension, \citet{sutter2021} created a text string from the attribute words.
Unlike the label vector, absent attributes are dismissed in the text string instead of negated.
Additionally, the attributes are randomly ordered.

\subsubsection{Implementation \& Training}
We use ResNet-based encoders and decoders for this experiment as well \citep{he_deep_2016}, similar to the ones in the PolyMNIST experiment.
The image encoder and decoder consist of 2D convolutions, while the text encoder and decoder consist of 1D convolutions.
We use a character-level encoding and not a word or token-level encoding because of the synthetic nature of the text modality.
We also use an Adam optimizer \citep{kingma_adam_2014} with a starting learning rate of $0.0002$ and a batch size of $128$.
We train all models for 400 epochs and $3$ seeds.
The implementation follows the one described in \cref{app:experiments_polymnist}.

We use NVIDIA GTX 2080 GPUs for all our runs.  Each experiment can be run with $4$ CPU workers and $16$ GB of memory.
An average run takes around $24$ hours. To train all methods used in this paper, we had to train $10 \times 3 \times 6 = 180$ different models:
$10$ different $\beta$ values, $3$ different seeds, and $6$ different methods.
Hence, the total GPU compute time used to generate the results for the PolyMNIST dataset equals $180 \times 24 \approx 4320$ hours.
We---of course---also had to invest GPU time to develop the method, which we did not measure.

\subsubsection{Additional Results}
Given the multilabel nature of the CelebA dataset, we evaluate the learned latent representation with respect to the individual attributes and not only the average performance across all attributes.
\Cref{fig:exp_celeba_downstream_attributes} shows the detailed results according to the full set of $40$ attributes for the three methods: independent VAEs, aggregated VAE, and MMVM VAE.
We again train linear binary classifiers on inferred representations of the training set and evaluate them on representations of the test set.
However, we now report the individual performance of every classifier.
In the main text (see \cref{fig:exp_benchmarks_downstream_celeba,fig:exp_benchmarks_coherence_celeba}), we report the average performance of the $40$ binary classifiers.

\begin{figure*}
    \centering
    \begin{subfigure}[t]{0.95\textwidth}
        \centering
        \includegraphics[width=0.75\textwidth]{figures/legend_coherence_recloss.png}
    \end{subfigure}
    \centering
    \begin{subfigure}[t]{0.95\textwidth}
        \includegraphics[width=\textwidth]{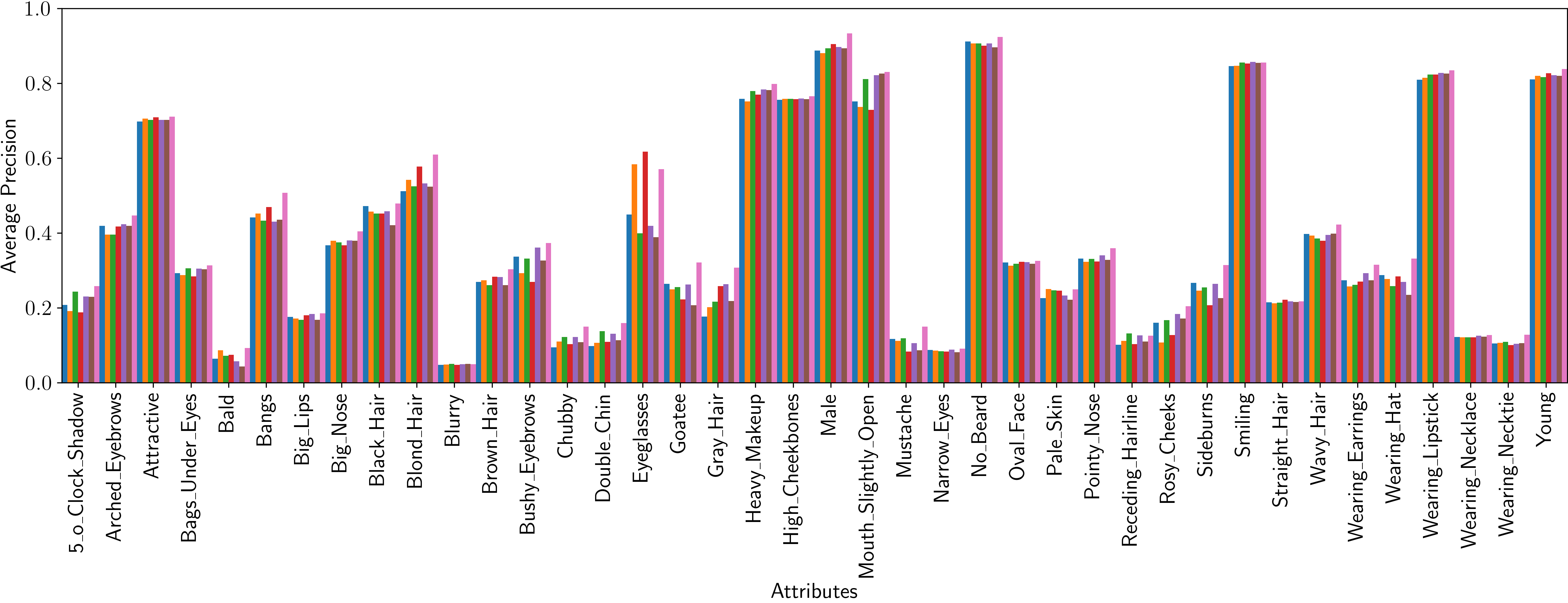}
        \caption{Image Representations}
        \label{fig:exp_celeba_downstream_attributes_img}
    \end{subfigure}
    \hfill
    \centering
    \begin{subfigure}[t]{0.95\textwidth}
        \includegraphics[width=\textwidth]{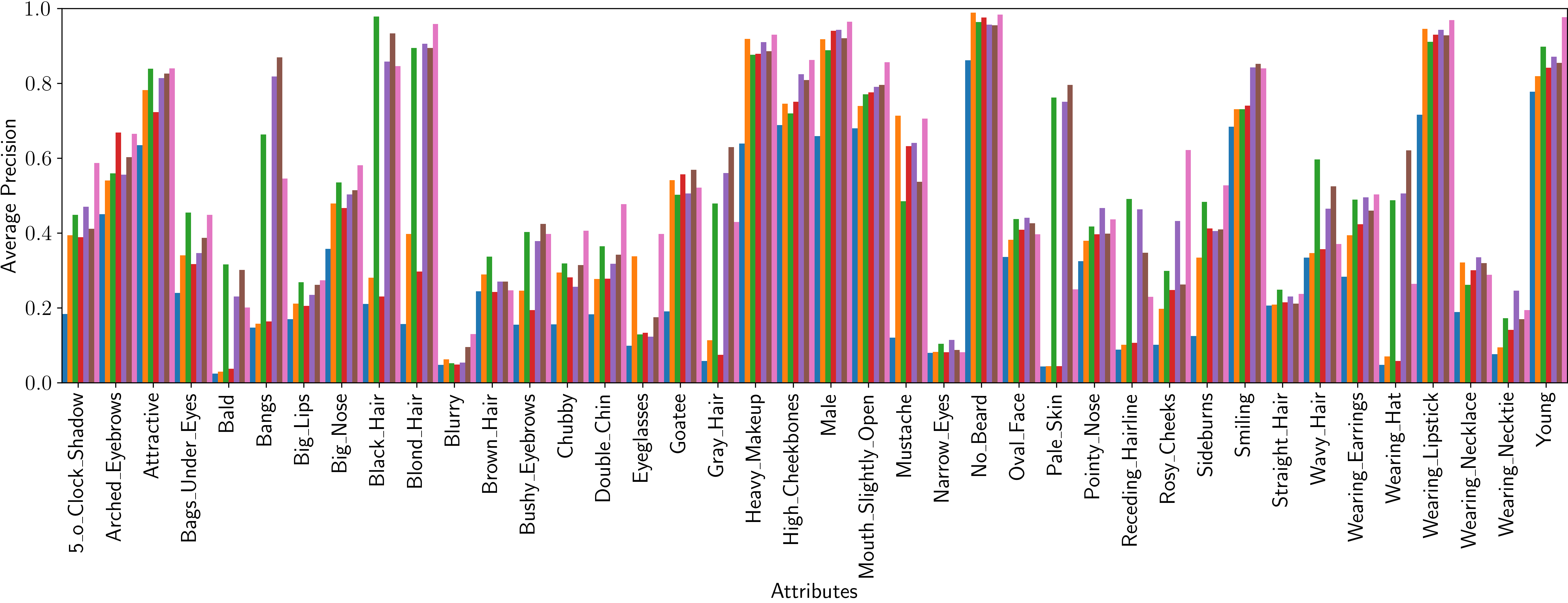}
        \caption{Text Representations}
        \label{fig:exp_celeba_downstream_attributes_text}
    \end{subfigure}
    \hspace{0.25cm}
    \caption{
    Attribute-level results on the bimodal CelebA datasets for the latent representation classification.
    The MMVM VAE outperforms the independent VAEs and the aggregated VAE on most attributes.    
    }
    \label{fig:exp_celeba_downstream_attributes}
\end{figure*}

\begin{figure*}
    \centering
    \begin{subfigure}[t]{0.3\textwidth}
        \centering
        \includegraphics[width=1.0\textwidth]{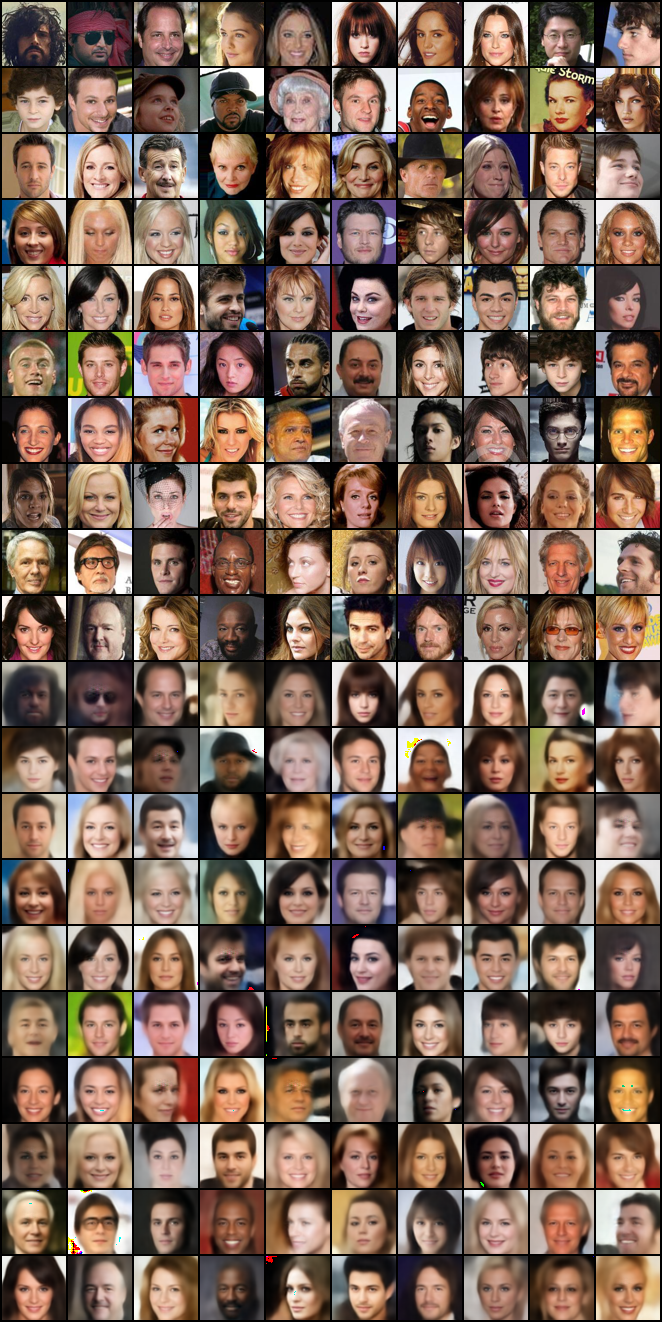}
        \caption{Independent VAEs}
        \label{fig:exp_celeba_qualitative_imgtoimg_independent}
    \end{subfigure}
    \hfill
    \begin{subfigure}[t]{0.3\textwidth}
        \includegraphics[width=1.0\textwidth]{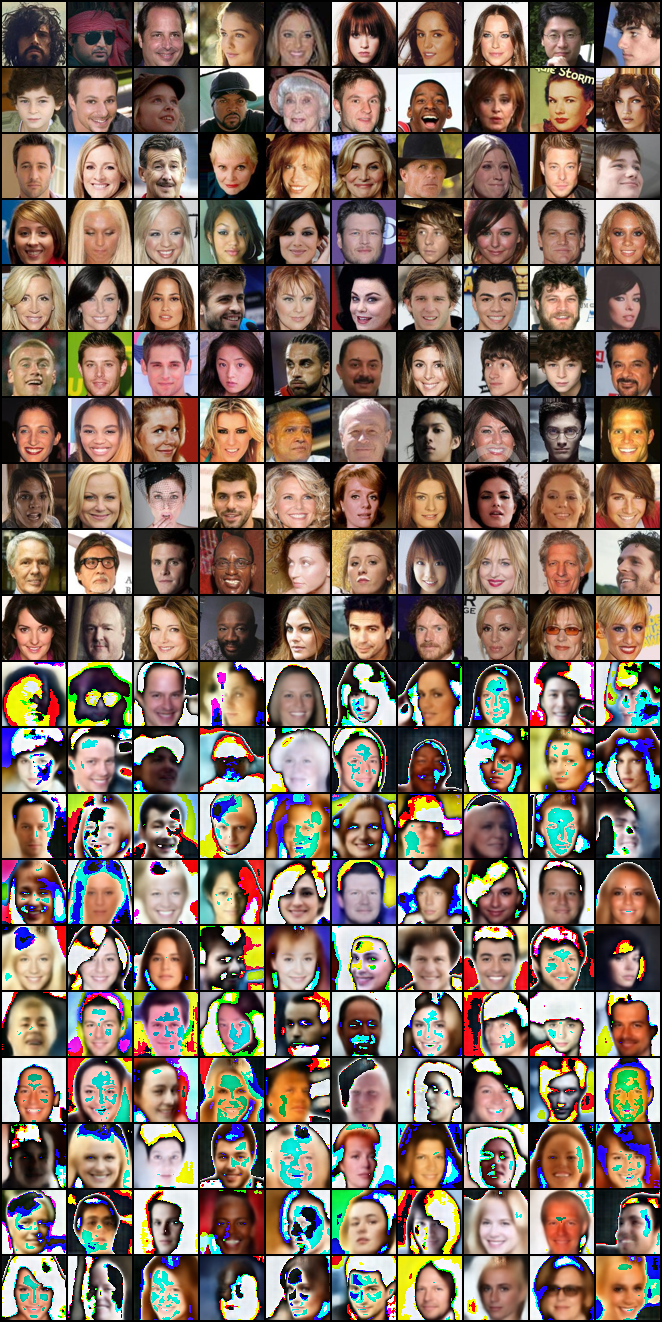}
        \caption{AVG VAE}
        \label{fig:exp_celeba_qualitative_imgtoimg_aggregated}
    \end{subfigure}
    \hfill
    \begin{subfigure}[t]{0.3\textwidth}
        \includegraphics[width=1.0\textwidth]{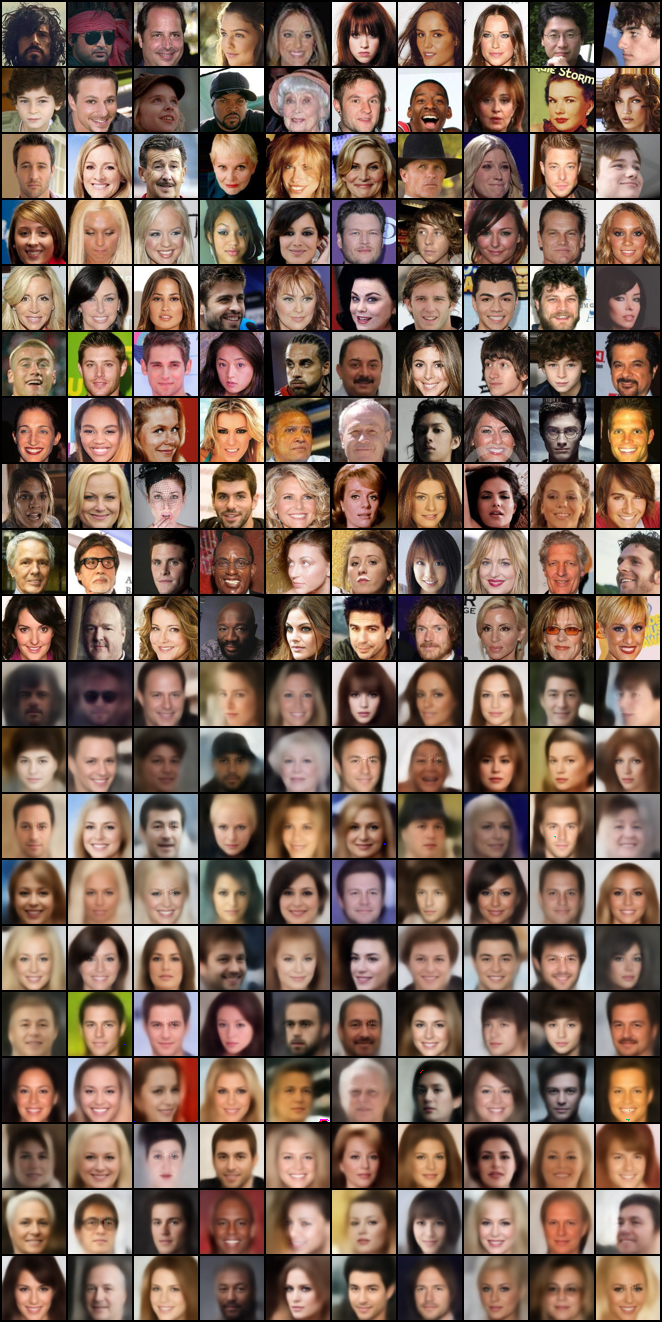}
        \caption{MMVM VAE}
        \label{fig:exp_celeba_qualitative_imgtoimg_mmvamp}
    \end{subfigure}
    \caption{
    Qualitative Results for the CelebA dataset on the image-to-image generation task.
    The first 10 rows of every subplot show the input image and the bottom 10 rows its conditional generation.
    Different to the training, we provide only the image to every model and based on the latent representation of that image, we generate a sample.
    We see that the aggregated VAE (\cref{fig:exp_celeba_qualitative_imgtoimg_aggregated}) is not able to conditionally generate visually pleasing samples compared to the independent VAEs (\cref{fig:exp_celeba_qualitative_imgtoimg_independent}) and the MMVM VAE (\cref{fig:exp_celeba_qualitative_imgtoimg_mmvamp}).
    }
    \label{fig:exp_celeba_condgen_imgtoimg_qualitative}
\end{figure*}
\begin{figure*}
    \centering
    \begin{subfigure}[t]{0.3\textwidth}
        \centering
        \includegraphics[width=1.0\textwidth]{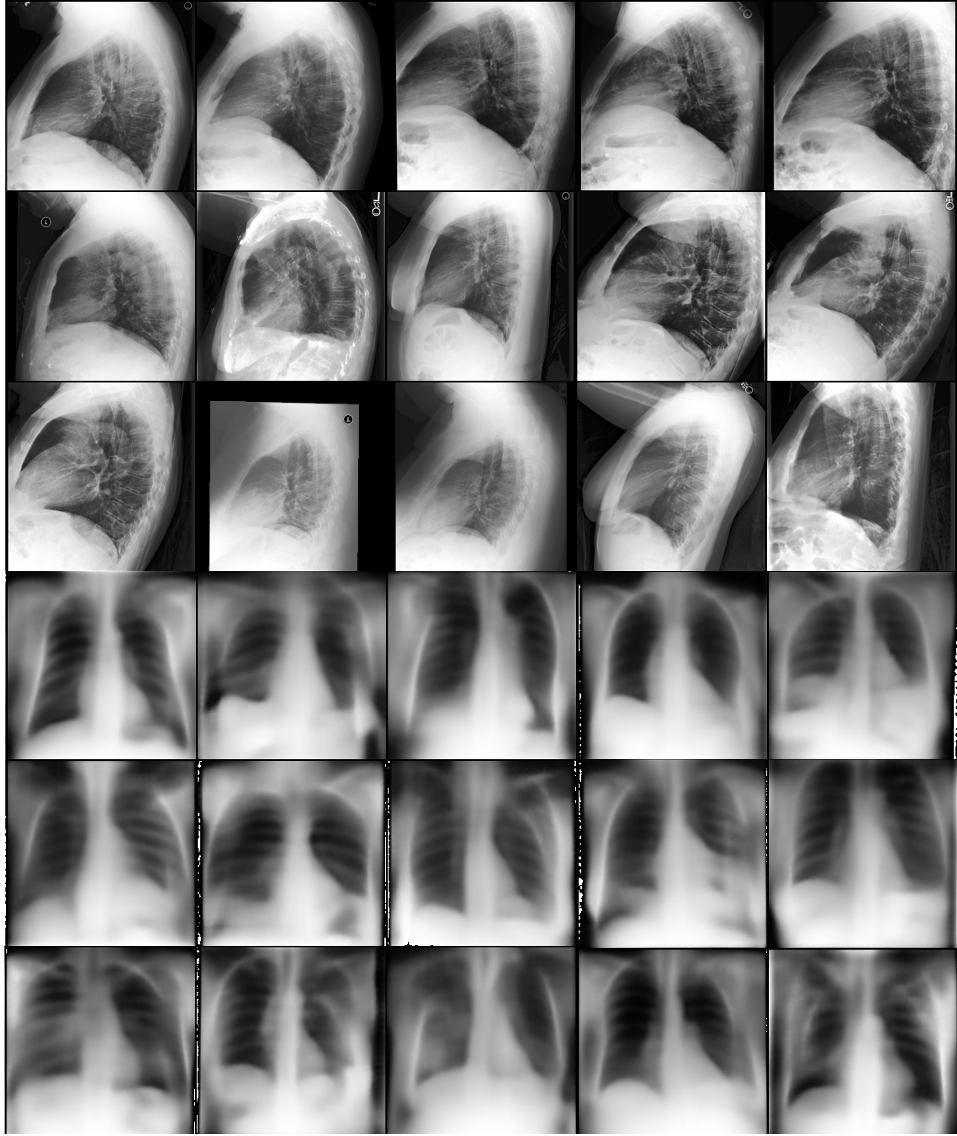}
        \caption{Independent VAEs}
        \label{fig:exp_mimic_qualitative_imgtoimg_independent}
    \end{subfigure}
    \hfill
    \begin{subfigure}[t]{0.3\textwidth}
        \includegraphics[width=1.0\textwidth]{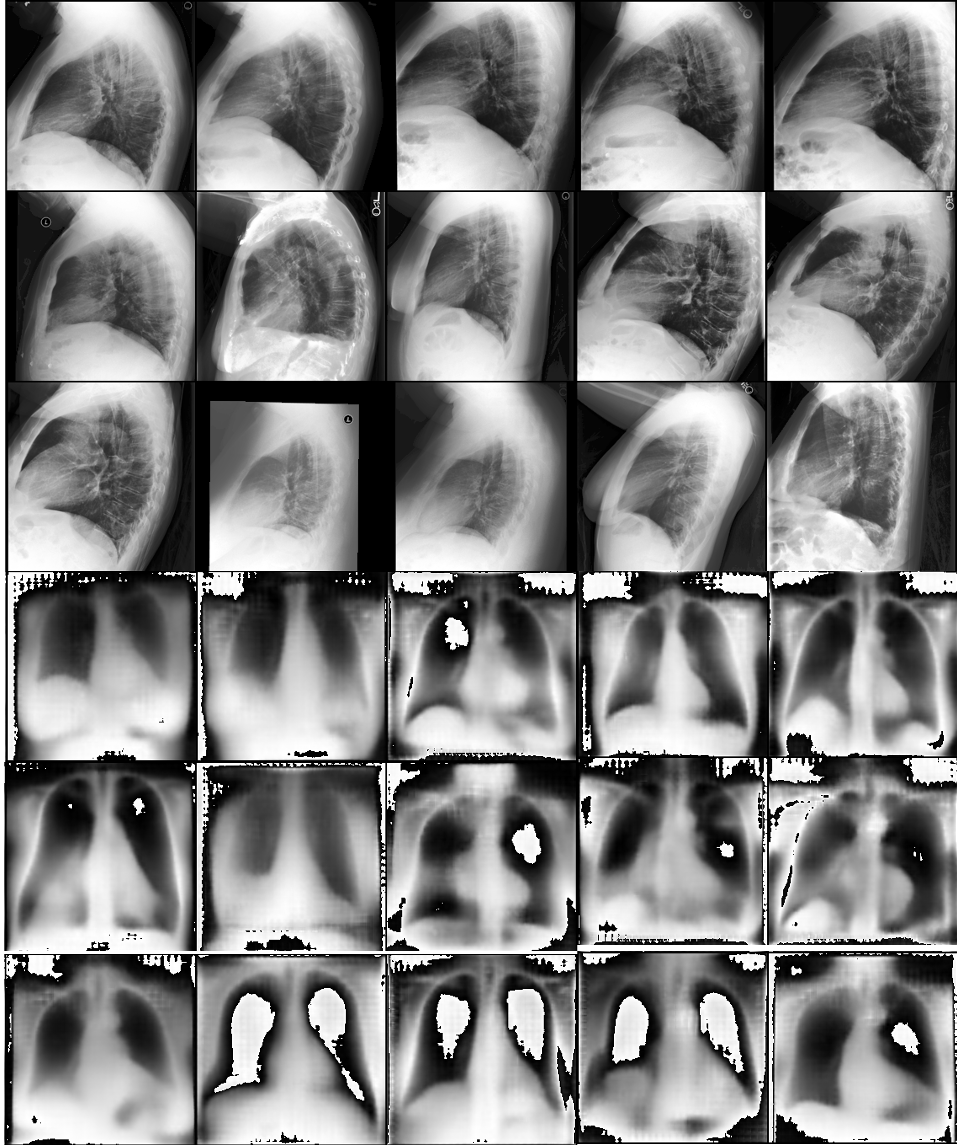}
        \caption{PoE VAE}
        \label{fig:exp_mimic_qualitative_imgtoimg_aggregated}
    \end{subfigure}
    \hfill
    \begin{subfigure}[t]{0.3\textwidth}
        \includegraphics[width=1.0\textwidth]{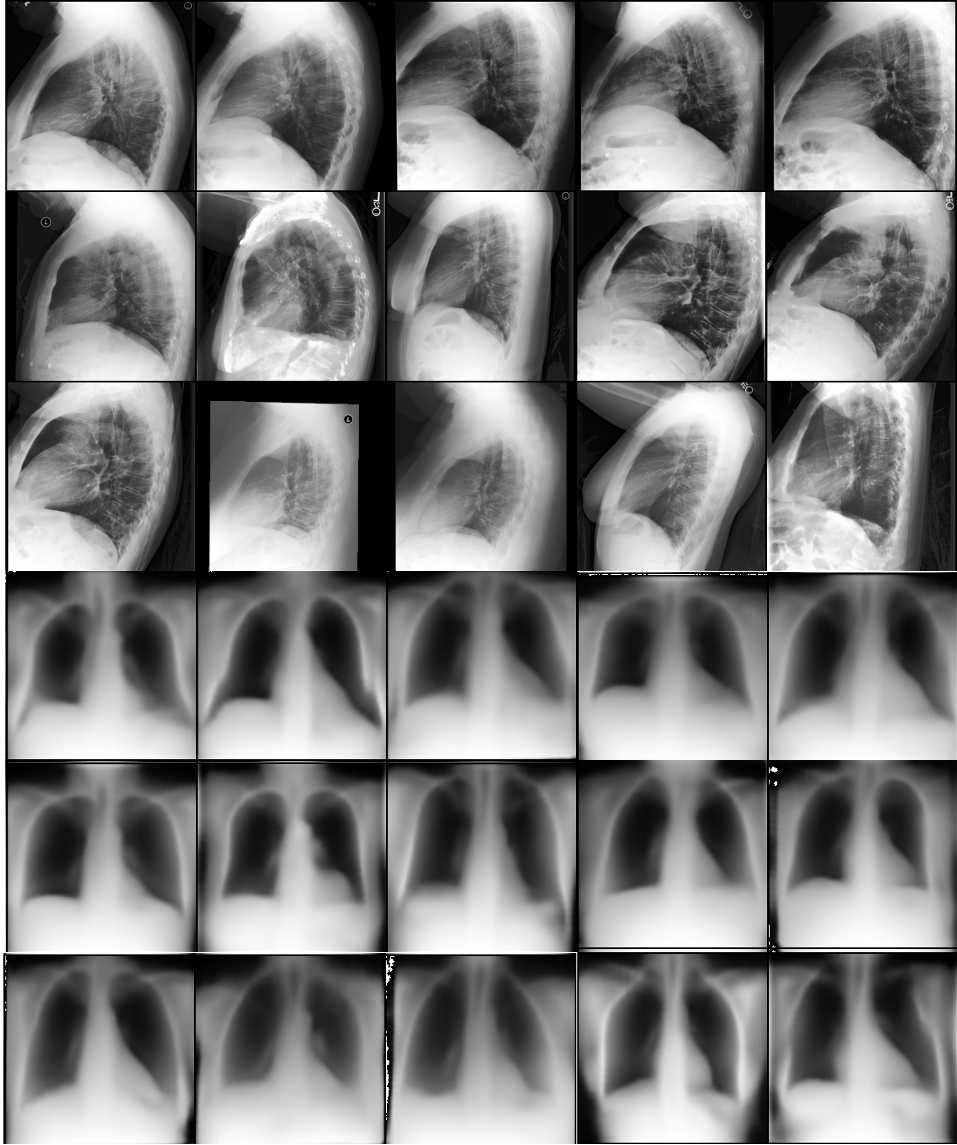}
        \caption{MMVM VAE}
        \label{fig:exp_mimic_qualitative_imgtoimg_mmvamp}
    \end{subfigure}
    \caption{
    Qualitative results for the MIMIC-CXR dataset on the conditional generation task lateral to frontal. Results are aligned with the other datasets.
    }
    \label{fig:exp_mimic_condgen_ftol_qualitative}
\end{figure*}
\begin{figure}
    \centering
    \includegraphics[width=1\textwidth]{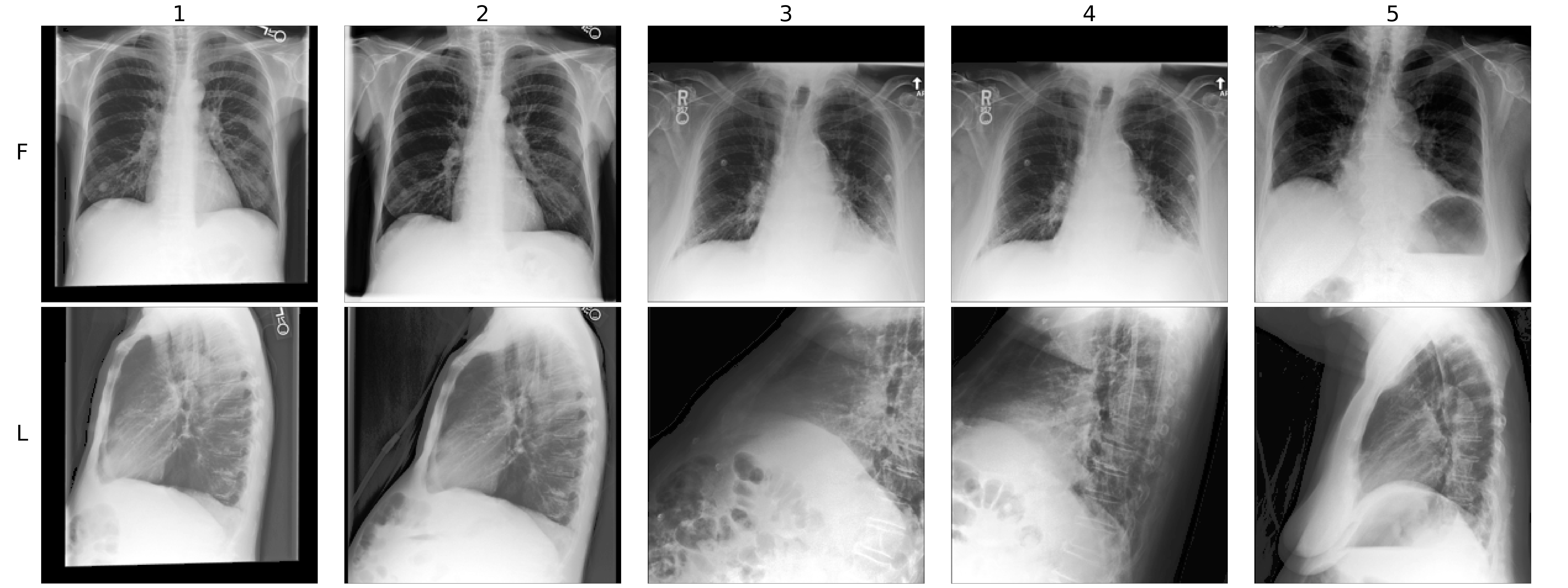}
    \caption{
    MIMIC-CXR experiment dataset: every column is a bimodal tuple $\bm{X}$, the top row shows samples of the frontal modality $\bm{x}_f$, and the bottom row shows samples of the lateral modality $\bm{x}_l$.
    The first two tuples are linked to \textit{No Findings}, indicative of healthy conditions. Tuples three and four are labeled with \textit{Consolidation} disease. The tuple five is labeled with \textit{Atelectasis} disease. We can observe that tuples three and four share the same frontal image, but they differ due to having distinct lateral images.
    }
    \label{fig:exp_data_bimodalmimic}
\end{figure}

\subsection{CUB Image-Captions}
\label{app:cub}

\subsubsection{Dataset}

The dataset contains 11,788 images of birds and 117,880 (10 times as many) captions, each image with 10 fine-grained captions describing the bird’s appearance characteristics collected through Amazon Mechanical Turk (AMT). We performed a 75-25 training-test split, with 8,855 and 2,933 images, and 88,550 and 29,330 corresponding captions in the training set and test set, respectively. We resized the images to 3 channels and 64 by 64 pixels and modeled the captions using embeddings.

The labels applied to downstream tasks are based on the bird's primary colors. Based on the primary color attribute, we remapped the original 15-class primary color label set into a 6-class label set to address the sparsity issue. The labels and mappings are blue-to-red (blue, iridescent, purple, green, pink, red), brown (brown), grey (grey), yellow (yellow, olive, orange, buff), black (black), and white (white).

\subsubsection{Implementation Details \& Training}

The network structures are similar to the ones in other experiments. For both image and caption data, we employed a CNN encoder and decoder. We used a 128-dimensional latent space with a Laplace likelihood on images and a Categorical likelihood for captions. The activation function was chosen to be ReLU between layers. For all experiments on this dataset, we used an Adam optimizer with an initial learning rate of 0.001 and a batch size of 128. We trained all models for 100 epochs.

\subsection{MIMIC-CXR}
\label{app:mimiccxr}

\subsubsection{Dataset}

The dataset we use in our experiment is a multimodal interpretation of the original MIMIC-CXR dataset. The original MIMIC-CXR dataset comprises high-resolution chest X-ray images related to imaging studies. A study may include multiple chest X-ray images captured from several view positions. We categorized these views into two primary modalities: \textit{frontal} (including “AP” and “PA” views) and \textit{lateral} (including “LL” and “Lateral” views). For each study, we pair every frontal image with every lateral image in all possible combinations. Studies lacking at least one frontal and one lateral image are excluded. This approach formalizes a new dataset composed of image pairs, thus offering a bimodal interpretation of the original MIMIC-CXR dataset. More rigorously, we define a dataset $\mathbb{X} = \{ \bm{X}^{(i)} \}_{i=1}^n$ where each $\bm{X}^{(i)} = \{ \bm{x}_f^{(i)}, \bm{x}_l^{(i)} \}$ is a bimodal tuple composed of one frontal image and one lateral image of the same study. An image may appear in multiple tuples, but we ensure that each tuple is unique by having at least one different image. Examples of these bimodal tuples are illustrated in Figure \ref{fig:exp_data_bimodalmimic}.

We use the JPG version of the MIMIC-CXR dataset, namely the MIMIC-CXR-JPG dataset \citep{johnson2024mimic}. In our preprocessing pipeline, we apply center cropping and downscale the images to a resolution of $224\times 224$. We utilize the labels from the MIMIC-CXR-JPG dataset which are obtained using the CheXpert tool \citep{irvin2019chexpert}. All non-positive labels (including “negative,” “non-mentioned,” or “uncertain”) were combined into an aggregate “negative” label following the approach adopted by \cite{Haque2021.07.30.21261225}. Each imaging study is connected to a subject. We split the dataset into distinct training (80\%), validation (10\%), and test (10\%) sets based on subjects, thus ensuring that the same image or study cannot be present in multiple sets.

\subsubsection{Implementation \& Training}

\paragraph{Multimodal VAEs Implementation \& Training}
We use ResNet-based encoders and decoders for this experiment \citep{he_deep_2016}, similar to those used in the PolyMNIST and the Bimodal CelebA experiments. The image encoder and decoder consist of 2D convolution layers. The architectural design of the image encoders and decoders is uniform for both frontal and lateral modalities. We use an Adam optimizer \citep{kingma_adam_2014} with a learning rate of $0.00005$ and a batch size of $32$. We train all methods for $240$ epochs and $3$ seeds. The implementation follows the one described in \cref{app:experiments_polymnist}.

We use NVIDIA A100-SXM4-40GB GPUs for all our runs. An average run, evaluating one method on one seed, takes approximately $45$ hours. To train all methods evaluated in this experiment, we had to train $3 \times 3 = 9$ different models: $3$ seeds and $3$ methods. Hence, the total GPU compute time used to generate the VAE results for the MIMIC-CXR experiment is around $45 \times 9 = 405$ hours. We also had to use GPU time in the development process, which we did not measure.

\paragraph{Supervised classifier Implementation \& Training}

We use supervised classifiers based on Resnet.
The classifier architecture is derived from the VAEs encoder we use in this experiment. To adapt the encoder for classification tasks, we added a linear layer equipped with $14$ neurons corresponding to the number of labels in the MIMIC-CXR dataset. We use an Adam optimizer with a learning rate of $0.0001$ and a batch size of $256$. We train distinct models for both frontal and lateral modalities, training each for $30$ epochs and $3$ seeds.

We train the classifiers using NVIDIA A100-SXM4-40GB GPUs. To train one classifier on one seed takes approximately $1$ hour. We train a total of $3\times2=6$ classifiers: $3$ seeds and $2$ modalities, resulting in about $6$ hours of GPU compute time. We also had to use GPU time in the development process, which we did not measure. 
\subsubsection{Additional Results}
\label{app:mimiccxr_addresults}

In our main experiment, we evaluate the quality of the learned latent representations using Random Forest (RF) classifiers. Specifically, we train independent RF binary classifiers for each model and each label on the inferred representations of the training set and evaluate them on the representations of the test set. The RF classifiers are configured with 5,000 estimators and a maximum depth of 30. In \cref{tab:exp_mimic_cxr} in \cref{sec:exp_mimic}, we report for each model the average performance over the two modalities and three seeds, totaling six RF classifiers per model: one for each modality and each of the three different seeds.

Here, we provide detailed insights into the capabilities of the models to leverage different modalities during training to improve the unimodal representations by reporting the performance of the latent representation classification for each modality separately in \cref{tab:mimic_cxr_frontal} and \cref{tab:mimic_cxr_lateral}. These results allow us to highlight the strengths of the MMVM model in improving unimodal representations with additional modalities. For instance, in the main text we mentioned that in the Cardiomegaly classification task, the MMVM lateral representations $z_L$ lead to substantial improvements over the frontal representations $z_F$ of the other VAEs, even though the former modality $x_L$ is less informative than the latter $x_F$ as indicated by the respective performance of the supervised model on each modality. This behavior is not specific to Cardiomegaly and can be observed for other labels such as Atelectasis, Lung Opacity, No Finding, Pleural Effusion, Pneumonia, and Support Devices. Even though the difference between modalities for these labels is sometimes small, the consistency of those results demonstrate the soft-sharing capabilities of the MMVM approach.
Conversely, in tasks where the frontal view is less informative, such as Consolidation, Enlarged Cardiomediastinum, and Pneumothorax, the MMVM's lateral representations outperform the frontal ones of the other VAEs. 

The efficacy of the soft sharing mechanism in the MMVM VAE is also reflected in the average performance for all labels. The performance difference between the two unimodal latent representations produced by the MMVM VAE is substantially smaller (0.9 percentage points on average) compared to the independent VAEs (2.5 percentage points on average) and the PoE VAE (2.2 percentage points on average).

\begin{table}
\caption{Evaluation of the VAEs' frontal latent representations $z_F$ classification performance on the test split. The performance of a fully-supervised non-linear deep network is included for reference. The average AUROC [\%] and standard deviation over three seeds are reported. Enl. Cardiom. stands for Enlarged Cardiomediastinum and Support Dev. for Support Device.}
\label{tab:mimic_cxr_frontal}
\vspace{2pt}
\centering
\resizebox{\textwidth}{!}{%
\begin{tabular}{lccccccc}
\toprule
 & \small\textit{supervised} & \small independent & AVG & MoE & MoPoE & PoE & MMVM \\
\midrule
Atelectasis & \textit{\textbf{\small79.5} \scriptsize $\pm$0.3} & \small73.1 \scriptsize $\pm$0.0 & \small75.2 \scriptsize $\pm$0.3 & \small73.0 \scriptsize $\pm$0.5 & \small74.2 \scriptsize $\pm$0.4 & \small75.7 \scriptsize $\pm$0.3 & \small77.6 \scriptsize $\pm$0.1 \\
Cardiomegaly & \textit{\textbf{\small81.7} \scriptsize $\pm$0.1} & \small76.3 \scriptsize $\pm$0.4 & \small78.5 \scriptsize $\pm$0.2 & \small76.5 \scriptsize $\pm$0.6 & \small77.1 \scriptsize $\pm$0.1 & \small78.5 \scriptsize $\pm$0.3 & \small80.5 \scriptsize $\pm$0.1 \\
Consolidation & \textit{\small65.3 \scriptsize $\pm$0.7} & \small62.4 \scriptsize $\pm$0.4 & \small66.0 \scriptsize $\pm$0.8 & \small62.9 \scriptsize $\pm$0.6 & \small63.9 \scriptsize $\pm$0.3 & \small66.7 \scriptsize $\pm$0.8 & \textbf{\small69.1} \scriptsize $\pm$0.6 \\
Edema & \textit{\textbf{\small88.0} \scriptsize $\pm$0.2} & \small83.0 \scriptsize $\pm$0.3 & \small84.6 \scriptsize $\pm$0.3 & \small82.4 \scriptsize $\pm$0.6 & \small83.1 \scriptsize $\pm$0.6 & \small84.5 \scriptsize $\pm$0.3 & \small86.3 \scriptsize $\pm$0.1 \\
Enl. Cardiom. & \textit{\small57.9 \scriptsize $\pm$1.3} & \small59.5 \scriptsize $\pm$1.2 & \small64.9 \scriptsize $\pm$0.8 & \small61.7 \scriptsize $\pm$0.5 & \small64.1 \scriptsize $\pm$0.5 & \small66.4 \scriptsize $\pm$0.5 & \textbf{\small68.6} \scriptsize $\pm$1.1 \\
Fracture & \textit{\small51.4 \scriptsize $\pm$0.5} & \small56.0 \scriptsize $\pm$0.1 & \small58.4 \scriptsize $\pm$0.5 & \small57.3 \scriptsize $\pm$0.2 & \small57.4 \scriptsize $\pm$0.7 & \textbf{\small58.8} \scriptsize $\pm$0.3 & \small58.6 \scriptsize $\pm$0.8 \\
Lung Lesion & \textit{\small52.4 \scriptsize $\pm$0.4} & \small61.3 \scriptsize $\pm$0.4 & \small61.8 \scriptsize $\pm$0.2 & \small60.6 \scriptsize $\pm$0.9 & \small60.7 \scriptsize $\pm$0.8 & \small63.4 \scriptsize $\pm$0.4 & \textbf{\small64.1} \scriptsize $\pm$0.2 \\
Lung Opacity & \textit{\textbf{\small69.5} \scriptsize $\pm$0.2} & \small63.8 \scriptsize $\pm$0.3 & \small65.7 \scriptsize $\pm$0.3 & \small63.5 \scriptsize $\pm$0.2 & \small64.7 \scriptsize $\pm$0.0 & \small66.3 \scriptsize $\pm$0.2 & \small68.1 \scriptsize $\pm$0.1 \\
No Finding & \textit{\small73.9 \scriptsize $\pm$1.3} & \small76.6 \scriptsize $\pm$0.3 & \small77.8 \scriptsize $\pm$0.0 & \small77.1 \scriptsize $\pm$0.2 & \small77.4 \scriptsize $\pm$0.1 & \small77.2 \scriptsize $\pm$0.2 & \textbf{\small79.1} \scriptsize $\pm$0.1 \\
Pleural Effusion & \textit{\textbf{\small88.0} \scriptsize $\pm$0.0} & \small81.2 \scriptsize $\pm$0.6 & \small82.8 \scriptsize $\pm$0.0 & \small81.6 \scriptsize $\pm$0.4 & \small82.5 \scriptsize $\pm$0.5 & \small82.8 \scriptsize $\pm$0.2 & \small85.7 \scriptsize $\pm$0.3 \\
Pleural Other & \textit{\small53.9 \scriptsize $\pm$1.0} & \small67.8 \scriptsize $\pm$1.1 & \small68.9 \scriptsize $\pm$0.5 & \small67.5 \scriptsize $\pm$1.0 & \small68.3 \scriptsize $\pm$1.2 & \small68.7 \scriptsize $\pm$1.2 & \textbf{\small70.0} \scriptsize $\pm$2.0 \\
Pneumonia & \textit{\textbf{\small61.3} \scriptsize $\pm$0.3} & \small55.3 \scriptsize $\pm$0.5 & \small57.8 \scriptsize $\pm$0.4 & \small56.4 \scriptsize $\pm$0.4 & \small57.3 \scriptsize $\pm$0.4 & \small57.5 \scriptsize $\pm$0.4 & \small60.0 \scriptsize $\pm$0.6 \\
Pneumothorax & \textit{\small76.9 \scriptsize $\pm$1.1} & \small75.3 \scriptsize $\pm$1.0 & \small78.3 \scriptsize $\pm$0.6 & \small77.7 \scriptsize $\pm$0.4 & \small78.8 \scriptsize $\pm$0.7 & \small78.5 \scriptsize $\pm$0.7 & \textbf{\small81.9} \scriptsize $\pm$0.4 \\
Support Dev. & \textit{\textbf{\small77.7} \scriptsize $\pm$0.1} & \small70.8 \scriptsize $\pm$0.3 & \small73.1 \scriptsize $\pm$0.4 & \small72.7 \scriptsize $\pm$0.7 & \small73.7 \scriptsize $\pm$0.6 & \small73.8 \scriptsize $\pm$0.1 & \small76.6 \scriptsize $\pm$0.2 \\
\midrule
All Labels & \textit{\small69.8 \scriptsize $\pm$12.6} & \small68.7 \scriptsize $\pm$9.0 & \small71.0 \scriptsize $\pm$8.6 & \small69.4 \scriptsize $\pm$8.8 & \small70.2 \scriptsize $\pm$8.8 & \small71.3 \scriptsize $\pm$8.4 & \textbf{\small73.3} \scriptsize $\pm$8.9 \\
\bottomrule
\end{tabular}
}
\end{table}

\begin{table}
\caption{Evaluation of the VAEs' lateral latent representation $z_L$ classification performance on the test split. The performance of a fully-supervised non-linear deep network is included for reference. The average AUROC [\%] and standard deviation over three seeds are reported. Enl. Cardiom. stands for Enlarged Cardiomediastinum and Support Dev. for Support Device.}
\vspace{2pt}
\centering
\resizebox{\textwidth}{!}{%
\begin{tabular}{lccccccc}
\toprule
 & \small\textit{supervised} & \small independent & AVG & MoE & MoPoE & PoE & MMVM \\
\midrule
Atelectasis & \textit{\textbf{\small78.0} \scriptsize $\pm$0.1} & \small70.7 \scriptsize $\pm$0.3 & \small73.5 \scriptsize $\pm$0.4 & \small72.8 \scriptsize $\pm$0.1 & \small74.7 \scriptsize $\pm$0.2 & \small73.7 \scriptsize $\pm$0.2 & \small77.0 \scriptsize $\pm$0.2 \\
Cardiomegaly & \textit{\textbf{\small79.0} \scriptsize $\pm$0.2} & \small70.8 \scriptsize $\pm$0.9 & \small73.7 \scriptsize $\pm$0.1 & \small73.3 \scriptsize $\pm$0.2 & \small75.5 \scriptsize $\pm$0.1 & \small74.8 \scriptsize $\pm$0.1 & \small78.7 \scriptsize $\pm$0.0 \\
Consolidation & \textit{\small68.6 \scriptsize $\pm$1.4} & \small64.4 \scriptsize $\pm$1.4 & \small65.4 \scriptsize $\pm$1.5 & \small64.9 \scriptsize $\pm$0.9 & \small65.8 \scriptsize $\pm$0.8 & \small66.7 \scriptsize $\pm$0.9 & \textbf{\small70.2} \scriptsize $\pm$0.8 \\
Edema & \textit{\textbf{\small86.2} \scriptsize $\pm$0.1} & \small75.4 \scriptsize $\pm$0.9 & \small78.0 \scriptsize $\pm$0.3 & \small78.0 \scriptsize $\pm$0.5 & \small81.1 \scriptsize $\pm$0.8 & \small79.1 \scriptsize $\pm$0.1 & \small84.3 \scriptsize $\pm$0.3 \\
Enl. Cardiom. & \textit{\small61.5 \scriptsize $\pm$1.0} & \small60.1 \scriptsize $\pm$0.7 & \small62.0 \scriptsize $\pm$1.0 & \small60.5 \scriptsize $\pm$0.5 & \small64.2 \scriptsize $\pm$0.9 & \small63.5 \scriptsize $\pm$0.8 & \small69.0 \scriptsize $\pm$0.7 \\
Fracture & \textit{\small52.3 \scriptsize $\pm$0.2} & \small57.9 \scriptsize $\pm$0.6 & \small58.3 \scriptsize $\pm$0.7 & \small56.8 \scriptsize $\pm$0.8 & \small58.6 \scriptsize $\pm$0.8 & \small59.0 \scriptsize $\pm$0.5 & \textbf{\small60.9} \scriptsize $\pm$0.3 \\
Lung Lesion & \textit{\small57.3 \scriptsize $\pm$0.4} & \small58.9 \scriptsize $\pm$0.2 & \small59.0 \scriptsize $\pm$0.2 & \small58.6 \scriptsize $\pm$0.8 & \small60.8 \scriptsize $\pm$0.3 & \small59.3 \scriptsize $\pm$0.3 & \textbf{\small63.0} \scriptsize $\pm$0.7 \\
Lung Opacity & \textit{\textbf{\small68.9} \scriptsize $\pm$0.2} & \small61.9 \scriptsize $\pm$0.5 & \small63.4 \scriptsize $\pm$0.4 & \small63.9 \scriptsize $\pm$0.1 & \small65.4 \scriptsize $\pm$0.4 & \small64.1 \scriptsize $\pm$0.4 & \small68.1 \scriptsize $\pm$0.2 \\
No Finding & \textit{\small72.0 \scriptsize $\pm$0.8} & \small73.9 \scriptsize $\pm$0.3 & \small74.8 \scriptsize $\pm$0.2 & \small75.9 \scriptsize $\pm$0.2 & \small77.1 \scriptsize $\pm$0.1 & \small74.6 \scriptsize $\pm$0.1 & \textbf{\small78.3} \scriptsize $\pm$0.1 \\
Pleural Effusion & \textit{\textbf{\small91.0} \scriptsize $\pm$0.3} & \small80.2 \scriptsize $\pm$0.2 & \small82.0 \scriptsize $\pm$0.1 & \small82.0 \scriptsize $\pm$0.3 & \small84.3 \scriptsize $\pm$0.2 & \small82.1 \scriptsize $\pm$0.1 & \small85.7 \scriptsize $\pm$0.1 \\
Pleural Other & \textit{\small61.5 \scriptsize $\pm$1.9} & \small62.8 \scriptsize $\pm$1.3 & \small64.3 \scriptsize $\pm$0.7 & \small62.7 \scriptsize $\pm$1.7 & \small63.6 \scriptsize $\pm$1.0 & \small63.9 \scriptsize $\pm$1.0 & \textbf{\small68.5} \scriptsize $\pm$1.9 \\
Pneumonia & \textit{\textbf{\small61.2} \scriptsize $\pm$0.4} & \small56.4 \scriptsize $\pm$0.5 & \small56.9 \scriptsize $\pm$0.3 & \small57.5 \scriptsize $\pm$1.2 & \small58.3 \scriptsize $\pm$0.5 & \small58.2 \scriptsize $\pm$0.1 & \small59.0 \scriptsize $\pm$0.2 \\
Pneumothorax & \textit{\textbf{\small82.4} \scriptsize $\pm$0.8} & \small75.6 \scriptsize $\pm$0.5 & \small77.8 \scriptsize $\pm$0.3 & \small76.9 \scriptsize $\pm$0.7 & \small79.2 \scriptsize $\pm$0.6 & \small78.6 \scriptsize $\pm$0.2 & \small81.7 \scriptsize $\pm$0.3 \\
Support Dev. & \textit{\textbf{\small77.3} \scriptsize $\pm$0.1} & \small71.9 \scriptsize $\pm$0.6 & \small72.9 \scriptsize $\pm$0.5 & \small73.6 \scriptsize $\pm$0.7 & \small75.9 \scriptsize $\pm$0.4 & \small74.7 \scriptsize $\pm$0.5 & \small77.1 \scriptsize $\pm$0.3 \\
\midrule
All Labels & \textit{\small71.2 \scriptsize $\pm$11.3} & \small67.2 \scriptsize $\pm$7.6 & \small68.7 \scriptsize $\pm$8.1 & \small68.4 \scriptsize $\pm$8.4 & \small70.3 \scriptsize $\pm$8.6 & \small69.4 \scriptsize $\pm$8.0 & \textbf{\small73.0} \scriptsize $\pm$8.5 \\
\bottomrule
\end{tabular}
}
\label{tab:mimic_cxr_lateral}
\end{table}

\cref{fig:exp_mimic_cxr_latentreprclassif} illustrates the performance of the average unimodal latent representation $z$ classification described above against the reconstruction loss for the three different VAE methods introduced in the main text (independent, aggregate PoE VAE and the proposed MMVM VAE). As a reminder, the AUROC performance [in \%] is averaged over three seeds and the two modalities. 

\begin{figure}
    \centering
    \includegraphics[width=0.5\textwidth]{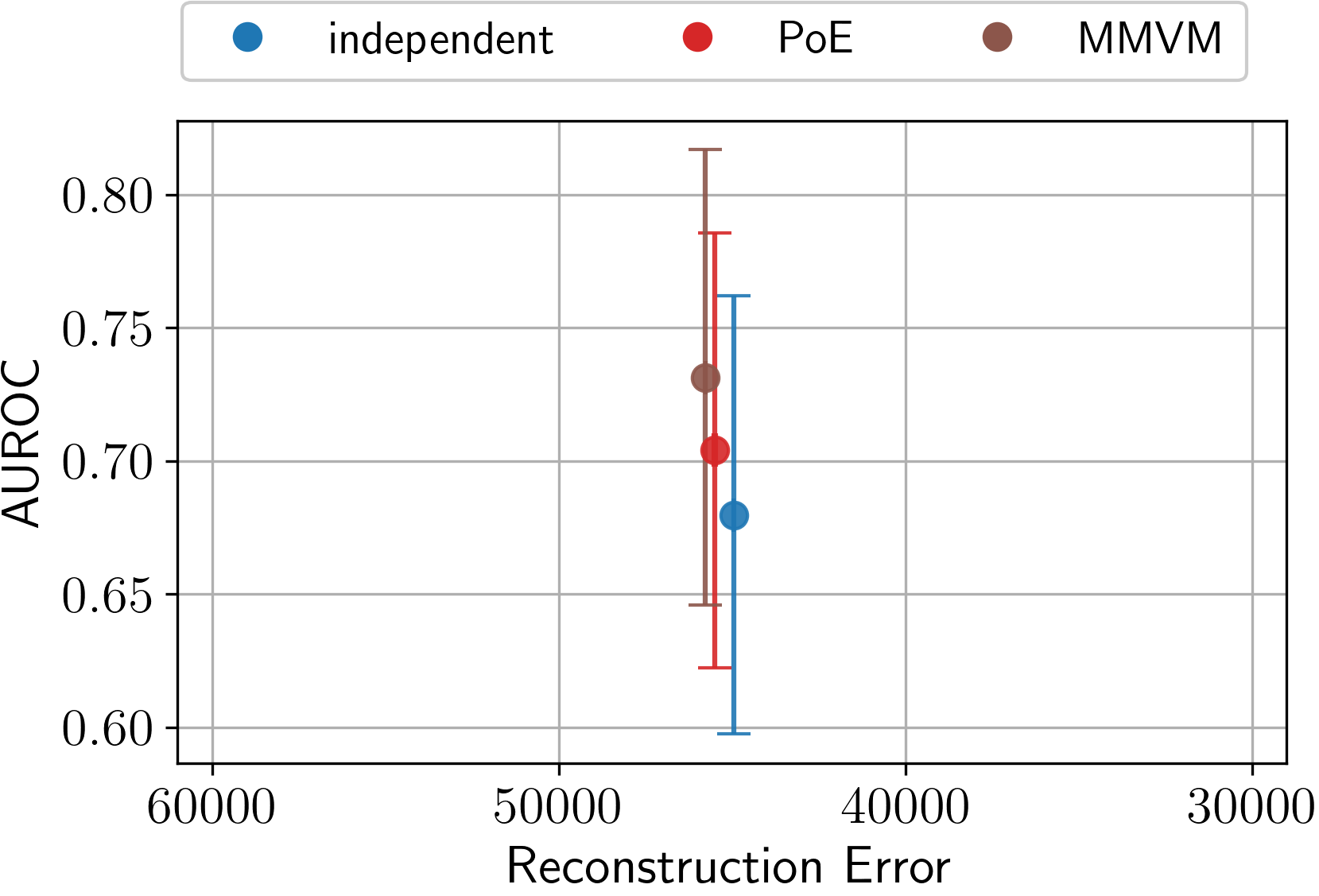}
    \caption{Latent representation classification for the MIMIC-CXR dataset. The mean AUROC over all labels and averaged over three seeds is reported.}
    \label{fig:exp_mimic_cxr_latentreprclassif}
\end{figure}

\subsection{Hippocampal Neural Activities}
\label{app:experiments_rats}
\subsubsection{Dataset}
The training data was collected from 250 ms length time frames after the port entry. Due to the behavior difference from each rat (some rats react faster to the odors while some others react slower), the training time frames of the five rats started from 250 ms, 250 ms, 500 ms, 500 ms, and 250 ms, respectively. During training, we treated each data point as independent and trained all the VAE models based on sliding windows (100 ms sub-window, 10 ms steps; 16 data points per window on each trial). The 100 ms sub-windows constituted the input data, with the dimension equal to the rat’s number of neurons multiplied by 10, as the data was further binned into 10-ms increments.

\subsubsection{Implementation \& Training}
We use the same network architectures for all multimodal VAEs. 
Each of the autoencoders includes its unique encoder and decoder, both containing two hidden layers, without weight-sharing during training and evaluation. 
All modalities share the same architecture but the layers' dimensions are different, with 920, 790, 1040, 490, 460 dimensional input and hidden layers, respectively. The activation function was chosen to be LeakyReLU with a 0.01 negative slope. 
For all experiments on this dataset, we use an Adam optimizer with an initial learning rate of $0.001$, a batch size of $128$. We train all models for $1000$ epochs.

\subsubsection{Additional Results}
We show the 2-dimensional latent representations for every rat and the six VAE encoders in \Cref{fig:rats_latent_odor} and \Cref{fig:rats_latent_rat}. In these two figures, each dot is the two-dimensional latent representation of a 100 ms sub-window of one odor trial for one rat. \Cref{fig:rats_latent_odor} is colored by 4 odors, and \Cref{fig:rats_latent_rat} is colored by 5 modalities (rats). \Cref{fig:rats_latent_odor} shows the odor stimuli separation on the latent space and how good MMVM VAE is in separating the odors. \Cref{fig:rats_latent_rat} shows that the proposed MMVM VAE can best utilize the shared information between the five rats by pulling the latent representations together. At the same time, the independent AVG and PoE baseline models fail to extract the information shared between rats. Although it shows separation in some views, the independent model does not provide a connection between views. The five tiny clusters in \Cref{fig:rats_latent_odor} and \Cref{fig:rats_latent_rat} show that, instead of showing a clear odor separation on the latent space, the AVG model separated the data by rats. The results went against the intention to share information across views. In other words, the five rats' latent representations were far away from each other, so the aggregated VAE completely failed to connect the five views. 

\begin{figure*}[t!]
    \centering
    \begin{subfigure}[t]{0.95\textwidth}
        \centering
        \includegraphics[width=0.75\textwidth]{figures/odor_legend.png}
    \end{subfigure}
    \centering
    \begin{subfigure}[t]{0.25\textwidth}
        \includegraphics[width=1.0\textwidth]{figures/latent_odor_uni.png}
        \caption{independent}
        \label{fig:app_rats_latent_odor_independent}
    \end{subfigure}
    \hspace{0.25cm}
    \centering
    \begin{subfigure}[t]{0.25\textwidth}
        \includegraphics[width=1.0\textwidth]{figures/latent_odor_joint.png}
        \caption{AVG}
        \label{fig:app_rats_latent_odor_joint}
    \end{subfigure}
    \hspace{0.25cm}
    \centering
    \begin{subfigure}[t]{0.25\textwidth}
        \includegraphics[width=1.0\textwidth]{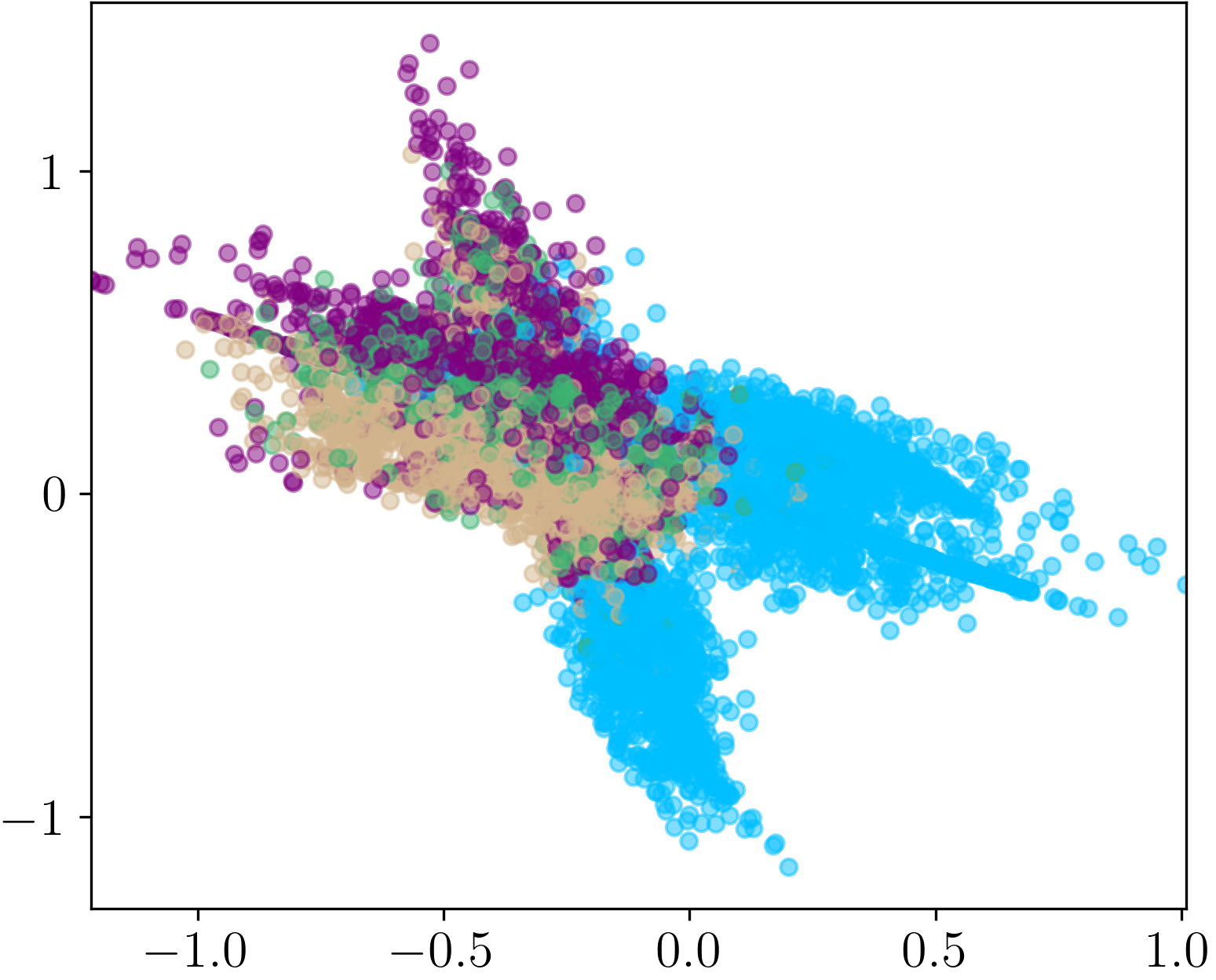}
        \caption{MoE}
        \label{fig:app_rats_latent_odor_moe}
    \end{subfigure}
    \hspace{0.25cm}
    \centering
    \begin{subfigure}[t]{0.25\textwidth}
        \includegraphics[width=1.0\textwidth]{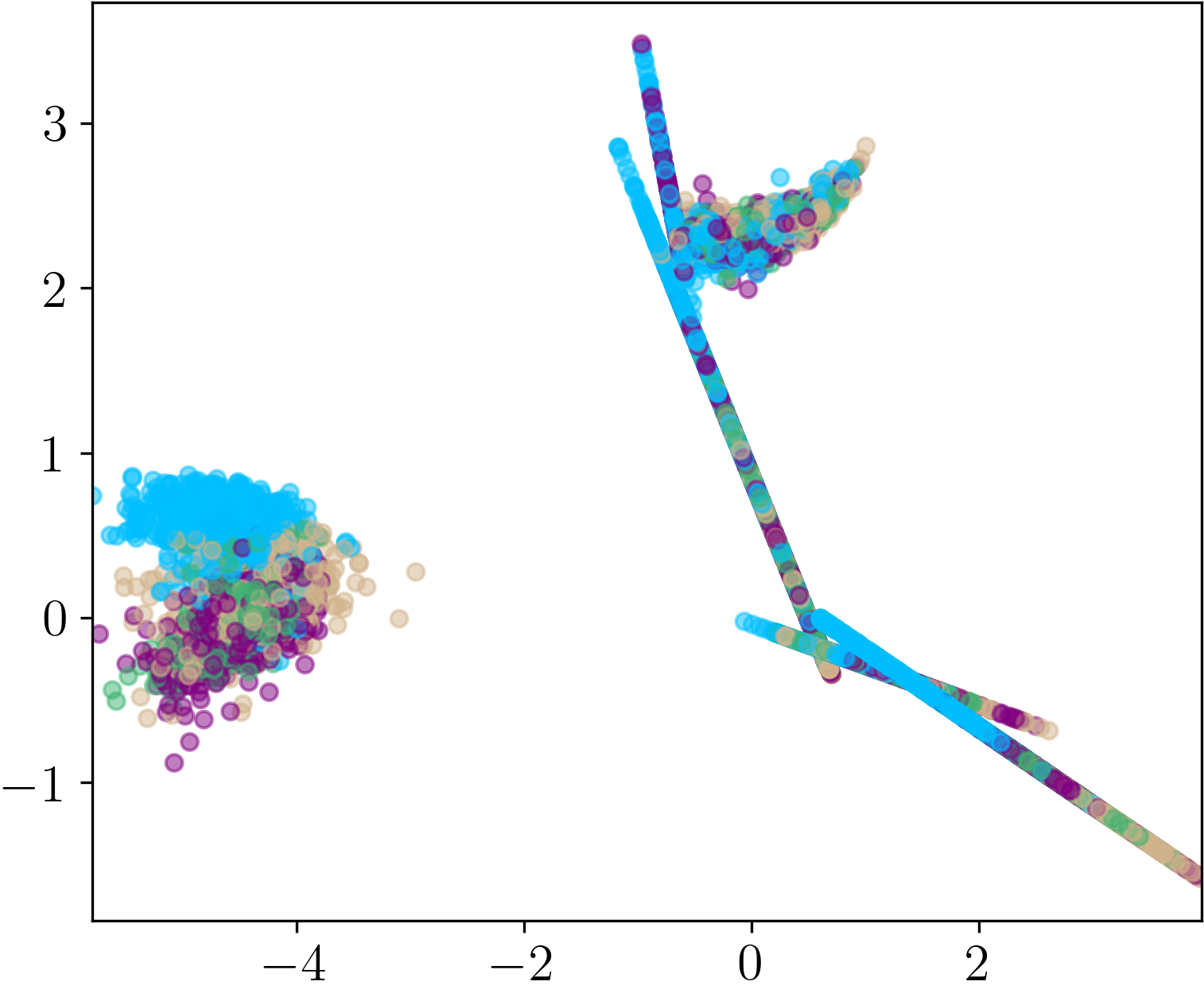}
        \caption{PoE}
        \label{fig:app_rats_latent_odor_poe}
    \end{subfigure}
    \hspace{0.25cm}
    \centering
    \begin{subfigure}[t]{0.25\textwidth}
        \includegraphics[width=1.0\textwidth]{figures/latent_odor_mopoe.png}
        \caption{MoPoE}
        \label{fig:app_rats_latent_odor_mopoe}
    \end{subfigure}
    \hspace{0.25cm}
    \centering
    \begin{subfigure}[t]{0.25\textwidth}
        \includegraphics[width=1.0\textwidth]{figures/latent_odor_mixedprior.png}
        \caption{MMVM}
        \label{fig:app_rats_latent_odor_mixedprior}
    \end{subfigure}
    \hspace{0.25cm}
    \caption{Latent neural representation during a memory experiment. Each model's performance is evaluated based on its own optimal $\beta$ value (0.00001, 0.01, 0.001 for independent, aggregated, and MMVM respectively) in terms of the self-conditioned latent representation classification accuracy according to \Cref{fig:exp_rats_downstream_recloss}. Our model can distinguish the odor stimuli in the latent space with a clear separation of odors (4 different colors). 
    }
    \label{fig:app_rats_latent_odor}
\end{figure*}

\begin{figure*}
    \centering
    \begin{subfigure}[t]{0.95\textwidth}
        \centering
        \includegraphics[width=0.75\textwidth]{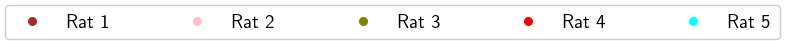}
    \end{subfigure}
    \centering
    \begin{subfigure}[t]{0.25\textwidth}
        \includegraphics[width=1\textwidth]{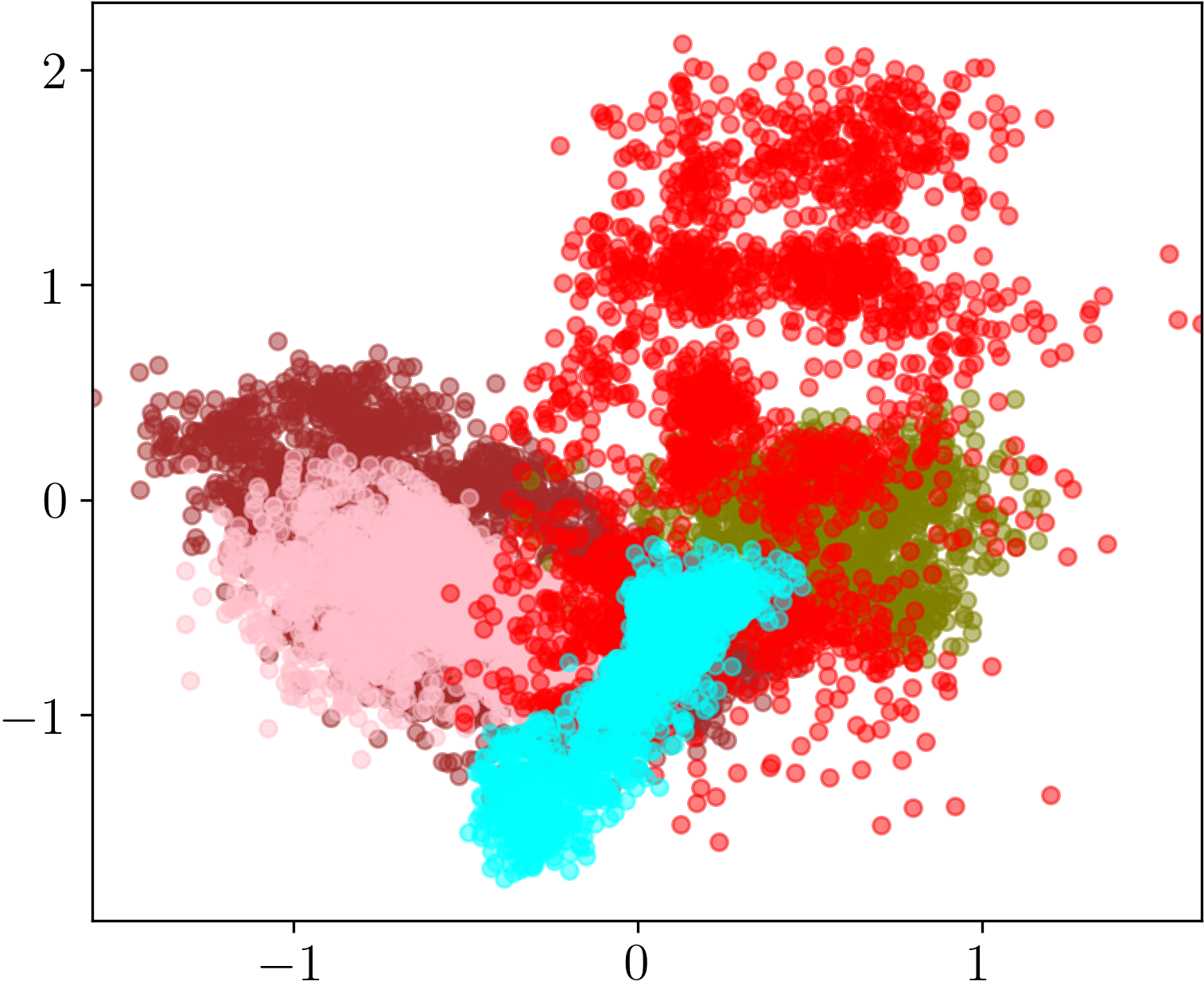}
        \caption{independent}
        \label{fig:rats_latent_rat_independent}
    \end{subfigure}
    \hspace{0.25cm}
    \centering
    \begin{subfigure}[t]{0.25\textwidth}
        \includegraphics[width=1\textwidth]{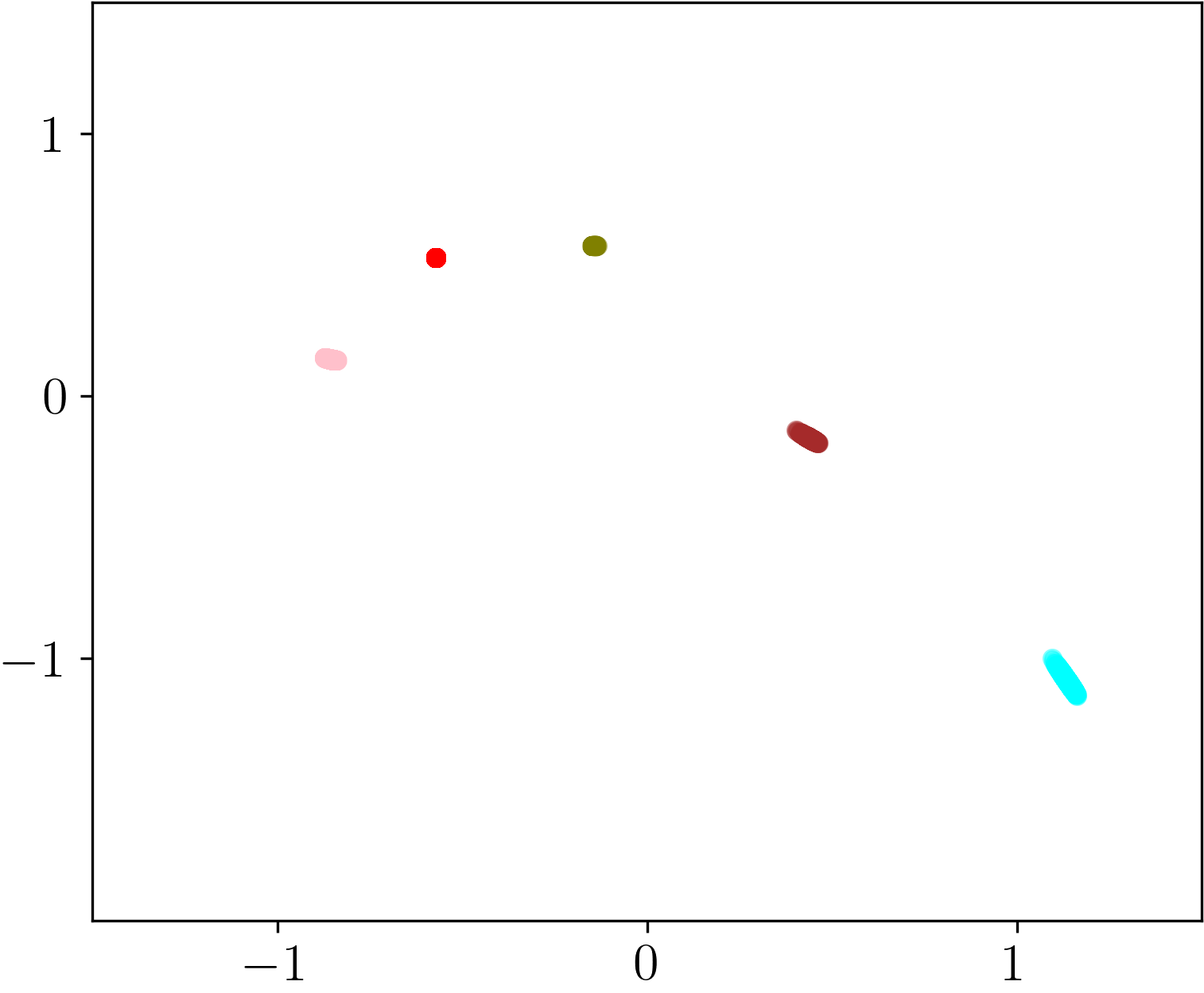}
        \caption{AVG}
        \label{fig:rats_latent_rat_joint}
    \end{subfigure}
    \hspace{0.25cm}
    \centering
    \begin{subfigure}[t]{0.25\textwidth}
        \includegraphics[width=1\textwidth]{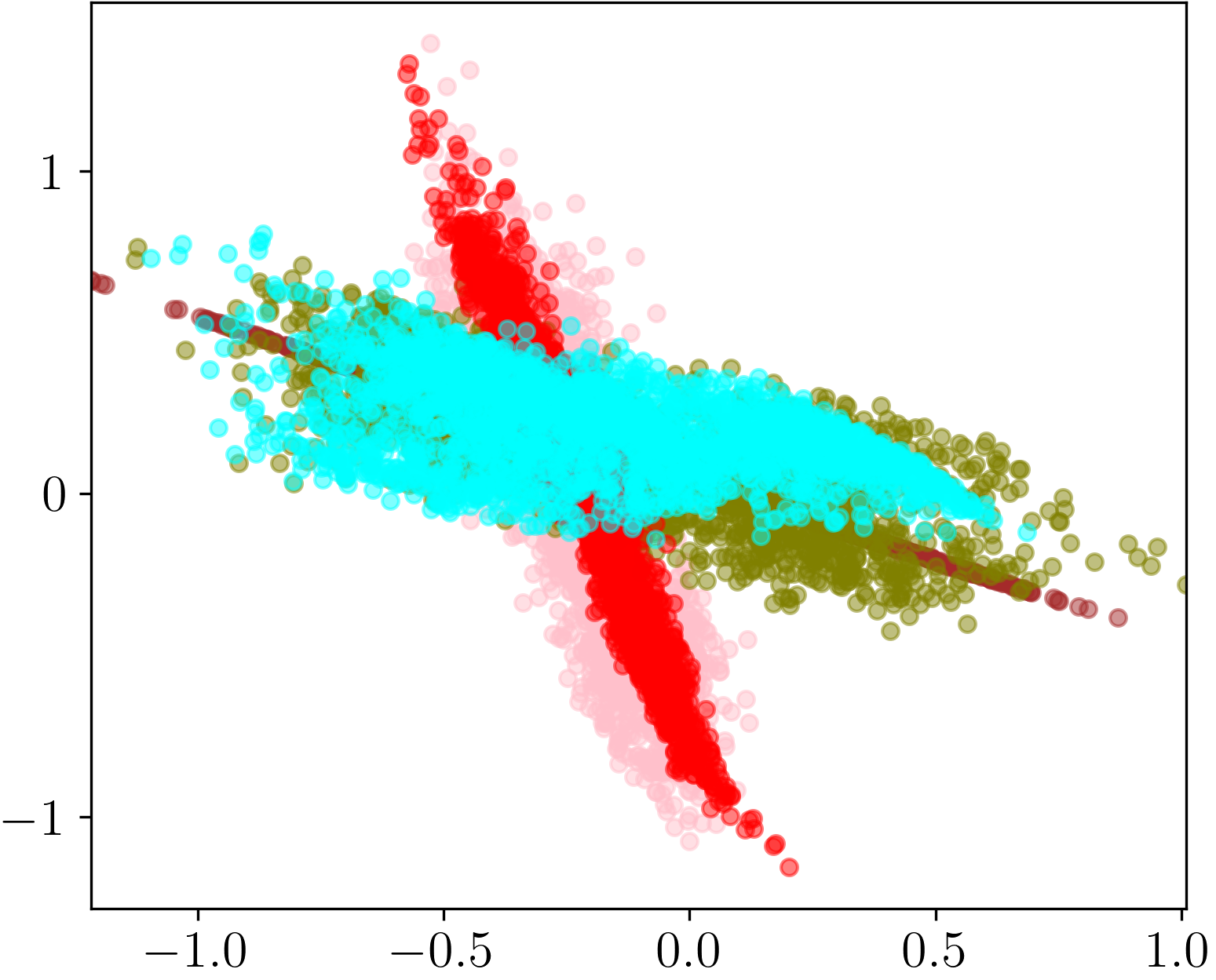}
        \caption{MoE}
        \label{fig:rats_latent_rat_moe}
    \end{subfigure}
    \hspace{0.25cm}
    \centering
    \begin{subfigure}[t]{0.25\textwidth}
        \includegraphics[width=1\textwidth]{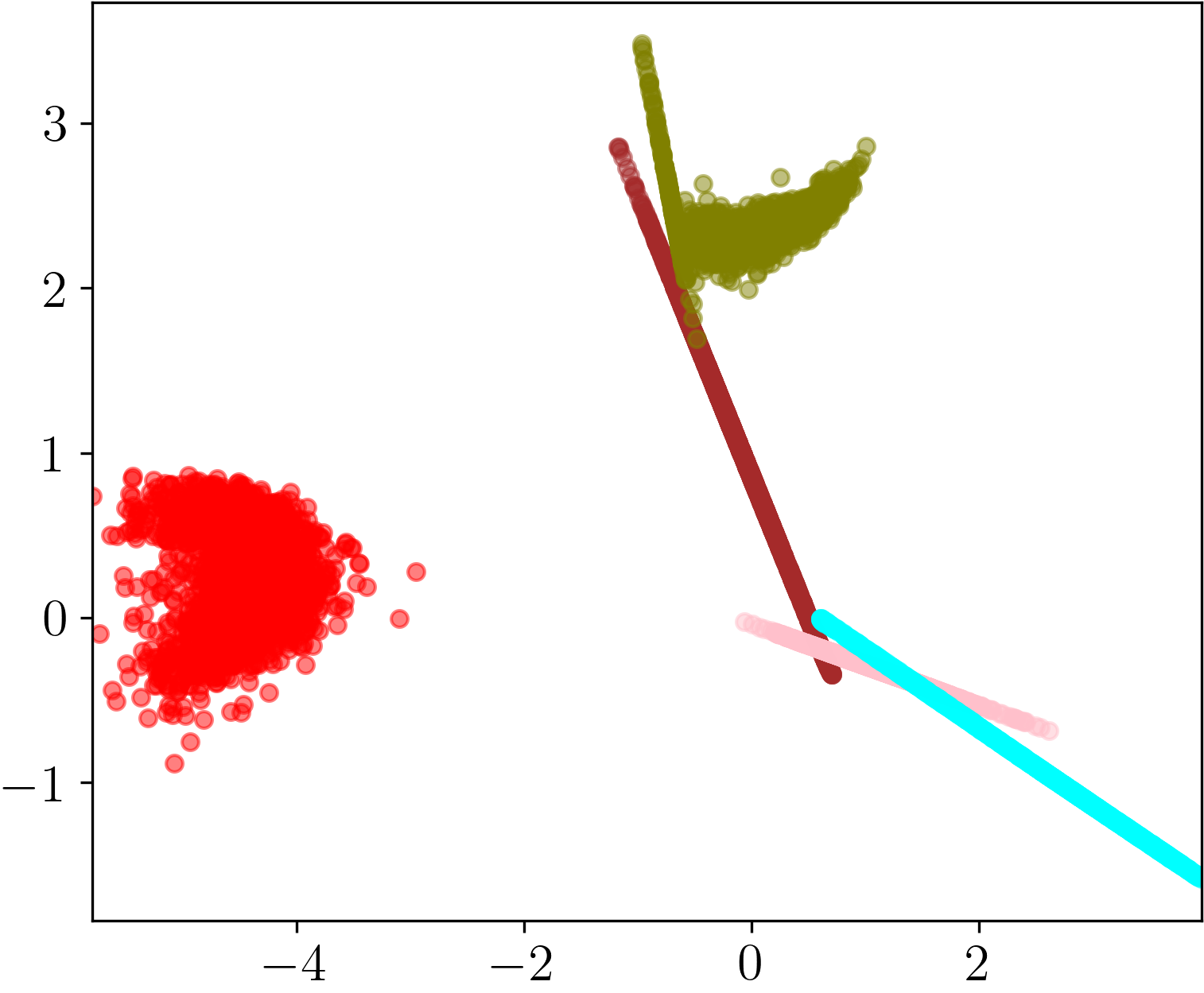}
        \caption{PoE}
        \label{fig:rats_latent_rat_poe}
    \end{subfigure}
    \hspace{0.25cm}
    \centering
    \begin{subfigure}[t]{0.25\textwidth}
        \includegraphics[width=1\textwidth]{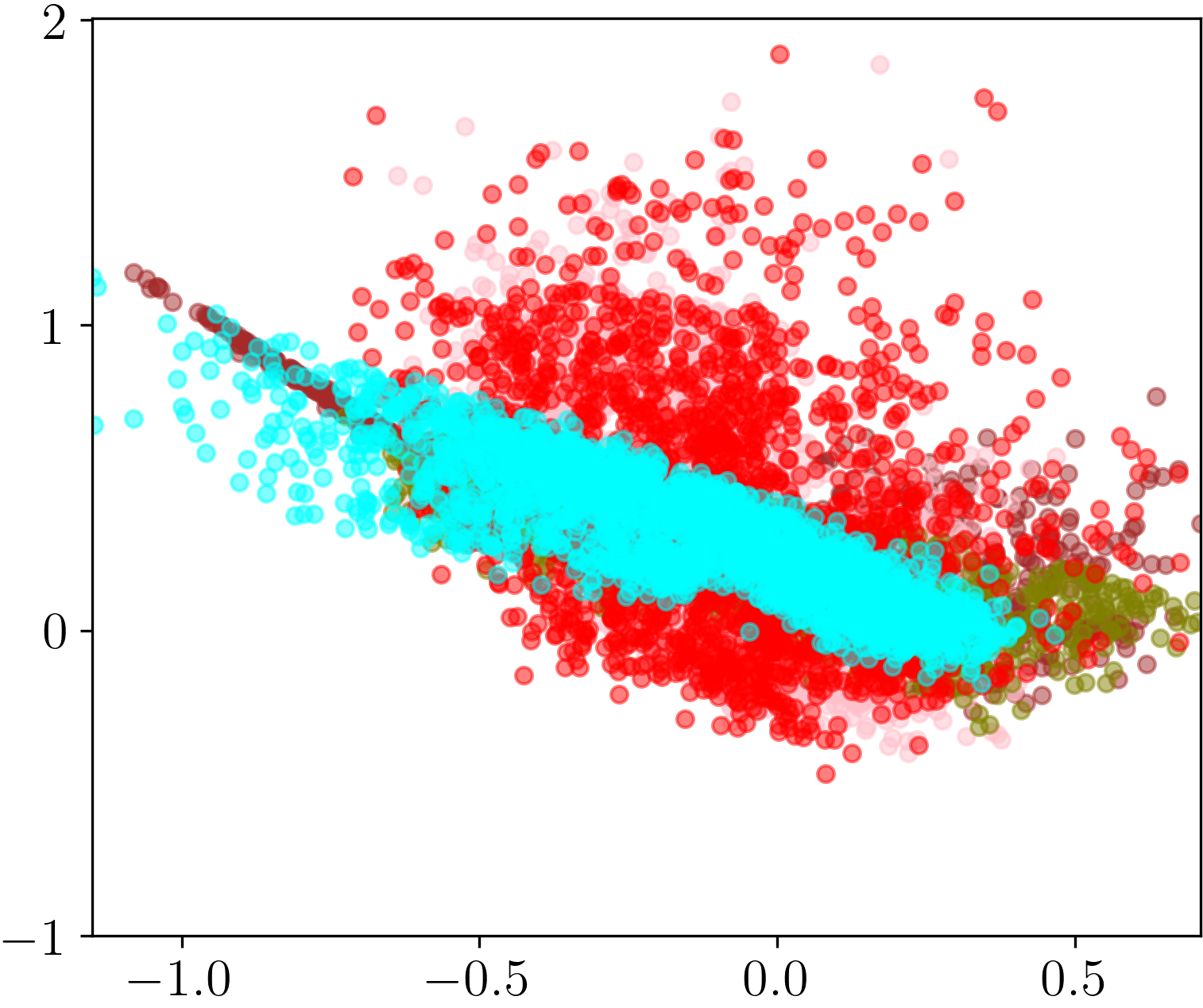}
        \caption{MoPoE}
        \label{fig:rats_latent_rat_mopoe}
    \end{subfigure}
    \hspace{0.25cm}
    \centering
    \begin{subfigure}[t]{0.25\textwidth}
        \includegraphics[width=1\textwidth]{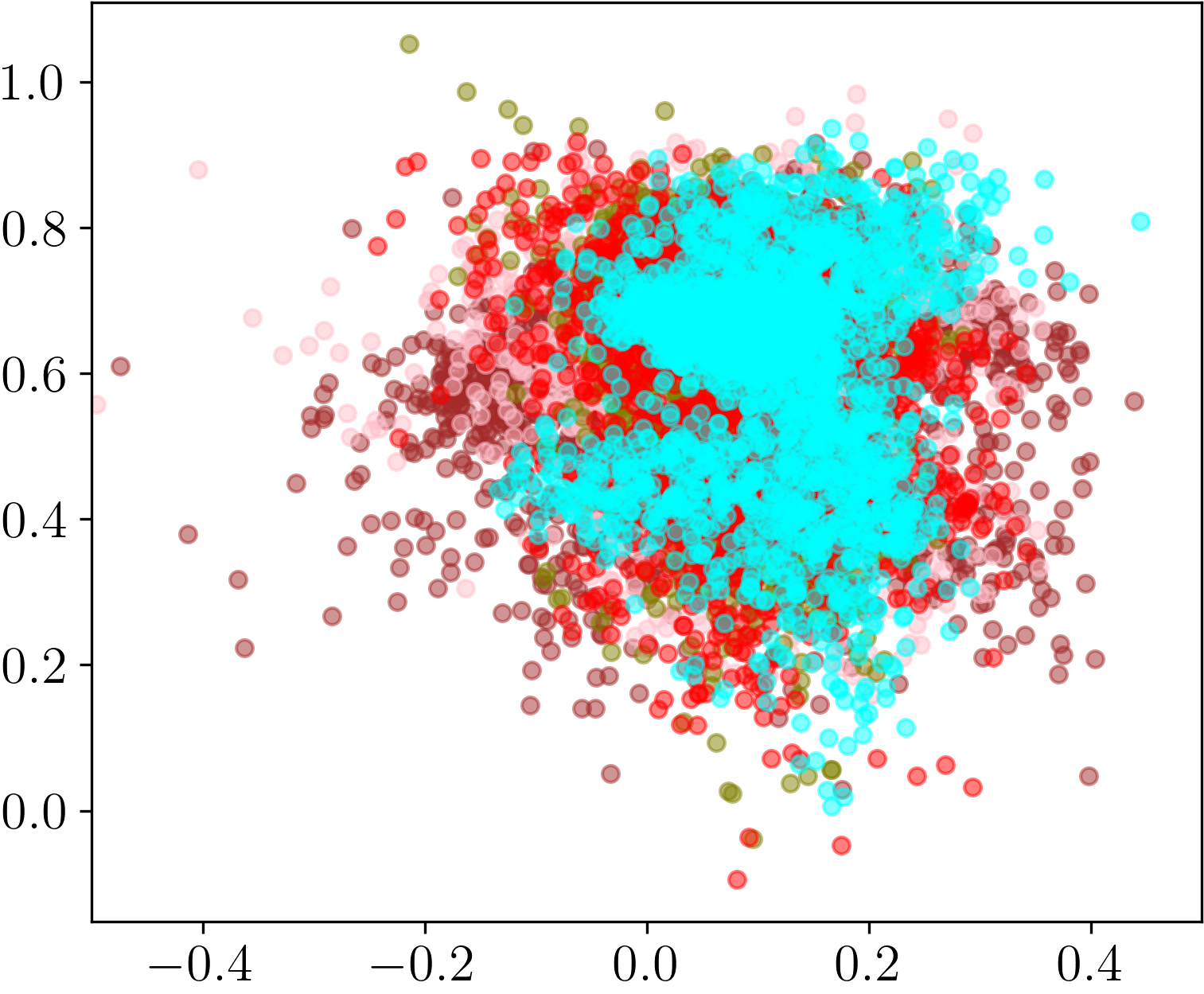}
        \caption{MMVM}
        \label{fig:rats_latent_rat_mixedprior}
    \end{subfigure}
    \hspace{0.25cm}
    \caption{Latent Representation of Rats Brain by Each Rat. In our proposed MMVM model, the five views shared latent representations as the latent representation of all five views (colors) gathered together, while the two baseline models failed to combine multi-views.
    }
    \label{fig:rats_latent_rat}
\end{figure*}